%% file: main_neurips.tex
\documentclass{article}

\PassOptionsToPackage{numbers, compress}{natbib}

\usepackage[final]{neurips_2022}

\usepackage[utf8]{inputenc} 
\usepackage[T1]{fontenc}    

\usepackage{hyperref}       
\usepackage{url}            
\usepackage{booktabs}       
\usepackage{amsfonts}       
\usepackage{nicefrac}       
\usepackage{microtype}      
\usepackage{amsmath,amsthm,amssymb,amsfonts}
\usepackage{algorithm,algorithmic}
\usepackage[dvipsnames,table]{xcolor}
\usepackage{xspace}
\usepackage{dsfont}
\usepackage{makecell}
\usepackage{nicefrac}
\usepackage{caption}
\usepackage{subcaption}

\definecolor{pearThree}{HTML}{E74C3C}
\definecolor{pearcomp}{HTML}{B97E29}
\definecolor{pearDark}{HTML}{2980B9}
\definecolor{pearDarker}{HTML}{1D2DEC}
\hypersetup{
	colorlinks,
	citecolor=pearDark,
	linkcolor=pearThree,
	breaklinks=true,
	urlcolor=black}

\newtheorem{definition}{Definition}

\newtheorem{proposition}{Proposition}

\newtheorem{theorem}{Theorem}
\newtheorem{lemma}{Lemma}
\newtheorem{remark}{Remark}

\usepackage{macros}

\newcommand{\hl}[1]{{#1}}

\title{Near Instance-Optimal PAC Reinforcement Learning \\ for Deterministic MDPs}

\author{%
  Andrea Tirinzoni\thanks{Work done while at Inria Lille.} \\
  Meta AI\\
  Paris, France\\
  \texttt{tirinzoni@fb.com} \\
  \And
  Aymen Al-Marjani \\
  UMPA, ENS Lyon\\
  Lyon, France\\
  \texttt{aymen.al$\_$marjani@ens-lyon.fr} \\
  \And
  Emilie Kaufmann \\
  Univ. Lille, CNRS, Inria, Centrale Lille, UMR 9189 - CRIStAL\\
  Lille, France\\
  \texttt{emilie.kaufmann@univ-lille.fr} \\
}

\usepackage{minitoc}

\begin{document}

\maketitle

\doparttoc 
\faketableofcontents 

\begin{abstract}

In probably approximately correct (PAC) reinforcement learning (RL), an agent is required to identify an $\epsilon$-optimal policy with probability $1-\delta$. While minimax optimal algorithms exist for this problem, its instance-dependent complexity remains elusive in episodic Markov decision processes (MDPs). In this paper, we propose the first nearly matching \blue{(up to a horizon squared factor and logarithmic terms)} upper and lower bounds on the sample complexity of PAC RL in deterministic episodic MDPs with finite state and action spaces. In particular, our bounds feature a new notion of sub-optimality gap for state-action pairs that we call the deterministic return gap. 
While our instance-dependent lower bound is written as a linear program, our algorithms are very simple and do not require solving such an optimization problem during learning. Their design and analyses employ novel ideas, including graph-theoretical concepts (minimum flows) and a new maximum-coverage exploration strategy.  
\end{abstract}

\input{sections/introduction}

\input{sections/preliminaries}
\input{sections/lower_bounds}

\input{sections/algorithms}
\input{sections/experiments}

\input{sections/discussion}

\begin{ack}Aymen Al-Marjani ackowledges the support of the Chaire SeqALO (ANR-20-CHIA-0020). 
Emilie Kaufmann acknoweldges the support of the French National Research Agency under the BOLD project (ANR-19-CE23-0026-04).
\end{ack}

\bibliographystyle{unsrt} 
\bibliography{biblio_bpi}

\newpage

\section*{Checklist}


\begin{enumerate}

	\item For all authors...
	\begin{enumerate}
		\item Do the main claims made in the abstract and introduction accurately reflect the paper's contributions and scope?
		\answerYes{}
		\item Did you describe the limitations of your work?
		\answerYes{See the discussion section.}
		\item Did you discuss any potential negative societal impacts of your work?
		\answerNA{Our results are theoretical and the paper is not oriented towards a speciﬁc application, so a wider broader impact discussion is not applicable.}
		\item Have you read the ethics review guidelines and ensured that your paper conforms to them?
		\answerYes{}
	\end{enumerate}

	\item If you are including theoretical results...
	\begin{enumerate}
		\item Did you state the full set of assumptions of all theoretical results?
		\answerYes{These are stated in the preamble of every Theorem.}
		\item Did you include complete proofs of all theoretical results?
		\answerYes{See the appendix.}
	\end{enumerate}

	\item If you ran experiments...
	\begin{enumerate}
		\item Did you include the code, data, and instructions needed to reproduce the main experimental results (either in the supplemental material or as a URL)?
		\answerYes{See the appendix.}
		\item Did you specify all the training details (e.g., data splits, hyperparameters, how they were chosen)?
		\answerYes{See the appendix.}
		\item Did you report error bars (e.g., with respect to the random seed after running experiments multiple times)?
		\answerYes{}
		\item Did you include the total amount of compute and the type of resources used (e.g., type of GPUs, internal cluster, or cloud provider)?
		\answerYes{See the appendix.}
	\end{enumerate}

	\item If you are using existing assets (e.g., code, data, models) or curating/releasing new assets...
	\begin{enumerate}
		\item If your work uses existing assets, did you cite the creators?
		\answerYes{}
		\item Did you mention the license of the assets?
		\answerYes{}
		\item Did you include any new assets either in the supplemental material or as a URL?
		\answerYes{}
		\item Did you discuss whether and how consent was obtained from people whose data you're using/curating?
		\answerNA{}
		\item Did you discuss whether the data you are using/curating contains personally identifiable information or offensive content?
		\answerNA{}
	\end{enumerate}

	\item If you used crowdsourcing or conducted research with human subjects...
	\begin{enumerate}
		\item Did you include the full text of instructions given to participants and screenshots, if applicable?
		\answerNA{}
		\item Did you describe any potential participant risks, with links to Institutional Review Board (IRB) approvals, if applicable?
		\answerNA{}
		\item Did you include the estimated hourly wage paid to participants and the total amount spent on participant compensation?
		\answerNA{}
	\end{enumerate}

\end{enumerate}
\clearpage

\appendix

\part{Appendix}


\parttoc
\newpage

\input{sections/app_related}
\input{sections/app_flows}

\input{sections/app_lower_bounds}

\input{sections/app_algorithms}
\input{sections/app_tree}
\input{sections/app_experiments}

\end{document}

%% file: sections/introduction.tex
\section{Introduction}

In reinforcement learning \citep[RL,][]{sutton2018reinforcement}, an agent interacts with an environment modeled as a Markov decision process (MDP) by sequentially selecting actions and receiving feedback in the form of reward signals. Depending on the application, the agent may seek to maximize the cumulative rewards received during learning (which is typically phrased as a \emph{regret minimization} problem) or to minimize the number of learning interactions (i.e., the \emph{sample complexity}) for identifying a near-optimal policy. The latter \emph{pure exploration} problem was introduced in \cite{Fiechter94} under the name of Probably Approximately Correct (PAC) RL: given two parameters $\epsilon,\delta > 0$, the agent must return a policy that is $\epsilon$-optimal with probability at least $1-\delta$. Our work focuses on this problem in the context of episodic (a.k.a. finite-horizon) tabular MDPs.

The PAC RL problem has been mostly studied under the lens of minimax (or worst-case) optimality. In the episodic setting, the algorithm proposed in \cite{dann2019policy} has sample complexity bounded by $O(SAH^2\log(1/\delta)/\epsilon^2)$ for an MDP with $S$ states, $A$ actions, horizon $H$, and time-homogeneous transitions and rewards (i.e., not depending on the stage). This is minimax optimal for such a context \citep{dann15PAC}. Similarly, in \cite{Menard21RFE} the authors designed a strategy with $O(SAH^3\log(1/\delta)/\epsilon^2)$ complexity in time-inhomogeneous MDPs, which was later shown to be minimax optimal \cite{Omar21LB}. 

While the minimax framework provides a strong notion of statistical optimality, it does not account for one of the most desirable properties for an RL algorithm: the ability to adapt to the difficulty of the MDP instance. For this reason, researchers recently started to investigate the instance-dependent complexity of PAC RL. Earlier attempts were made in the simplified setting where the agent has access to a generative model (i.e., it can query observations from any state-action pair using a simulator) in $\gamma$-discounted infinite-horizon MDPs \citep{zanette2019almost,al2021adaptive}. The online setting, where the agent can only sample trajectories from the environment, has been studied in \cite{al2021navigating} for discounted MDPs and in \cite{wagenmaker21IDPAC} for episodic time-inhomogeneous MDPs. All these works derive sample complexity bounds that scale with certain gaps between optimal value functions. For instance, in the episodic setting, the \emph{value gap} $\Delta_h(s,a) := V^\star_h(s) - Q^\star_h(s,a)$\footnote{$V^\star$ and $Q^\star$ respectively denote the optimal value and action-value functions, that are defined in Section~\ref{sec:prelim}.} intuitively characterizes the degree of sub-optimality of action $a$ for state $s$ at stage $h$. Unfortunately, these bounds are known to be sub-optimal and how to achieve instance optimality remains one of the main open questions. In fact, recent works on regret minimization \citep{tirinzoni2021fully,dann21ReturnGap} showed that value gaps are often overly conservative, and the same holds for PAC RL. We refer the reader to Appendix~\ref{app:related} for a deeper discussion on problem-dependent results in RL and the review of other related PAC learning frameworks. 


The main challenge towards instance optimality is that existing lower bounds for exploration problems in MDPs \citep{al2021adaptive,tirinzoni2021fully,al2021navigating,dann21ReturnGap}
 are written in terms of non-convex optimization problems. Their ``implicit'' form makes it hard to understand the actual complexity of the setting and, thus, to design optimal algorithms. Existing solutions either derive explicit \emph{sufficient} complexity measures that inspire algorithmic design \citep{wagenmaker21IDPAC}, or solve (a relaxation of) the optimization problem from the lower bound using the empirical MDP as a proxy for the unknown MDP \citep{al2021navigating}. The latter extends the Track-and-Stop idea originally proposed in \cite{garivier2016optimal} for bandits ($H=1$), and requires in particular a large amount of forced exploration. Both solutions have limitations. On the one hand, it is not clear if and how such sufficient complexity measures or relaxations are related to an actual lower bound. On the other hand, strategies solving a black-box optimization problem to find an optimal exploration strategy are typically very inefficient and often come with either only asymptotic ($\delta\rightarrow 0$) guarantees or with poor (far from minimax optimal) sample complexity in the regime of moderate $\delta$.



\vspace{-0.1in}
\paragraph{Contributions}

This paper presents a complete study of PAC RL in tabular \emph{deterministic} episodic MDPs with time-inhomogeneous transitions, a sub-class of stochastic MDPs where state transitions are deterministic and the agent observes stochastic rewards from unknown distributions. 
Our first contribution is an \emph{instance-dependent lower bound} on the sample complexity of any PAC algorithm. We show that the number of visits $n_h^\tau(s,a)$ to any state-action-stage triplet $(s,a,h)$ at the stopping time $\tau$ satisfies
\begin{align}\label{eq:local-lb-sketch}
\bE[n_h^\tau(s,a)] \gtrsim \frac{\log(1/\delta)}{\max(\overline{\Delta}_h(s,a), \epsilon)^2},
\end{align}
where $\overline{\Delta}_h(s,a) := V_1^\star - \max_{\pi\in\Pi_{s,a,h}}V_1^\pi$, with $V_1^\pi$ the expected return of policy $\pi$, $V_1^\star$ the optimal expected return, and $\Pi_{s,a,h}$ the set of all deterministic policies that visit $(s,a)$ at stage $h$. We call these quantities the \emph{deterministic return gaps} due to their closeness with the \emph{return gaps} introduced in \cite{dann21ReturnGap} for general MDPs. In deterministic MDPs, the deterministic return gaps are actually $H$ times larger than the return gaps and they are never smaller than value gaps. Our lower bound on the sample complexity $\tau$ is then the value of a \emph{minimum flow} with local lower bounds \eqref{eq:local-lb-sketch}, i.e., roughly the minimum number of policies that must be played to ensure \eqref{eq:local-lb-sketch} for all $(s,a,h)$. To our knowledge, this is the first instance-dependent lower bound for the PAC setting in episodic MDPs.


On the {algorithmic} side, we design EPRL, a \emph{generic elimination-based method} for PAC RL, and couple it with a novel adaptive sampling rule called \emph{maximum-coverage sampling}. The latter is a simple strategy which does not require solving the optimization problem from the lower bound at learning time in a Track-and-Stop fashion. Instead, it greedily selects the policy that maximizes the number of visited under-sampled triplets $(s,a,h)$, i.e. those having received the least amount of visits so far. We prove that EPRL is $(\varepsilon,\delta)$-correct under any sampling rule. Moreover, we show that the sample complexity of EPRL with max-coverage sampling matches our instance-dependent lower bound up to logarithmic factors and a multiplicative $O(H^2)$ term, while also being minimax optimal. Finally, we perform numerical simulations on random deterministic MDPs which reveal that EPRL can indeed improve over existing minimax-optimal algorithms tailored for the deterministic case.

%% file: sections/preliminaries.tex
\section{Preliminaries}\label{sec:prelim}

Let $\cM := (\cS, \cA, \{f_h,\nu_h\}_{h\in[H]}, s_1, H)$ be a \emph{deterministic} time-inhomogeneous finite-horizon MDP, where $\cS$ is a finite set of $S$ states, $\cA$ is a finite set of $A$ actions, $f_h : \cS\times\cA \rightarrow \cS$ and $\nu_h : \cS\times\cA \rightarrow \cP(\mathbb{R})$ are respectively the transition function and the reward distribution at stage $h\in[H]$, $s_1\in\cS$ is the unique initial state, and $H$ is the horizon. Without loss of generality, we assume that, at each stage $h\in[H]$ and state $s\in\cS$, only a subset $\cA_h(s) \subseteq \cA$ of actions is available. We denote by $r_h(s,a) := \mathbb{E}_{x\sim \nu_h(s,a)}[x]$ the expected reward after taking action $a$ in state $s$ at stage $h$.


A deterministic policy $\pi = \{\pi_h\}_{h\in[H]}$ is a sequence of mappings $\pi_h : \cS \rightarrow \cA$. We let $\Pi := \{\pi \mid \forall h\in[H],s\in\cS : \pi_h(s) \in \cA_h(s)\}$ be the set of all valid deterministic policies. Executing a policy $\pi\in\Pi$ on MDP $\cM$ yields a deterministic sequence of states and actions $(s_h^\pi,a_h^\pi)_{h\in[H]}$, where $s_1^\pi = s_1$, $a_h^\pi = \pi_h(s_h^\pi)$ for all $h\in[H]$, and $s_h^\pi = f_{h-1}(s_{h-1}^\pi,a_{h-1}^\pi)$ for all $h\in\{2,\dots,H\}$. 
We let $\cS_h := \{s\in\cS \mid \exists \pi\in\Pi : s_h^\pi = s\}$ be the subset of states that are reachable at stage $h\in[H]$. Finally, we define $N := \sum_{h=1}^H \sum_{s\in\cS_h}|\cA_h(s)|$ as the total number of reachable state-action-stage triplets.


For each $(s,a,h)$, the \emph{action-value function} $Q_{h}^\pi(s,a)$ of a policy $\pi\in\Pi$ quantifies the expected return when starting from $s$ at stage $h$, playing $a$ and following $\pi$ thereafter. In deterministic MDPs, it has the simple expression $Q_{h}^\pi(s,a) = r_h(s,a) + V_{h+1}^\pi(f_h(s,a))$, where  $V_{h}^\pi(s) := Q_{h}^\pi(s,\pi_h(s))$ is the corresponding value function (with $V_{H+1}^\pi(s) = 0$).
The expected \emph{return} of $\pi$ is simply its value at the initial state, i.e., $V_1^\pi(s_1) = \sum_{h=1}^H r_h(s_h^\pi,a_h^\pi)$.
We let $\Pi^\star := \{\pi^\star\in\Pi : V_1^{\pi^\star}(s_1) = \max_{\pi\in\Pi}V_1^{\pi}(s_1) \}$ be the set of \emph{optimal policies}, i.e., those with maximal return. Finally, we denote by $V^\star_{h}(s)$ and $Q^\star_{h}(s,a)$ the optimal value and action-value function, respectively. These are related by the Bellman optimality equations as $Q^\star_{h}(s,a) = r_h(s,a) + V_{h+1}^\star(f_h(s,a))$ and $V_h^\star(s) = \max_{a\in\cA_h(s)} Q_h^\star(s,a)$.
\vspace{-0.1in}
\paragraph{Learning problem} The agent interacts with an MDP $\cM$ in episodes indexed by $t\in\mathbb{N}$. At the beginning of the $t$-th episode, the agent selects a policy $\pi^t\in\Pi$ based on past history through its \emph{sampling rule}, executes it on $\cM$, and observes the corresponding deterministic trajectory $(s_h^{\pi^t},a_h^{\pi^t})_{h\in[H]}$ together with random rewards $(y_h^t)_{h\in[H]}$, where $y_h^t \sim \nu_h(s_h^{\pi^t},a_h^{\pi^t})$. At the end of each episode, the agent may decide to terminate the process through its \emph{stopping rule} and return a policy $\widehat{\pi}$ prescribed by its \emph{recommendation rule}. We denote by $\tau$ its random stopping time. An algorithm for PAC identification is thus made of a triplet $(\{\pi^t\}_{t\in\mathbb{N}}, \tau, \widehat{\pi})$. The goal of the agent is two-fold. First, for given parameters $\epsilon,\delta > 0$, it must return an $\epsilon$-optimal policy with probability at least $1-\delta$.
\begin{definition}
An algorithm is $(\epsilon,\delta)$-PAC on a set of MDPs $\mathfrak{M}$ if, for all $\cM \in \mathfrak{M}$, it stops a.s. with 
\begin{align*}
\mathbb{P}_{\cM}\left(V_1^{\widehat{\pi}}(s_1) \geq V_1^\star(s_1) - \epsilon\right) \geq 1 - \delta.
\end{align*}
\end{definition}
Second, it should stop as early as possible, i.e., by minimizing the \emph{sample complexity} $\tau$. Henceforth, we assume that the transition function $f$ is known but not the reward distribution $\nu$. Note that if the transitions are unknown, the agent can still estimate them (since it knows that $\cM$ is deterministic) with at most $N \leq SAH$ episodes. 



\vspace{-0.1in}
\paragraph{Minimum flows}

We review some basic concepts from graph theory which will be at the core of our algorithms and analyses later. Full details can be found in Appendix \ref{app:flows}. First note that a deterministic MDP (without reward) can be represented as a \emph{directed acyclic graph} (DAG) with one arc for each available state-action-stage triplet. Let $\cE := \{(s,a,h) : h\in[H],s\in\cS_h,a\in\cA_h(s)\}$ be the set of arcs in the DAG. The minimum flow problem, originally introduced in \cite{voitishin1980algorithms} and later studied in, e.g.,  \cite{adlakha1991minimum, adlakha1999alternate,ciurea2004sequential}, consists of findining a flow (i.e., an allocation of visits) of minimal value which satisfies certain demand constraints in each arc of the graph. In our specific setting, we define a \emph{flow} as any non-negative function $\eta : \cE \rightarrow [0,\infty)$ that belongs to the following set \\ $\Omega := \left\{ \eta : \cE \rightarrow [0,\infty) \mid \sum_{(s',a') : f_{h-1}(s',a')=s} \eta_{h-1}(s',a') = \sum_{a\in\cA_h(s)} \eta_{h}(s,a)\ \ \forall h>1, s\in\cS_h \right\}$.

 This implies that a flow, seen as an allocation of visits to the arcs, satisfies the \emph{navigation constraints} (i.e., incoming and outcoming flows are equal at each state). The minimum flow for a non-negative \emph{lower-bound} function $\underline{c} : \cE \rightarrow [0,\infty)$ is the solution to the following linear program (LP):
 \begin{align*}
     \varphi^\star(\underline{c}) := \min_{\eta\in\Omega} \sum_{a\in\cA_1(s_1)} \eta_1(s_1,a) \quad \text{s.t.} \quad \eta_h(s,a) \geq \underline{c}_h(s,a) \quad \forall (s,a,h) \in \cE.
 \end{align*}
Intuitively, the goal is to minimize the amount of flow leaving the initial state while satisfying the navigation and demand constraints. We note that more efficient algorithms exist for this problem than the LP formulation, e.g., the variant of the Ford-Fulkerson method proposed in \cite{ciurea2004sequential} which is guaranteed to find an integer solution when the lower bound function is integer-valued.

%% file: sections/lower_bounds.tex
\section{The Complexity of PAC RL in Deterministic MDPs}\label{sec:lower-bounds}


Before stating our lower bound, we formally introduce the new notion of sub-optimality gap it features and compare it with other notions that appeared in the literature. 
\vspace{-0.1in}
\paragraph{On sub-optimality gaps} The most popular notion of sub-optimality gap is the so-called \emph{value gap}. It was introduced first in the discounted infinite-horizon setting \citep[e.g.,][]{zanette2019almost} and later for episodic MDPs \citep[e.g.,][]{simchowitz2019non,xu2021fine}. Formally, in the latter context, the value gap of any action $a\in\cA_h(s)$ in state $s\in\cS_h$ at stage $h\in[H]$ is $\Delta_h(s,a) := V^\star_h(s) - Q_h^\star(s,a)$. 
Such a notion of gap appears in the complexity measure for PAC RL proposed in \cite{wagenmaker21IDPAC}. In the deterministic setting, such a complexity measure can be written as $\cC(\cM,\epsilon) = \sum_{(s,a,h)} \frac{1}{\max(\widetilde{\Delta}_h(s,a), \epsilon)^2},$
where $\widetilde{\Delta}_h(s,a) = \min_{a':\Delta_h(s,a')>0}\Delta_h(s,a')$ if $a$ is the unique optimal action at $(s,h)$, and $\widetilde{\Delta}_h(s,a) = {\Delta}_h(s,a)$ otherwise. Intuitively, the (inverse) value gap is proportional to the difficulty of learning whether an action $a$ is sub-optimal for state $s$ at stage $h$. Then, $\cC(\cM,\epsilon)$ is proportional to the difficulty of learning a near optimal action at \emph{all} states and stages. Recent works \citep{tirinzoni2021fully,dann21ReturnGap} showed that this is actually not necessary: if one only cares about computing a policy maximizing the return at the initial state, it is not necessary to learn an optimal action at states which are not visited by such an optimal policy, in particular when the return of all policies visiting the state is small. 
The \emph{return gap} \cite{dann21ReturnGap} was introduced to cope with this limitation. In deterministic MDPs, it can be expressed as
$
\overline{\text{gap}}_h(s,a) := \tfrac{1}{H}\min_{\pi \in \Pi_{s,a,h}} \sum_{\ell=1}^{h} \Delta_{\ell}(s_{\ell}^{\pi},a_{\ell}^{\pi}),
$
where we denote by $\Pi_{s,a,h} := \{\pi\in\Pi : s_h^\pi=s,a_h^\pi=a\}$ the subset of deterministic policies that visit $(s,a)$ at stage $h$. In words, the return gap of $(s,a,h)$ is proportional to the \emph{sum} of value gaps along the best trajectory (i.e., one with maximal return) that visits $(s,a)$ at stage $h$. Intuitively, this means that, if $\Delta_h(s,a)$ is extremely small but all policies visiting $(s,a)$ at stage $h$ need to play a highly sub-optimal action before, then $\Delta_h(s,a) \ll \overline{\text{gap}}_h(s,a)$. 
In the deterministic case, our lower bound reveals that the normalization by $H$ is not necessary, and we define the \emph{deterministic return gap} to be
\begin{align}\label{eq:drg}
\overline{\Delta}_h(s,a) := V^\star(s_1)-\max_{\pi \in \Pi_{s,a,h}} V^{\pi}(s_1).
\end{align}
Using the well-known relationship $ V^\star_1(s_1) - V_1^{\pi}(s_1) = \sum_{h=1}^H \Delta_{h}(s_{h}^{\pi},a_{h}^{\pi})$ \cite[e.g.,][Proposition 5]{tirinzoni2021fully}, it is easy to see that $
\Delta_h(s,a) \leq \overline{\Delta}_h(s,a) 
= H \times \overline{\text{gap}}_h(s,a)$.


\vspace{-0.1in}
\paragraph{Lower Bound} We now present our instance-dependent lower bound based on deterministic return gaps, which will guide us in the design and analysis of sample efficient algorithms. This result is the first instance-dependent lower bound for PAC RL in the episodic setting.  
Lower bounds for $\varepsilon$-best arm identification in a bandit model (which corresponds to $H = S = 1$) were derived in \cite{MannorTsi04,DegenneK19,garivier2021nonasymptotic}, {while problem-dependent regret lower bounds for finite-horizon MDPs are provided in \cite{dann21ReturnGap,tirinzoni2021fully}}. 

\hl{We consider the class $\mathfrak{M}_{\sigma^2}$ of deterministic MDPs with $\sigma^2$-\emph{Gaussian} rewards, in which $\nu_h(s,a) = \cN(r_h(s,a),\sigma^2)$.} Let $\Pi^\epsilon := \{\pi \in \Pi : V_1^\pi(s_1) \geq V_1^\star(s_1) - \epsilon\}$ be the set of all $\epsilon$-optimal policies and denote by $\cZ_h^\epsilon :=\{s \in \cS_h, a \in \cA_h(s) : \Pi_{s,a,h} \cap \Pi^\epsilon \neq \emptyset\} $ the set of state-action pairs that are reachable at stage $h$ by some $\epsilon$-optimal policy. Note that $\overline{\Delta}_h(s,a) \leq \epsilon$ for all $(s,a)\in\cZ_h^\epsilon$.

\begin{theorem}\label{th:instance-lb}
\hl{Let $\sigma^2 > 0$ and fix any MDP $\cM \in \mathfrak{M}_{\sigma^2}$. Then, any algorithm which is $(\epsilon,\delta)$-PAC on the class $\mathfrak{M}_{\sigma^2}$ must satisfy, for any $h\in[H]$, $s\in\cS_h$, and $a\in\cA_h(s)$,}
\begin{align}\label{eq:local-lb}
\bE_{\cM}[n_h^\tau(s,a)] \geq \underline{c}_h(s,a) := \frac{\hl{\sigma^2}\log(1/4\delta)}{4\max(\overline{\Delta}_h(s,a), \overline{\Delta}_{\min}^h, \epsilon)^2},
\end{align}
where $\overline{\Delta}_{\min}^h := \min_{(s',a') : \overline{\Delta}_h(s',a') > 0} \overline{\Delta}_h(s',a')$ if $|\cZ_h^\epsilon| = 1$ and $\overline{\Delta}_{\min}^h := 0$ otherwise. Moreover, for $\underline{c} : \cE \rightarrow [0,\infty)$ the lower bound function defined above,
\begin{align}\label{eq:global-lb}
\bE_{\cM}[\tau] \geq \varphi^\star(\underline{c}).
\end{align}
\end{theorem}
The first lower bound (\ref{eq:local-lb}) is on the number of visits required for any state-action-stage triplet. It intuitively shows that an $(\epsilon,\delta)$-PAC algorithm must visit each triplet proportionally to its inverse deterministic return gap. The second one (\ref{eq:global-lb}) shows that the actual sample complexity of the algorithm must be at least the value of a minimum flow computed with the local lower bounds \eqref{eq:local-lb}, i.e. that the algorithm must play the minimum number of episodes (i.e., policies) that guarantees \eqref{eq:local-lb} for each $(s,a,h)$. Intuitively, due to the navigation constraints of the MDP, there might be no algorithm which tightly matches \eqref{eq:local-lb} for each $(s,a,h)$, and \eqref{eq:global-lb} is exactly enforcing these constraints. While $\varphi^\star(\underline{c})$ has no explicit form, Lemma \ref{lem:flow-bounds} in Appendix \ref{app:flows} gives an idea of how it scales with the gaps:{\small
\begin{align*}
\hl{\max_{h\in[H]} \sum_{s\in\cS_h}\sum_{a\in\cA_h(s)} \frac{\sigma^2\log(1/4\delta)}{4\max(\overline{\Delta}_h(s,a), \overline{\Delta}_{\min}^h, \epsilon)^2} \leq \varphi^\star(\underline{c}) \leq \sum_{h\in[H]} \sum_{s\in\cS_h}\sum_{a\in\cA_h(s)} \frac{\sigma^2\log(1/4\delta)}{4\max(\overline{\Delta}_h(s,a), \overline{\Delta}_{\min}^h, \epsilon)^2}. }
\end{align*}}
Observe that the quantity on the right-hand side resembles the complexity measure $\cC(\cM,\epsilon)$ \cite{wagenmaker21IDPAC}, except that value gaps are replaced by return gaps. This implies that, in general, our lower bound can be much smaller than this complexity. For instance, in an MDP with extremely small value gaps in states which are not visited by an optimal policy, $\varphi^\star(\underline{c})$ does not scale with such gaps at all.

In Appendix~\ref{app:LB_minimax} we further provide a minimax lower bound for PAC RL in deterministic MDPs scaling as $\Omega\left({SAH^2 \log\left({1}/{\delta}\right)}/{\varepsilon^2} \right)$, with a reduced $H^2$ dependency compared to the $H^3$ that appear in the stochastic case \cite{Omar21LB}. \blue{We note that faster rates for deterministic MDPs have already been obtained in other RL settings \citep[e.g.,][]{Yin2021TowardsIO}}. The BPI-UCRL algorithm \cite{Kaufmann21RFE} particularized to deterministic MDPs is matching this lower bound and is thus minixal optimal. We now present the first algorithm which is simultaneously minimax optimal for deterministic MDPs and nearly matching \blue{(up to $O(H^2)$ and logarithmic factors)} the lower bound of Theorem~\ref{th:instance-lb}.

%% file: sections/algorithms.tex

\section{EPRL and Max-Coverage Sampling}\label{sec:algorithms}

We propose a general Elimination-based scheme for PAC RL, called EPRL (Algorithm \ref{alg:elimination-alg}). At each episode $t\in\mathbb{N}$, the algorithm plays a policy $\pi^t$ selected by some sampling rule. Then, based on the collected samples, the algorithm updates its statistics and eliminates all actions which are detected as sub-optimal with enough confidence. This procedure is repeated until a stopping rule triggers. 

Formally, EPRL maintains an estimate $\hat{r}_h^t(s,a) := \frac{1}{n_h^t(s,a)}\sum_{l=1}^t y_h^l \indi{s_h^l=s,a_h^l=a}$, with $\hat{r}_h^0=0$, of the unknown mean reward $r_h(s,a)$ for each $(s,a,h)$. 
Here $n_h^t(s,a) := \sum_{l=1}^t \indi{s_h^l=s,a_h^l=a}$ is the number of times $(s,a)$ is visited at stage $h$ up to episode $t$.
We define the following upper and lower confidence intervals to the value functions of a policy $\pi\in\Pi$:
\begin{align*}
\overline{Q}_h^{t,\pi}(s,a) & := \hat{r}_h^{t}(s,a) + b_h^t(s,a) + \overline{V}_{h+1}^{t,\pi}(f_h(s,a)), \quad \overline{V}_{h}^{t,\pi}(s) :=  \overline{Q}_h^{t,\pi}(s,\pi_h(s)), \\
\underline{Q}_h^{t,\pi}(s,a) &:= \hat{r}_h^{t}(s,a) - b_h^t(s,a) + \underline{V}_{h+1}^{t,\pi}(f_h(s,a)), \quad \underline{V}_{h}^{t,\pi}(s) :=  \underline{Q}_h^{t,\pi}(s,\pi_h(s)),
\end{align*}
where $b_h^t(s,a)$ is a \emph{bonus function}, i.e., the width of the confidence interval at $(s,a,h)$. We assume that rewards are $\sigma^2$-sub-Gaussian with a known factor $\sigma^2$,\footnote{\hl{Note that sub-Gaussianity generalizes the common assumption of bounded rewards in $[0,1]$ (in which case $\sigma^2=1/4$) and the one of Gaussian rewards with variance $\sigma^2$ (as used in the lower bound of Theorem \ref{th:instance-lb}).}} which allows us to choose
\begin{align}\label{eq:bonus}
  \hl{b_h^t(s,a) := \sqrt{\frac{\beta(n_h^t(s,a),\delta)}{n_h^t(s,a)}}, \quad \beta(t,\delta) := 2\sigma^2\log\left(\frac{4t^2N}{\delta}\right).}
\end{align}


\begin{algorithm}[h]
\caption{Elimination-based PAC RL (EPRL) for deterministic MDPs}\label{alg:elimination-alg}
\begin{algorithmic}[1]
\STATE \textbf{Input:} deterministic MDP (without reward) $\cM := (\cS, \cA, \{f_h\}_{h\in[H]}, s_1, H)$, $\epsilon$, $\delta$
\STATE Initialize ${\cA}^{0}_h(s) \leftarrow \cA_h(s)$ for all $h\in[H],s\in\cS_h$
\STATE Set $n^0_h(s,a) \leftarrow 0$ for all $h\in[H],s\in\cS_h,a\in\cA_h(s)$
\FOR{$t=1,\dots$}
\STATE Play $\pi^t \leftarrow \textsc{SamplingRule()}$ 
\STATE Update statistics $n_h^t(s,a), \hat{r}_h^t(s,a)$
\STATE $\cA^{t}_h(s) \leftarrow \cA^{t-1}_h(s) \cap \left\{a\in\cA : \max_{\pi\in\Pi_{s,a,h} \cap \Pi^{t-1}} \overline{V}_{1}^{t,\pi}(s_1) \geq \max_{\pi\in\Pi} \underline{V}_{1}^{t,\pi}(s_1) \right\}$
\STATE {\small\blue{\quad where $\Pi^{t-1} \leftarrow \big\{\pi\in\Pi \mid \forall s,h : \pi_h(s) \in \cA^{t-1}_h(s) \vee \cA^{t-1}_h(s) = \emptyset\big\}$ (need not be stored/computed)}}
\STATE \textbf{if} $\max_{\pi\in\Pi^{t}} \left(\overline{V}_1^{{\pi}, t}(s_1) - \underline{V}_1^{{\pi},t}(s_1) \right) \leq \epsilon$ or $\forall h\in[H], s\in\cS_h : |\cA^{t}_h(s)| \leq 1$ \textbf{then} 
\STATE \quad Stop and recommend $\widehat{\pi} \in \argmax_{\pi\in\Pi^{t}} \overline{V}_1^{{\pi}, t}(s_1)$ 
\STATE \textbf{end if}
\ENDFOR
\vspace{0.2cm}
\STATE \textbf{function} \textsc{MaxCoverage}()
\STATE \quad Let $k_t \leftarrow \min_{h\in[H],s\in\cS_h,a\in\cA_h^{t-1}(s)} n_h^{t-1}(s,a) + 1$ and $\bar{t}_{k_t} \leftarrow \inf_{l\in\mathbb{N}}\{l : k_l=k_t\}$
\STATE \quad \textbf{if} $t\ \mathrm{mod}\ 2 = 1$ \textbf{then}
\STATE \quad \quad\textbf{return} $\pi^t \leftarrow \argmax_{\pi\in\Pi} \sum_{h=1}^H \indi{a_h^\pi\in\cA_h^{\bar{t}_{k_t}-1}(s_h^\pi), n_h^{t-1}(s_h^\pi,a_h^\pi) < k_t}$ 
\STATE \quad \textbf{else}
\STATE \quad \quad \textbf{return} $\pi^t \leftarrow \argmax_{\pi\in\Pi^{t-1}} \sum_{h=1}^H b_h^{t-1}(s_h^\pi,a_h^\pi)$ \hfill(\textsc{MaxDiameter})
\STATE \quad \textbf{end if}
\end{algorithmic}
\end{algorithm}

\vspace{-0.1in}
\paragraph{Elimination rule}

Algorithm \ref{alg:elimination-alg} keeps a set of active (or candidate) actions $\cA^{t}_h(s)$ for each stage $h\in[H]$, state $s\in\cS_h$, and episode $t\in\mathbb{N}$. Let $\Pi^{t} := \{\pi\in\Pi \mid \forall s,h : \pi_h(s) \in \cA^{t}_h(s) \vee \cA^{t}_h(s) = \emptyset\}$ be the subset of \emph{active} policies that only play active actions at episode $t$. 
Note that an active policy can play an arbitrary action in states where all actions have been eliminated. As can be seen in Line 7 of Algorithm~\ref{alg:elimination-alg}, action $a$ is eliminated from $\cA_h^{t}(s)$ if 
$\max_{\pi\in\Pi_{s,a,h} \cap \Pi^{t-1}} \overline{V}_{1}^{t,\pi}(s_1) \leq \max_{\pi\in\Pi} \underline{V}_{1}^{t,\pi}(s_1)$, 
that is, when we are confident that none of the policies visiting $(s,a)$ at stage $h$ is optimal. We recall that $\Pi_{s,a,h}$ denotes the set of all deterministic policies that visit $s,a$ at stage $h$.
The maximum restricted to $\Pi_{s,a,h}$ can be easily computed by standard dynamic programming (e.g., it is enough to set the reward to $-\infty$ for all state-action pairs different than $(s,a)$ at stage $h$). If $\Pi_{s,a,h} \cap \Pi^{t-1} = \emptyset$, we set the maximum to $-\infty$ so that the elimination rule triggers. 
\begin{remark}
\blue{While defining $\Pi^t$ simplifies the presentation, EPRL neither stores nor enumerates the set of active policies. In particular, EPRL does not eliminate policies but rather $(s,a,h)$ triplets. The sets $\cA_h^t(s)$ can be updated in polynomial time by dynamic programming without ever computing $\Pi^t$.}
\end{remark}

\paragraph{Stopping rule}

EPRL uses two different stopping rules (Line 9). The first one checks whether, for all active policies $\pi\in\Pi^t$, the confidence interval on the return, $\overline{V}_1^{{\pi}, t}(s_1) - \underline{V}_1^{{\pi},t}(s_1) = 2\sum_{h=1}^H b_h^t(s_h^\pi,a_h^\pi)$, which we refer to as \emph{diameter}, is below $\epsilon$. 
The second one checks whether each set $\cA^{t}_h(s)$ contains either $1$ action or $0$ actions (which happens when the state is unreachable by an optimal policy). 
In both cases, we recommend the optimistic (active) policy (Line 10).

\paragraph{Sampling rule} While EPRL may be used with different sampling rules, we recommend the max-coverage sampling rule described in Algorithm~\ref{alg:elimination-alg}. This sampling rule aims at ensuring that no $(s,a,h)$ triplet remains under-visited for too long. This is achieved by selecting the policy which greedily maximizes the number of visited under-sampled triplets, denoted by $\mathcal{U}_t$. The quantity $k_t  = \min_{(s,a,h) : a \in \cA_h^{t-1}(s)} n_h^{t-1}(s,a) + 1$ can be interpreted as the target minimum number of visits from active triplets that we want to achieve in round $t$ and permit to define 
\[ \pi^{t} = \argmax_{\pi \in \mathrm{\Pi}}\sum_{h=1}^{H}\ind\left((s_h^{\pi},a_h^{\pi},h) \in \mathcal{U}_t\right) \text{ with } \ \mathcal{U}_t =\left\{(s,a,h) : a \in \cA_h^{\overline{t}_{k_t}-1}(s), n_h^{t-1}(s,a) < k_t\right\},\] 
where $\overline{t}_k = \inf\{ t : k_t = k\}$ is the first round in which the target is set to $k$. The argmax over $\Pi$ can be computed using dynamic programming. We emphasize that this argmax is not restricted to the set of active policies, meaning that we may play eliminated actions in order to augment the coverage (that is, the minimal number of visits) faster. Every even round, max-coverage instead chooses an active policy maximizing the diameter featured in the stopping rule (max-diameter sampling). As we shall see in our analysis, this dichotomous behavior is needeed in order to maintain minimax-optimality. 

\paragraph{Comparison with other elimination-based algorithms}

 The work of \cite{EvenDaral06} provides a heuristic using action eliminations to find an $\varepsilon$-optimal policy in a discounted MDP. However, no sample complexity guarantees are given for this algorithm, which uses a different elimination rule, based on confidence intervals on the optimal value function, and a uniform sampling rule. The MOCA algorithm \cite{wagenmaker21IDPAC} also uses a different action elimination rule compared to ours. In particular, the decision to eliminate $(s,a,h)$ is made based only on rewards that can be obtained after visiting $(s,a,h)$. Moreover, this algorithm uses a complex phase-based sampling rule, while the sampling rule of EPRL is fully adaptive.

\section{Theoretical Guarantees}\label{sec:analysis}

Our first result, proved in Appendix~\ref{app:correctness}, shows that EPRL is $(\varepsilon,\delta)$-PAC under any sampling rule. It follows from the fact that 1) the choice of bonus function \eqref{eq:bonus} ensures that all the confidence intervals are valid and 2) state-action pairs from optimal trajectories are never eliminated when this holds.

\begin{theorem}\label{th:pac}
Algorithm \ref{alg:elimination-alg} is $(\epsilon,\delta)$-PAC provided that the sampling rule makes it stop almost surely.
\end{theorem}


We now analyze the sample complexity of EPRL combined with max-coverage sampling. 

\begin{theorem}\label{th:max-cover-sample-comp}(Informal version of Theorem \ref{th:formal} in Appendix \ref{app:max-cover})
With probability at least $1-\delta$, the sample complexity of EPRL combined with the maximum-coverage sampling rule satisfies $\tau = \widetilde{O}(\varphi^\star(g))$, 
where $g : \cE \rightarrow [0,\infty)$ is the lower bound function defined by{\small
\begin{align*}
  \hl{g_h(s,a) := \frac{32\sigma^2 H^2}{\max\left(\overline{\Delta}_h(s,a),\overline{\Delta}_{\min}, \epsilon \right)^2}  \left( \log\left(\frac{4N^3}{\delta}\right) + 8\log\left( \frac{16\sigma H\log\left(\frac{4N^3}{\delta}\right)}{\max\left(\overline{\Delta}_h(s,a),\overline{\Delta}_{\min}, \epsilon \right)}  \right)\right) + 2.}
\end{align*}}
Moreover, with the same probability, $\tau = \widetilde{O}(\frac{SAH^2}{\epsilon^2}\log(1/\delta))$, where $\widetilde{O}$ hides logarithmic terms.
\end{theorem}

First note that EPRL combined with such a sampling rule is \emph{minimax optimal}\blue{, since it matches the worst-case lower bound derived in Appendix~\ref{app:LB_minimax}}. In addition, the leading term in the instance-dependent complexity is the value of a minimum flow with a lower bound function $g$ that, in case multiple disjoint optimal trajectories exist\footnote{When there is a unique optimal trajectory, our upper bound scales with $\overline{\Delta}_{\min} = \min_{h\in[H]}\overline{\Delta}_{\min}^h$ at all stages $h$, while the lower bound scales with $\overline{\Delta}_{\min}^h$ at stage $h$. We believe the latter should be improvable to obtain a dependence on $\overline{\Delta}_{\min}$ matching the one in the upper bound.}, matches the gap-dependence in \eqref{th:instance-lb}. \hl{If we suppose that there exist at least two disjoint optimal trajectories, in which case $\overline{\Delta}_{\min} = \overline{\Delta}_{\min}^h = 0$, then, thanks to Lemma \ref{lem:flow-of-sum} in Appendix \ref{app:flows}, one can easily see that $\varphi^\star(g) \leq \alpha H^2\varphi^\star(\underline{c}) + \varphi^\star(g')$,
where $g_h'(s,a) := \widetilde{O}(H^2/\max\left(\overline{\Delta}_h(s,a), \epsilon \right)^2)$ does not depend on $\delta$, $\underline{c}$ is the ``optimal'' lower bound function from \eqref{th:instance-lb}, and $\alpha$ is a numerical constant. Hence, in the asymptotic regime ($\delta \rightarrow 0$), $\varphi^\star(g)$ matches our lower bound up to a $O(H^2)$ multiplicative factor.}

\begin{remark}
  \hl{
    Since Theorem \ref{th:instance-lb} was derived for Gaussian rewards, EPRL is instance-optimal only when the reward distribution is Gaussian. This is not surprising since it is well known from the bandit literature \citep[e.g.,][]{lai1985asymptotically} that sample complexity bounds scaling with a sum of inverse squared gaps are optimal only for Gaussian distributions. Note, however, that EPRL works in greater generality and achieves complexity $\varphi^\star(\underline{c})$ for any $\sigma^2$-sub-Gaussian distribution without knowing its specific form (e.g., whether it is Gaussian or not). What is the optimal rate for other common distributions (e.g., bounded rewards in $[0,1]$) and how to achieve it remains an open question.}
\end{remark}


Finally, our sample complexity bound has an extra multiplicative logarithmic term which roughly scales as $O(\log(H)\log(H\log(1/\delta)/\epsilon))$. While this term makes the dependence on $\delta$ sub-optimal by a $\log\log(1/\delta)$ factor, we show in Appendix \ref{app:tree} that it can be removed in the specific case of tree-based MDPs \citep{dann21ReturnGap}.

\begin{remark}
  \blue{
  We believe that the sub-optimality on $H$ could be reduced to a single $H$ factor by boosting the lower bound. In Appendix \ref{app:tree}, we show that this is indeed possible in tree-based MDPs. As for the upper bound, reducing $H^2$ to $H$ is likely to require tighter concentration bounds on values.}
\end{remark}

\begin{remark}
  In Appendix \ref{app:max-diam}, we prove that, when using the max-diameter sampling rule (Line 18 in Algorithm \ref{alg:elimination-alg}) at each step, the sample complexity is $\widetilde{O}(\sum_{(s,a,h)}H^2/\max(\overline{\Delta}_h(s,a),\overline{\Delta}_{\min}, \epsilon )^2)$. While this scales with the same gaps as Theorem \ref{th:max-cover-sample-comp}, it is only a naive upper bound to the minimum flow value (see Section \ref{sec:lower-bounds}). The intuition is that max-diameter sampling alone does not ensure that all triplets are visited sufficiently often, which prevents us from tightly controlling their elimination times.
\end{remark}

\paragraph{Proof sketch} The complete proof is given in Appendix~\ref{app:max-cover}. It first relies on the following crucial result which relates the deterministic return gaps to the sum of confidence bonuses.

\begin{lemma}[Diameter vs gaps]\label{lem:diam-vs-gap}
With probability at least $1-\delta$, for any $t\in\mathbb{N}, h\in[H],s\in\cS_h,a\in\cA_h(s)$, if $a\in\cA_h^t(s)$ and the algorithm did not stop at the end of episode $t$,
\begin{align*}
 \max\left(\frac{\overline{\Delta}_h(s,a)}{4},\frac{\overline{\Delta}_{\min}}{4}, \frac{\epsilon}{2} \right) \leq
  \max_{\pi\in\Pi^{t-1}}\sum_{h=1}^H b_h^t(s_h^\pi,a_h^\pi),
\end{align*}
where $\overline{\Delta}_{\min} := \min_{h\in[H]}\min_{s\in\cS_h}\min_{a:\overline{\Delta}_h(s,a)>0}\overline{\Delta}_h(s,a)$ if there exists a unique optimal trajectory $(s_h^\star,a_h^\star)_{h\in[H]}$, and $\overline{\Delta}_{\min} := 0$ in the opposite case.
\end{lemma}

In our analysis, we refer to the set of consecutive time steps $\{t\in\mathbb{N} : k_t=k\}$ as the $k$-th \emph{period}. Using the fact that in period $k+1$ each active triplet has been visited at least $k$ times \hl{(which allows to upper bound each bonus $b_h^t(s_h^\pi,a_h^\pi)$ for $\pi \in \Pi^{t-1}$ by a quantity scaling in $\sqrt{1/k}$)}, one can use Lemma~\ref{lem:diam-vs-gap} to obtain an upper bound $\overline{\kappa}_{s,a,h} \simeq \frac{H^2\log(1/\delta)}{\max\left(\overline{\Delta}_h(s,a),\overline{\Delta}_{\min}, \epsilon \right)^2}$ on the last period in which $(s,a,h)$ is active (Lemma \ref{lem:elim-periods} in Appendix~\ref{app:max-cover}). A crucial step of the proof is then to upper bound the duration of the $k$-th period,  $d_k := \sum_{t=1}^\tau \indi{k_t = k}$.

\begin{lemma}\label{lem:period-duration-main-paper}
 $d_k \leq 2(\log(H)+1)\phi^\star(\underline{c}^{k})$ where $\underline{c}^{k}_h(s,a) = \ind(a \in \cA_h^{\overline{t}_k -1}(s), n_h^{\overline{t}_{k}-1}(s,a) < k)$.
\end{lemma}
The intuition behind this result is as follows. Recall that the goal of the max-coverage sampling rule in period $k$ is to visit at least once each $(s,a,h)$ that is active (i.e., $a \in \cA_h^{\overline{t}_k -1}(s)$) and undersampled (i.e., $n_h^{\overline{t}_{k}-1}(s,a) < k$). By definition, the minimum flow $\phi^\star(\underline{c}^{k})$ is the minimum number of policies that need to be played to achieve this goal. Interestingly, Lemma \ref{lem:period-duration-main-paper} shows that the number of policies played by max-coverage to visit all active undersampled triplets is very close to its theoretical minimum, despite the fact that the algorithm never computes an actual minimum flow. We prove this by interpreting max-coverage sampling as a greedy maximization of some coverage function (related to a minimum flow problem) and leveraging the theory of sub-modular maximization \citep[e.g.,][]{krause2014submodular}. 

Thanks to Lemma \ref{lem:period-duration-main-paper}, we have that \hl{\[\tau \leq 2(\log(H)+1)\sum_{k=1}^{k_\tau} \varphi^\star(\underline{c}^k),\]} where $k_\tau$ is the index of the period at which the algorithm stops. To bound this quantity we carefully apply the theory of minimum flows and their dual problem of \emph{maximum cuts}. Let us define a cut $\cC$ as any subset of states containing the initial state and let $\cE(\cC)$ be the set of arcs that connect states in $\cC$ with states not in $\cC$. The well-known min-flow-max-cut theorem (Theorem \ref{th:min-flow-max-cut-app} stated in Appendix \ref{app:flows}) states that, for any lower bound function $\underline{c}$, $\varphi^\star(\underline{c}) = \max_{\cC\in\mathfrak{C}} \sum_{(s,a,h) }\underline{c}_h(s,a)$, where $\mathfrak{C}$ denotes the set of all valid cuts.  Then,
\begin{align*}
  k\varphi^\star(\underline{c}^k) \leq \max_{\cC\in\mathfrak{C}}\sum_{(s,a,h)\in\cE(\cC)}  k \indi{a\in\cA_h^{\bar{t}_k-1}(s)}  \leq \max_{\cC\in\mathfrak{C}}\sum_{(s,a,h)\in\cE(\cC)}  (\overline{\kappa}_{s,a,h}+1) = \varphi^\star(g),
\end{align*}
where $g : \cE \rightarrow [0,\infty)$ is defined by $g_h(s,a) = \overline{\kappa}_{s,a,h}+1$. \hl{It follows that 
\begin{eqnarray*}
 \tau &\leq& 2(\log(H)+1)\sum_{k=1}^{k_\tau} \frac{1}{k}\varphi^\star(g) \\
 &\leq& 2(\log(H)+1)\left(\log(k_{\tau}) + 1\right)\varphi^\star(g) \\
 & \leq & 2(\log(H)+1)\left(\max_{(s,a,h)}\log(\overline{\kappa}_{s,a,h}) + 1\right) \varphi^\star(g).
\end{eqnarray*}
Using the expression of $\overline{\kappa}_{s,a,h}$ given in Lemma \ref{lem:elim-periods} of Appendix~\ref{app:max-cover} concludes the proof of the stated $\widetilde{O}(\varphi^\star(g))$ instance-dependent bound. For the worse-case bound, we refer the reader to Theorem~\ref{th:sample-comp-worst-case-adaptive}.}

\qed

%% file: sections/experiments.tex

\section{Experiments}\label{sec:expes}

We compare numerically EPRL to the minimax optimal BPI-UCRL algorithm \cite{Kaufmann21RFE}, adapted to the deterministic setting, on synthetic MDP instances. For EPRL, we experiment with two sampling rules: max-coverage (maxCov) and max-diameter (maxD, see Line 18 of Algorithm~\ref{alg:elimination-alg}). We defer to Appendix~\ref{sec:app_computational} some implementation details, including a precise description of the BPI-UCRL baseline.

We generate random ``easy'' deteministic MDP instances with Gaussian rewards of variance 1 using the following protocol. For fixed $S,A,H$ the mean rewards $r_h(s,a)$ are drawn i.i.d. from a uniform distribution over $[0,1]$ and for each state-action pair, the next state is chosen uniformly at random in $\{1,\dots,S\}$. Finally, we only keep MDP instances whose minimum value gap, denoted by $\Delta_{\min}$, is larger than $0.1$. Our first observation is that depending on the MDP, the identity of the best performing algorithm can be different. In Figure~\ref{fig:cases} we show the distribution of the sample complexity (estimated over 10 Monte Carlo simulations) for three different MDPs obtained from our sampling procedure with $S,A=2$ and $H=3$ and for algorithms that are run with parameters $\delta = 0.1$ and $\epsilon = 1.5\Delta_{\min}$.
 
\begin{figure}[h]
    \centering
    \begin{subfigure}[b]{0.32\textwidth}
    \includegraphics[width = \textwidth]{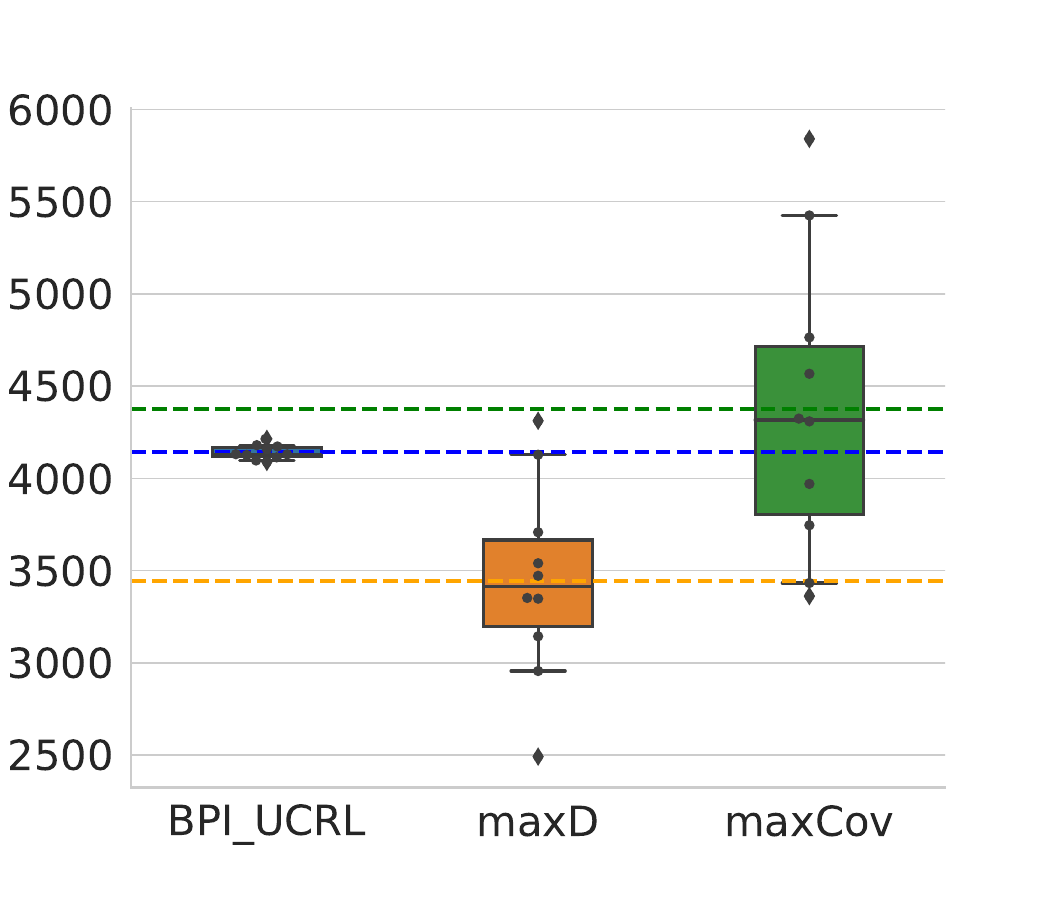}
    \label{fig:maxD}
    \end{subfigure}
    \hfill
    \begin{subfigure}[b]{0.32\textwidth}
    \includegraphics[width = \textwidth]{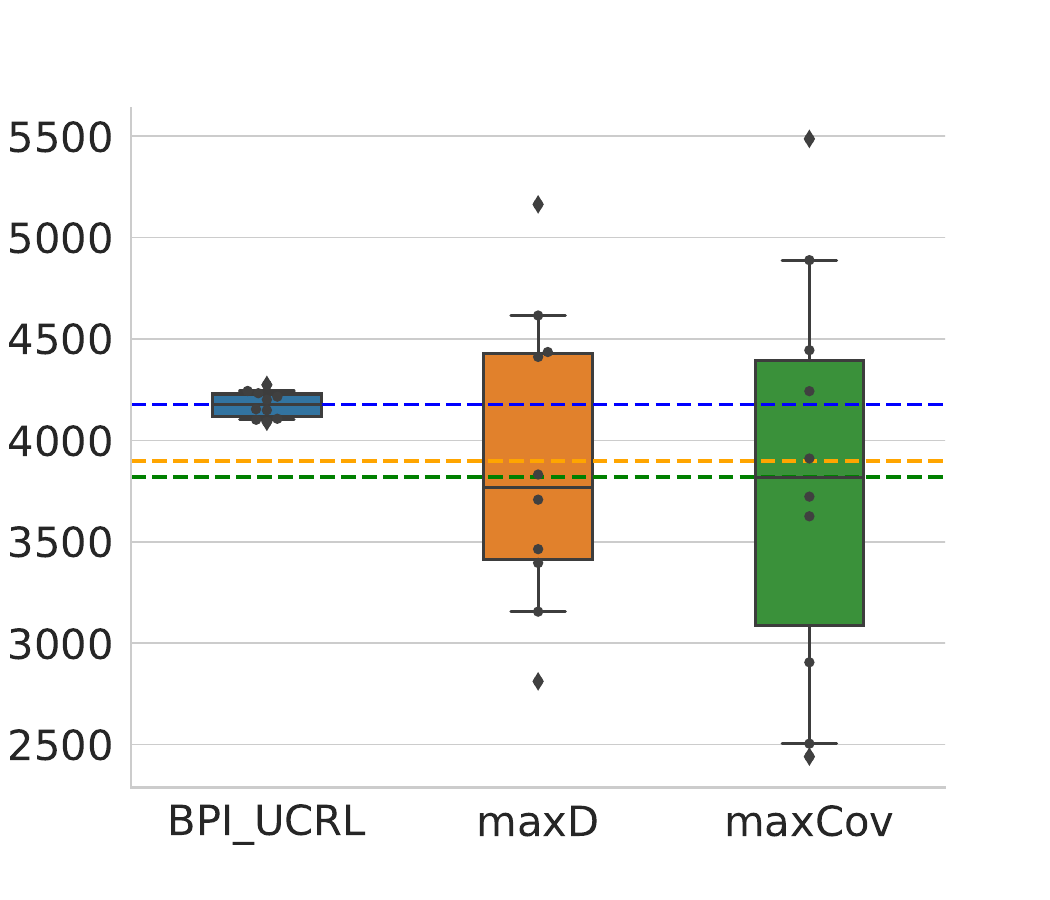}
   \label{fig:maxCov}
    \end{subfigure}
    \hfill
    \begin{subfigure}[b]{0.32\textwidth}
    \includegraphics[width = \textwidth]{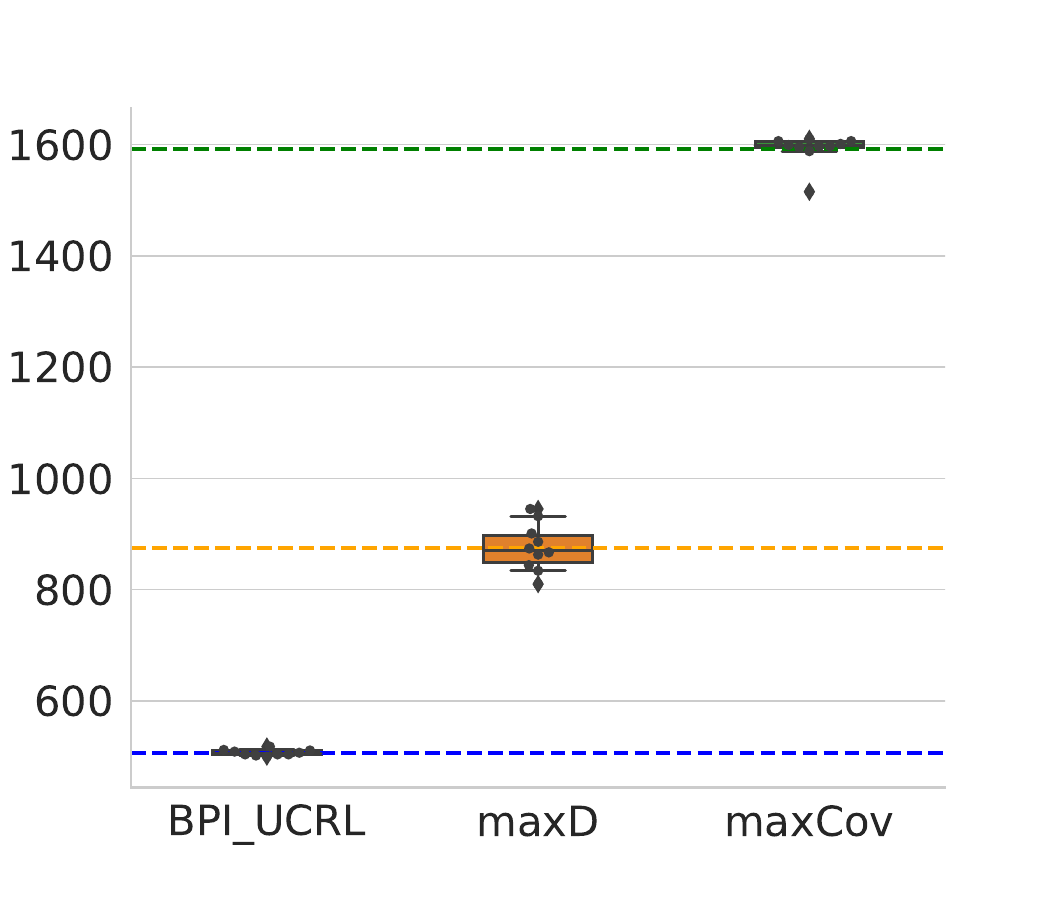}
   \label{fig:BPI_UCRL}
    \end{subfigure}
    \caption{Distribution of stopping times on particular MDPs over $10$ runs, with $\epsilon = 1.5 \Delta_{\min}$. The horizontal lines represent the average sample complexity.}
    \label{fig:cases}
\end{figure}

To get a better understanding of this phenomenon, we then generated $10$ MDP instances of size $(S,A,H) = (2,2,3)$ and for each MDP we ran EPRL and BPI-UCRL for 25 values of $\epsilon$ in a grid $[0.05 \Delta_{\min}, 10\Delta_{\min}]$ and $\delta = 0.1$.  We ran $10$ Monte-Carlo simulations for each value of the triplet $(\textrm{MDP}, \textrm{algorithm } \mathsf{A}, \epsilon)$, in order to estimate the expected sample complexity $\bE_{\mathsf{A}}[\tau_{\delta}]$. In Figure~\ref{fig:scatterplot} we plot the relative performance (ratio of sample complexities) 
of different algorithms as a function of the value of $\varepsilon/\Delta_{\min}$: each point corresponds to a different MDP and a different value of $\varepsilon$. We observe that for large values of $\epsilon/\Delta_{\min}$, BPI-UCRL has a smaller sample complexity than both versions of EPRL, with a ratio never exceeding $2$ (resp. $3$) for max-diameter (resp. max-coverage). However, in the more interesting small $\epsilon/\Delta_{\min}$ regime EPRL is better by several orders of magnitude. This is expected since, for small $\epsilon$, EPRL is able, through its elimination rule, to identify the optimal policy long before the diameter goes below $\epsilon$. We observe that the threshold of $\epsilon/\Delta_{\min}$ at which EPRL algorithms become a better choice than BPI-UCRL seems to vary with the MDP.

\begin{figure}[h]
     \centering
    \begin{subfigure}[b]{0.40\textwidth}
    \includegraphics[width = \textwidth]{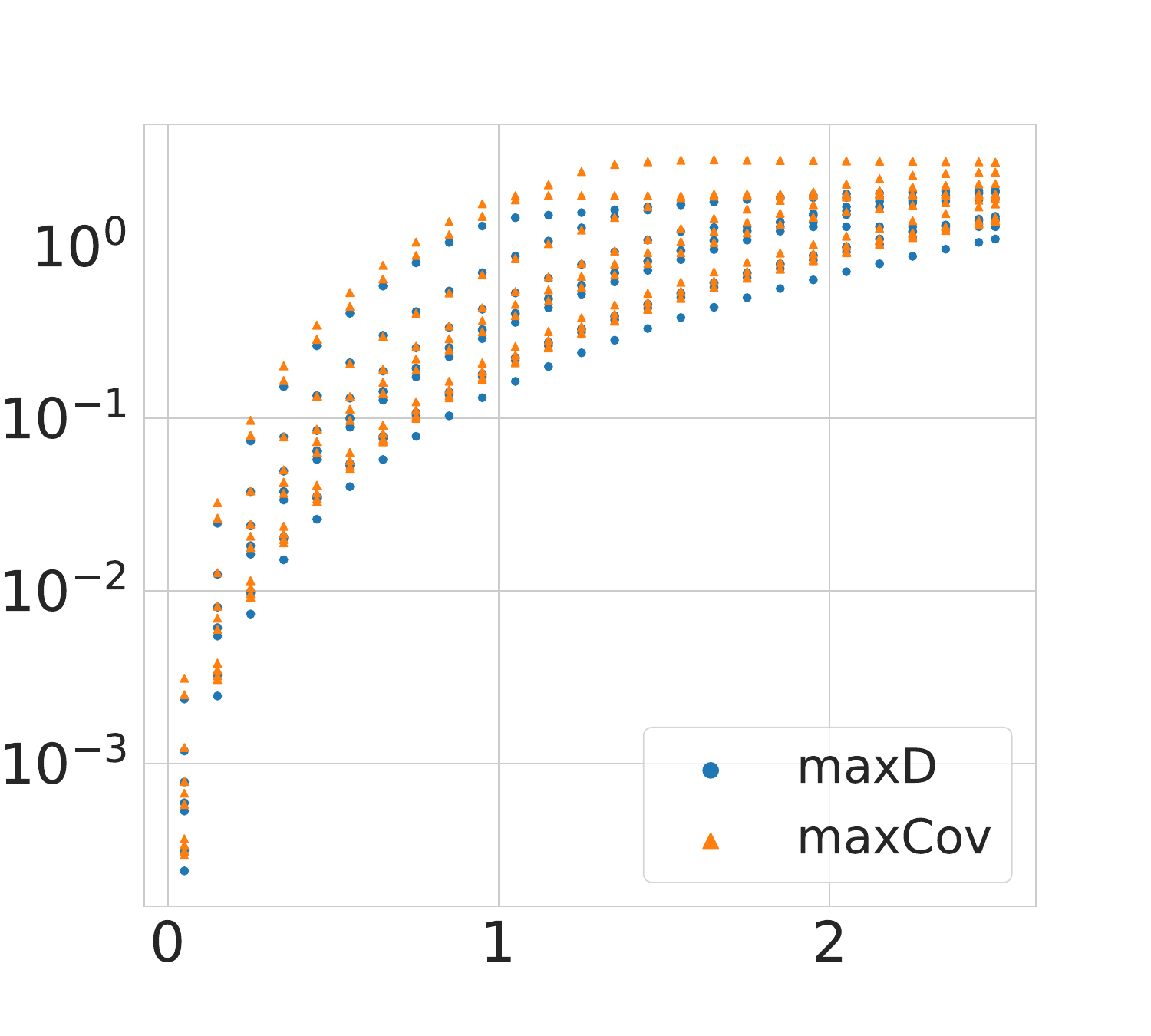}
    \end{subfigure}
    \begin{subfigure}[b]{0.40\textwidth}
    \includegraphics[width =\textwidth]{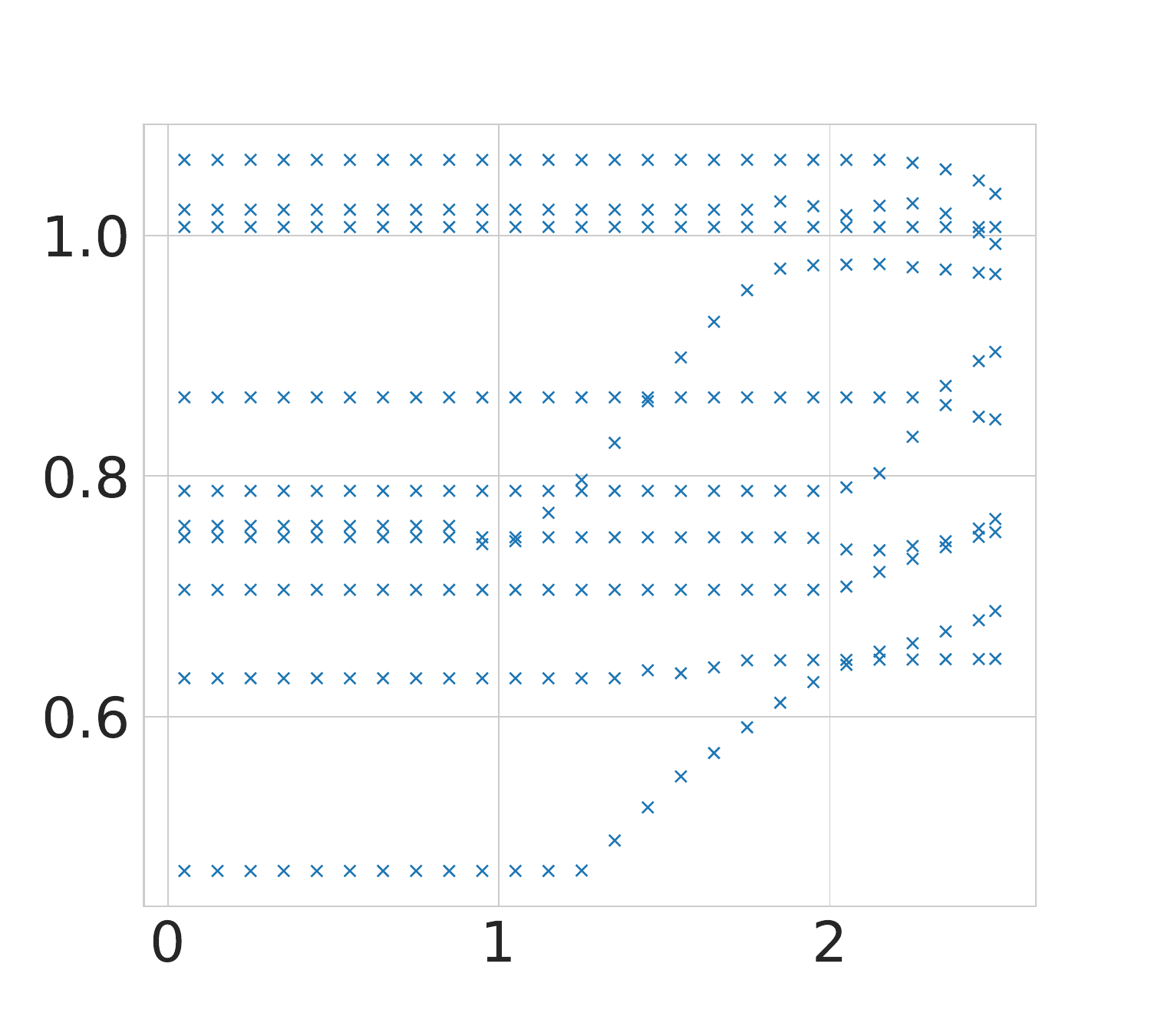}
    \end{subfigure}
    \caption{Ratios in log-scale $\bE_{\mathrm{A}}[\tau_\delta]/\bE_{\mathrm{BPI-UCRL}}[\tau_\delta]$ for $\mathrm{A}$ in $\{ \textrm{maxD, maxCov}\}$ (left) and $\bE_{\mathrm{maxD}}[\tau_\delta]/\bE_{\mathrm{maxCov}}[\tau_\delta]$ (right) as a function of $\epsilon/\Delta_{\min}$.}
   \label{fig:scatterplot}
\end{figure}

Our experiments also reveal an intriguing phenomenon: the use of max-diameter sampling within EPRL often outperforms max-coverage sampling, even if there exists MDPs (2 out of 10 in our experiments) in which max-coverage is indeed empirically better. We leave as future work to obtain a better characterizations of MDPs for which EPRL with max-coverage sampling performs best.

%% file: sections/discussion.tex
\section{Discussion}

We derived an instance-dependent and a worst-case lower bound characterizing the complexity of PAC RL in deterministic MDPs, and proposed a general elimination algorithm together with a novel maximum-coverage sampling rule that nearly matches them \blue{(up to $O(H^2)$ and logarithmic factors)}. We conclude with some discussion about our results and future directions.

\vspace{-0.2cm}

\paragraph{Max-coverage vs max-diameter}

While minimax optimality can be easily achieved with very simple strategies (like max-diameter or BPI-UCRL), instance optimality requires careful algorithmic design. Our coverage-based strategy is built around the idea of ``uniformly'' exploring the whole MDP, while using an elimination strategy to ensure that no $(s,a,h)$ is sampled much more than what the lower bound prescribes. Notably, this sampling rule is very simple, while exiting PAC RL algorithms with instance-dependent complexity are all quite involved \citep{wagenmaker21IDPAC,al2021navigating}. Moreover, max-coverage sampling naturally extends to stochastic MDPs, e.g., by doing optimistic planning on an MDP with a reward function equal to $1$ for under-sampled triplets and $0$ for the others. Finally, in our experiments on random instances, we observed that max-diameter is often comparable or better than max-coverage. We leave as future work to investigate whether the latter is also provably near instance-optimal.

\vspace{-0.2cm}

\paragraph{Computational aspects}

Our sampling rule requires solving one dynamic program per episode, which takes $O(N)$ time. The bottleneck is the elimination rule, which requires $O(N^2)$ per-episode time complexity to solve one dynamic program for each active triplet. However, we note that eliminations could be checked periodically (e.g., even at exponentially-separated times) without significantly compromising the sample complexity guarantees.

\vspace{-0.2cm}

\paragraph{Improving our results}

Our instance-dependent upper bound for max-coverage sampling is sub-optimal by a factor $H^2$ and a multiplicative $O(\log\log(1/\delta))$ term. In Appendix \ref{app:tree}, we show that, for the specific sub-class of tree-based MDPs \citep{dann21ReturnGap}, we can obtain improved results in all these aspects. In particular, we show that (1) the lower bound scales with an extra factor $H$ and it is fully explicit, (2) the multiplicative log terms in the sample complexity of coverage-based sampling can be removed, and (3) maximum-diameter sampling also achieves near instance-optimal guarantees.

\vspace{-0.2cm}

\paragraph{Beyond Gaussian distributions}

As it is common, e.g., in the bandit literature, the gaps in our lower and upper bounds are optimal only for Gaussian reward distributions. Extending Theorem \ref{th:instance-lb} to general distributions is actually simple (see, e.g., \cite{kaufmann2016complexity} and Lemma \ref{lem:CD_mdps} in Appendix \ref{app:lower-bounds}). However, this would yield gaps written in terms of KL divergences between arm distributions rather than differences of mean rewards as in the Gaussian case. How to match such gaps is an interesting open question.

\vspace{-0.2cm}

\paragraph{Instance optimality in stochastic MDPs}
The main open question is how to achieve (near) instance-optimality for PAC RL in stochastic MDPs. We believe that many of the results presented in this paper could help in this direction. First, our instance-dependent lower bound could be extended to the stochastic case by modifying return gaps to include visitation probabilities and minimum flows to account for stochastic navigation constraints. Second, on the algorithmic side, our maximum-coverage sampling rule easily extends to stochastic MDPs as mentioned above, while our elimination rule could also be adapted by computing the optimistic return of policies visiting a certain $(s,a,h)$ with a least some probability, which corresponds to a constrained MDP problem \citep[e.g.,][]{efroni2020exploration}. Studying how these components behave in stochastic MDPs is an exciting direction for future work.

%% file: sections/app_related.tex

\section{Additional Related Work} \label{app:related}

In this section, we mention other PAC learning problems that are related to our setting, and discuss problem-dependent guarantees for regret minimization.

\paragraph{Other PAC frameworks} The PAC RL problem, which concerns the \emph{identification} of an $\epsilon$-optimal policy, should not be confused with the PAC-MDP setting introduced by \cite{Kakade03PhD}. The latter is closer to regret minimization, as the agent interacts with the MDP online and seeks to maximize the number of learning steps where an $\epsilon$-optimal policy is played. It has been studied mostly in the discounted infinite-horizon setting. We refer the reader to \cite{strehl2009reinforcement} for a review of this setting. Other PAC pure exploration problems have been studied in the literature. In Monte-Carlo planning, the goal is to find an $\epsilon$-optimal action in a given state with high probability, rather than a complete policy. Most works have obtained worst-case upper bounds on the sample complexity of planning, scaling in terms of $\epsilon$ and some appropriate notion of near-optimality dimension \citep{Kearns02SS,bubeck2010olop,Grill16TrailBlazer}. Another line of work has derived problem-dependent guarantees in MDPs with a finite branching factor \citep{Feldman14BRUE,jonsson2020planning}, exhibiting a scaling with the value gap at step $h=1$. Finally, in reward-free exploration \citep{Jin20RewardFree,Kaufmann21RFE,Du21OptimalRFE,Menard21RFE} or task agnostic exploration \citep{Zhang20TAE} the goal is to explore the MDP in order to be able to find a near-optimal policy w.r.t. a reward function which is revealed only after the exploration phase. In the minimax sense, this setting is harder than our PAC RL problem: with an arbitray set of possible reward functions, existing algorithms exhibit an extra multiplicative dependence on the number of states in their sample complexity.

\paragraph{Instance-dependent bounds for regret minimization} The majority of instance-dependent results in the RL literature concerns regret minimization. The earliest works in this context focused on ergodic average-reward MDPs \citep{burnetas1997optimal,tewari2007opotimistic,ok2018exploration}.
These works presented asymptotic lower bounds (expressed as linear programs) on the expected regret of any ``good'' strategy, together with algorithms matching them as the number of learning steps tend to infinity. Average-reward communicating MDPs were studied by \cite{auer2008near,Tranos21}, with the latter proposing an asymptotic regret lower bound for deterministic MDPs. In the episodic setting, \cite{simchowitz2019non,xu2021fine} derived finite-time regret bounds which are roughly $O(\sum_{s,a,h} \frac{\log T}{\Delta_h(s,a)})$, where $T$ is the number of episodes and $\Delta_h(s,a)$ are the value gaps defined above. These results were later improved by \cite{dann21ReturnGap}, who derived regret bounds of the same shape but scaling with tighter ``return gaps''. Moreover, \cite{dann21ReturnGap} and \cite{tirinzoni2021fully} concurrently derived similar asymptotic instance-dependent lower bounds for the episodic setting. However, similarly to the one of \cite{al2021adaptive}, these lower bounds are written as non-convex optimization problems and it is an open question whether and how they can be matched.

\paragraph{PAC identification in structured bandits} 
Learning in a finite-horizon MDP can be seen as a \emph{structured bandit} problem with one arm for each deterministic policy whose return can be described by only $N$ parameters (the mean rewards), see, e.g., Appendix B.6 of \cite{tirinzoni2021fully}. As in our setting, instance-dependent lower bounds for structured bandits are often written as optimization problems with no closed-form solution. For this reason, the majority of (near) instance-optimal algorithms for structured bandits either repeatedly solve such an optimization problem during the learning process \citep{garivier2016optimal} or solve it incrementally using, e.g., no-regret learners \citep{degenne2019non}, primal-dual methods \citep{tirinzoni2020asymptotically}, or Frank-Wolfe  \citep{wang2021fast}. Notably, our coverage-based sampling rules achieve near instance optimality without doing any of this. This is advantageous for at least two reasons: (1) the optimization problem in \eqref{eq:global-lb} depends on unknown quantities (such as the gaps) whose estimation typically requires performing additional exploration than what prescribed by the lower bound, hence negatively affecting the sample complexity. (2) Repeatedly solving the minimum flow problem \eqref{eq:global-lb}, despite being a linear program, can be very computationally demanding, while findind an integer minimum flow or computing its greedy approximation is much more efficient. We wonder whether, taking inspiration from our work, near optimal strategies for general structured bandits could be designed without ever solving the optimization problems from lower bounds.

%% file: sections/app_flows.tex

\section{Minimum Flows and Maximum Cuts} \label{app:flows}

First note that a deterministic MDP (without reward) $\cM := (\cS, \cA, \{f_h\}_{h\in[H]}, s_1, H)$ can be represented as a \emph{layered directed acyclic graph} (DAG) $\cG(\cM) := (\cN, \cE, s_1, s_{H+1})$ with nodes $\cN := \{(s,h) : h\in[H],s\in\cS_h\}$, arcs $\cE := \{(s,a,h) : h\in[H],s\in\cS_h,a\in\cA_h(s)\}$, a unique \emph{source} node $(s_1,1)$, and a fictitious \emph{sink} node $(s_{H+1},H+1)$ which is the endpoint of every arc $(s,a,H)\in\cE$. In particular, for node $(s,h)\in\cN$, there is one arc for each $a\in\cA_h(s)$ which connects the node to $(f_h(s,a),h+1)$. The graph is \emph{layered}, in the sense that the set of nodes can be partitioned into $H$ subsets $(\{(s,h) : s\in\cS_h\})_{h\in[H]}$, one for each stage, and transitions are possible only between adjacent stages. Let $\cI_h(s) := \{(s',a')\in\cS\times\cA \mid s'\in\cS_{h-1},a'\in\cA_{h-1}(s), f_{h-1}(s',a')=s\}$ be the set of incoming arcs into $(s,h)$.

\subsection{The minimum flow problem}

In Section \ref{sec:prelim}, we introduced a specific instance of the minimum flow problem for layered DAGs with unbounded capacities. Here we introduce the general problem as described, e.g., by \cite{ciurea2004sequential}. While we still use notation for layered DAGs, we note that all results in this section hold for general directed graphs.

Recall that a \emph{flow} is a non-negative function $\eta : \cE \rightarrow [0,\infty)$ satisfying the navigation constraints whose value is given by $\varphi(\eta) := \sum_{a\in\cA_1(s_1)} \eta_1(s_1,a)$.
Let $\underline{c},\overline{c} : \cE \rightarrow [0,\infty)$ be two non-negative functions. We say that a flow $\eta$ is \emph{feasible} if
\begin{align*}
\underline{c}_h(s,a) \leq \eta_h(s,a) \leq \overline{c}_h(s,a) \quad \forall (s,a,h) \in \cE.
\end{align*}
That is, $\underline{c}_h(s,a)$ acts as a lower bound on the flow we require through arc $(s,a,h)$, while $\overline{c}_h(s,a)$ is the capacity of that arc. Finding a feasible flow of minimum value can be clearly solved as a linear program,
\begin{equation*}
	\begin{aligned}
    &\underset{\eta\in\mathbb{R}^{SAH}}{\mathrm{minimize}} \sum_{a\in\cA_1(s_1)} \eta_1(s_1,a),
    \\ & \text{subject to}
     \\
     & \quad \sum_{(s',a') \in \cI_h(s)} \eta_{h-1}(s',a') = \sum_{a\in\cA_h(s)} \eta_{h}(s,a) \quad \forall (s,h) \in \cN \setminus \{(s_1,1), (s_{H+1},H+1)\},
    \\ 
    & \quad \underline{c}_h(s,a) \leq \eta_h(s,a) \leq \overline{c}_h(s,a) \quad \forall (s,a,h) \in \cE.
	\end{aligned}
	\end{equation*}
We let $\varphi^\star(\underline{c},\overline{c})$ be its optimal value.

\paragraph{Residual graph}

The \emph{residual} of an arc $(s,a,h)\in\cE$ is defined as
\begin{align*}
\rho_h(s,a) := \eta_h(s,a) - \underline{c}_h(s,a)
\end{align*}
For each $(s,a,h)\in\cE$, we also define the residual of a fictitious backward arc (which does not exist in our layered DAG) as
\begin{align*}
\rho_h^{\mathrm{bw}}(s,a) := \overline{c}_h(s,a) - \eta_h(s,a).
\end{align*}
Then, we define the \emph{residual graph} $\cG_\eta(\cM)$ as a graph with the same nodes as $\cG(\cM)$ and one arc for each forward or backward arc of $\cG(\cM)$ with strictly positive residual. Note that, in our layered DAG setting, even if the original graph $\cG_\eta(\cM)$ has only forward arcs (transitions are only possible from two successive stages), its residual graph $\cG_\eta(\cM)$ might contain backward arcs if the fictitious backward arcs introduced before have positive residual. Intuitively, a forward arc $(s,a,h)$ in $\cG_\eta(\cM)$ means that we can decrease the flow in $(s,a,h)$ by at most $\rho_h(s,a)$ units, while its corresponding backward arc means that we can increase the flow by at most $\rho_h^{\mathrm{bw}}(s,a)$ units. Finally, we call any path from the source node $(s_1,1)$ to the sink node $(s_{H+1},H+1)$ in $\cG_\eta(\cM)$ a \emph{decreasing path}. This is a path where we can reduce the amount of flow while still satisfying all constraints.

\paragraph{Maximum cuts}

A \emph{source-sink cut} is a partition of the set of nodes $\cN$ into two subsets $\cC \subseteq \cN$ and $\cN \setminus \cC$ such that $(s_1,1) \in \cC$ and $(s_{H+1},H+1) \in \cN \setminus \cC$. As such, we will identify a cut by a single subset of states $\cC \subseteq \cN \setminus \{(s_{H+1},H+1)\}$ such that $(s_1,1) \in \cC$. The set of \emph{forward arcs} of a cut $\cC$ is
\begin{align*}
\cE(\cC) := \{(s,a,h)\in\cE : (s,h) \in \cC, (f_h(s,a),h+1) \in \cN \setminus \cC\}.
\end{align*}
The value of a cut $\cC$, as defined by \cite{ciurea2004sequential}, is
\begin{align*}
\psi(\cC, \underline{c}, \overline{c}) := \sum_{(s,a,h) \in \cE(\cC)} \underline{c}_h(s,a) - \sum_{(s,a,h) \in \cE_{\mathrm{bw}}(\cC)} \overline{c}_h(s,a),
\end{align*}
where we also define $\cE_{\mathrm{bw}}(\cC) := \{(s,a,h) \in\cE : (s,h) \in \cN \setminus \cC, (f_h(s,a),h+1) \in \cC \}$ as the set of \emph{backward arcs} in the cut. We now present an important result (Lemma \ref{lem:flow-vs-cut}) which shows that the value of any feasible flow is at least the value of any cut. Its proof is based on the following lemma.

\begin{lemma}\label{lem:flow-in-out-cut}
Let $\eta$ be any flow (not necessarily feasible) and $\cC$ be any source-sink cut. Then,
\begin{align*}
\varphi(\eta) = \sum_{(s,a,h)\in\cE(\cC)} \eta_h(s,a) - \sum_{(s,a,h)\in\cE_{\mathrm{bw}}(\cC)} \eta_h(s,a).
\end{align*}
\end{lemma}
\begin{proof}
We have
\begin{align*}
\varphi(\eta) 
&\stackrel{(a)}{=} \sum_{a\in\cA_1(s_1)} \eta_1(s_1,a)
\\ &\stackrel{(b)}{=} \sum_{(s,h)\in\cC} \left( \sum_{a\in\cA_h(s)} \eta_h(s,a) - \sum_{(s',a')\in\cI_h(s)} \eta_{h-1}(s',a') \right)
\\ &\stackrel{(c)}{=} \sum_{(s,a,h)\in\cE(\cC)} \eta_h(s,a) - \sum_{(s,a,h)\in\cE_{\mathrm{bw}}(\cC)} \eta_h(s,a).
\end{align*}
where (a) is from the definition of flow value, (b) uses the balance constraints together with $(s_1,1)\in\cC$, and (c) uses that the flow through every arc with both endpoints in $\cC$ cancels since it appears once in the left term and once in the right one.
\end{proof}

\begin{lemma}\label{lem:flow-vs-cut}
Let $\eta$ be any feasible flow and $\cC$ be any source-sink cut. Then,
\begin{align*}
\varphi(\eta) \geq \psi(\cC, \underline{c}, \overline{c}).
\end{align*}
\end{lemma}
\begin{proof}
We have
\begin{align*}
\varphi(\eta) 
&\stackrel{(a)}{=} \sum_{(s,a,h)\in\cE(\cC)} \eta_h(s,a) - \sum_{(s,a,h)\in\cE_{\mathrm{bw}}(\cC)} \eta_h(s,a)
\\ &\stackrel{(b)}{\geq} \sum_{(s,a,h)\in\cE(\cC)} \underline{c}_h(s,a) - \sum_{(s,a,h)\in\cE_{\mathrm{bw}}(\cC)} \overline{c}_h(s,a)
\\ &\stackrel{(c)}{=} \psi(\cC, \underline{c}, \overline{c}),
\end{align*}
where (a) follows from Lemma \ref{lem:flow-in-out-cut}, (b) uses the feasibility constraints, and (c) uses the definition of value of a cut.
\end{proof}
Thanks to Lemma \ref{lem:flow-vs-cut}, we know that, if we find a flow and a cut whose values coincide, then it must be that we found a minimum flow and its value coincides with the one of a maximum cut. This is what is shown in the next theorem.

\begin{theorem}[Theorem 1.1 of \cite{ciurea2004sequential}]\label{th:min-flow-max-cut-app}
If there exists a feasible flow, the value of the minimum flow with non-negative lower bounds $\underline{c}$ equals the value of the maximum source-sink cut.
\end{theorem}

\begin{theorem}[Theorem 1.2 of \cite{ciurea2004sequential}]\label{th:min-flow-no-path}
A feasible flow $\eta$ is minimum if, and only if, the residual graph $\cG_\eta(\cM)$ contains not decreasing path (i.e., no path from source to sink).
\end{theorem}
Note that Theorem \ref{th:min-flow-max-cut-app} is the equivalent of Theorem \ref{th:min-flow-max-cut} stated in Section \ref{sec:prelim} for DAGs with unbounded capacities. We prove both theorems in the following unified proof.
\begin{proof}[Proof of Theorem \ref{th:min-flow-max-cut-app} and Theorem \ref{th:min-flow-no-path}]
Let $\eta$ be a feasible flow. We prove both theorems by showing that the following three statements are equivalent:
\begin{enumerate}
\item there exists a cut $\cC$ such that $\varphi(\eta) = \psi(\cC, \underline{c}, \overline{c})$;
\item $\eta$ is a minimum flow;
\item there is no path from source to sink in $\cG_\eta(\cM)$.
\end{enumerate}
Note that, by Lemma \ref{lem:flow-vs-cut} we clearly have that $1 \implies 2$. In fact, if a stricly better flow that $\eta$ existed, call it $\eta'$, then we would have
\begin{align*}
\varphi(\eta') < \varphi(\eta) = \psi(\cC, \underline{c}, \overline{c}),
\end{align*}
which is a contradiction since $\varphi(\eta') \geq \psi(\cC, \underline{c}, \overline{c})$ by Lemma \ref{lem:flow-vs-cut}.

Clearly, $2 \implies 3$ by definition of decreasing path. If a path from source to sink existed in $\cG_\eta(\cM)$, then we could decrease the flow along it while still satisfying all constraints. Hence, $\eta$ would not be a minimum flow, which is a contradiction.

It remains to prove that $3 \implies 1$. We do so by building an explicit cut $\cC$ from the residual graph $\cG_\eta(\cM)$ which satifies property $1$. This uses the same construction as in the well-known proof of the max-flow-min-cut theorem. Suppose that $\eta$ is a feasible flow with no decreasing paths in $\cG_\eta(\cM)$. Let $\cC$ be the set of nodes that are reachable from $(s_1,1)$ in $\cG_\eta(\cM)$. It must be that $(s_1,1) \in \cC$ and $(s_{H+1},H+1) \notin \cC$ since the sink node is unreachable from the source node in $\cG_\eta(\cM)$. Therefore, $\cC$ is a valid cut. From Lemma \ref{lem:flow-in-out-cut}, we know that
\begin{align*}
\varphi(\eta) = \sum_{(s,a,h)\in\cE(\cC)} \eta_h(s,a) - \sum_{(s,a,h)\in\cE_{\mathrm{bw}}(\cC)} \eta_h(s,a).
\end{align*}
It only remains to prove that $\eta_h(s,a) = \underline{c}_h(s,a)$ for all $(s,a,h)\in\cE(\cC)$ and $\eta_h(s,a) = \overline{c}_h(s,a)$ for all $(s,a,h)\in\cE_{\mathrm{bw}}(\cC)$. 

Take any $(s,a,h)\in\cE(\cC)$. Since $(s,h) \in \cC$ and $(f_h(s,a),h+1) \notin \cC$, we must have that the forward arc $(s,a,h)$ does not belong to $\cG_\eta(\cM)$. But this means that $\rho_h(s,a) = 0$, which in turns implies that $\eta_h(s,a) = \underline{c}_h(s,a)$. This proves the first claim.

Now take any $(s,a,h)\in\cE_{\mathrm{bw}}(\cC)$. Here we have the opposite situation: $(f_h(s,a),h+1) \in \cC$ but $(s,h) \notin \cC$. This means that there is no arc from the first node to the second in $\cG_\eta(\cM)$. But if $\eta_h(s,a) < \overline{c}_h(s,a)$ we would have $\rho_h^{\mathrm{bw}}(s,a) > 0$ and thus there would be a backward arc between those two nodes. This is a contradiction, and thus it must be that $\eta_h(s,a) = \overline{c}_h(s,a)$. This concludes the proof of $3 \implies 1$, which in turns proves the main theorems.
\end{proof}

\subsection{Layered DAGs with unlimited capacity}

In all our applications, we will consider DAGs with unlimited capacity, i.e., $\overline{c}_h(s,a) = \infty$ for all $s\in\cS,a\in\cA,h\in[H]$. In this case, some of the previously-introduced quantities can be simplified using the notation adopted in Section \ref{sec:prelim}. First, we can simplify the notation for a minimum flow $\varphi^\star(\underline{c})$ and for the value of a cut $\psi(\cC, \underline{c})$ by dropping $\overline{c}$. The upper-bound constraints in the definition of feasible flow and in the LP can be simply dropped. Moreover, backward arcs in the residual graph have always residual equal to $\infty$. This means that backward arcs are always present in the residual graph, which has the intuitive meaning that we can always arbitrarily increase the flow along each forward arc in the original graph.

Now note that, by definition of value of a cut, if a cut $\cC$ contains an available backward arc (i.e., one of the arcs in the original graph connects a node outside the cut with a node inside the cut), its value is $-\infty$. Therefore, if a feasible flow exists (whose value must be non-negative), then, by Theorem \ref{th:min-flow-max-cut}, we know that a cut $\cC$ with backward arcs cannot be a maximum cut.
Therefore, we can define the set of \emph{valid cuts} $\mathfrak{C}$ as
\begin{align*}
\mathfrak{C} := \left\{ \cC \subseteq \cN \setminus \{(s_{H+1},H+1)\} \mid (s_1,1) \in \cC,  \cE_{\mathrm{bw}}(\cC) = \emptyset \right\}.
\end{align*}
Then, for all $\cC \in \mathfrak{C}$, we clearly have $\psi(\cC, \underline{c}) = \sum_{(s,a,h) \in \cE(\cC)} \underline{c}_h(s,a)$ since there is no backward arc. Moreover, by Theorem \ref{th:min-flow-max-cut-app} together with the fact that cuts not belonging to $\mathfrak{C}$ cannot be maximizers,
\begin{align*}
\varphi^\star(\underline{c}) = \max_{\cC \in \mathfrak{C}} \psi(\cC, \underline{c}) = \max_{\cC \in \mathfrak{C}} \sum_{(s,a,h) \in \cE(\cC)} \underline{c}_h(s,a).
\end{align*}
We formally state this result in the following theorem.
\begin{theorem}\label{th:min-flow-max-cut}
    Consider a layered DAG with unlimited capacity. If there exists a feasible flow,
    \begin{align*}
    \varphi^\star(\underline{c}) = \max_{\cC \in \mathfrak{C}} \psi(\cC, \underline{c}).
    \end{align*}
 \end{theorem}

\paragraph{Useful properties}

We prove some simple properties of flows and cuts which will be useful later on.

\begin{lemma}[Monotonicity]\label{lem:flow-mono}
Let $\underline{c}^1,\underline{c}^2 : \cE \rightarrow [0,\infty)$ be such that $\underline{c}^1_h(s,a) \leq \underline{c}^2_h(s,a)$ for all $(s,a,h)\in\cE$. Then,
\begin{align*}
\varphi^\star(\underline{c}^1) \leq \varphi^\star(\underline{c}^2). 
\end{align*}
\end{lemma}
\begin{proof}
This can be immediately seen from the LP formulation: any feasible flow $\eta$ for $\underline{c}^2$ is also feasible for $\underline{c}^1$.
\end{proof}

\begin{lemma}[Flow bounds]\label{lem:flow-bounds}
For any lower bound function $\underline{c}$,
\begin{align*}
\max_{h\in[H]} \sum_{s\in\cS_h}\sum_{a\in\cA_h(s)} \underline{c}_h(s,a) \leq \varphi^\star(\underline{c}) \leq \sum_{h\in[H]} \sum_{s\in\cS_h}\sum_{a\in\cA_h(s)} \underline{c}_h(s,a). 
\end{align*}
\end{lemma}
\begin{proof}
Both inequalities are easy to see from Theorem \ref{th:min-flow-max-cut} and the definition of value of a cut. The upper bound is trivial since $\varphi^\star(\underline{c}) = \max_{\cC \in \mathfrak{C}} \sum_{(s,a,h) \in \cE(\cC)} \underline{c}_h(s,a)$ and the set of outgoing arcs $\cE(\cC)$ from a cut $\cC$ can contain at most all possible arcs $\cE$. To see the lower bound, note that $\cC_h := \{s\in\cS_l : l \leq h\}$ is a valid cut for any $h\in[H]$ whose outgoing arcs are all those connecting states at stage $h$ with states at stage $h+1$. Thus,
\begin{align*}
\varphi^\star(\underline{c}) = \max_{\cC \in \mathfrak{C}} \sum_{(s,a,h) \in \cE(\cC)} \underline{c}_h(s,a) \geq \max_{h\in[H]} \sum_{(s,a,h) \in \cE(\cC_h)} \underline{c}_h(s,a) = \max_{h\in[H]} \sum_{s\in\cS_h}\sum_{a\in\cA_h(s)} \underline{c}_h(s,a).
\end{align*}
\end{proof}

\begin{lemma}\label{lem:flow-of-sum}
Let $\underline{c}^1,\underline{c}^2$ be two non-negative lower bound functions and $\alpha > 0$. Then,
\begin{align*}
\varphi^\star(\alpha \underline{c}^1 + \underline{c}^2) \leq \alpha\varphi^\star( \underline{c}^1)  + \varphi^\star(\underline{c}^2).
\end{align*}
\end{lemma}
\begin{proof}
We first prove that $\varphi^\star(\alpha \underline{c}^1) = \alpha\varphi^\star( \underline{c}^1)$. This can be easily seen from the linear programming formulation stated above: by performing the change of variables $\eta_h'(s,a) = \eta_h(s,a)/\alpha$, we obtain exactly $\alpha$ times the value of a minimum flow with lower bound function $\underline{c}^1$. Next, we prove that $\varphi^\star( \underline{c}^1 + \underline{c}^2) \leq \varphi^\star( \underline{c}^1)  + \varphi^\star(\underline{c}^2)$. Let $\eta^1$ and $\eta^2$ be minimum flows for the problems with lower bounds $\underline{c}^1$ and $\underline{c}^2$, respectively. The proof follows by noting that $\eta := \eta^1 + \eta^2$ is a feasible flow for the problem with lower bound function $\underline{c}^1 + \underline{c}^2$ and that is value is exactly $\varphi(\eta) = \varphi(\eta^1) + \varphi(\eta^2) = \varphi^\star(\underline{c}^1)  + \varphi^\star(\underline{c}^2)$. Combining these two results concludes the proof.
\end{proof}

\subsection{Minimum flows and minimum policy covers}\label{app:min-cover}

A crucial problem in the analysis of our sampling rules is the problem of computing a \emph{minimum policy cover}. Formally, given a subset $\cE' \subseteq \cE$ of the arcs (i.e., of the state-action-stage triplets), the goal is to find a set of policies $\Pi_{\mathrm{cover}} \subseteq \Pi$ of \emph{minimum size} such that
\begin{align*}
\forall (s,a,h) \in \cE', \exists \pi \in \Pi_{\mathrm{cover}} : (s_h^\pi,a_h^\pi) = (s,a).
\end{align*}
That is, $\Pi_{\mathrm{cover}}$ is the smallest set of policies that, played together, visit all arcs in $\cE'$. This problem can be easily reduced to a minimum flow problem with lower bound function
\begin{align*}
\underline{c}_h(s,a) := \indi{(s,a,h)\in\cE'},
\end{align*}
which intuitively demands at least one visit to all $(s,a,h)\in\cE'$, and zero visits from the other triplets. Moreover, since $\underline{c}$ is integer-valued, an integer minimum flow exists which can be computed by existing algorithms \citep[e.g.,][]{brandizi2012graph2tab}. Suppose that $\eta$ is one such integer minimum flow. A policy cover can be easily extracted from it by the procedure shown in Algorithm \ref{alg:extract-cover}, which is similar to the method proposed by  \cite{brandizi2012graph2tab} to obtain a minimum path cover in a layered DAG.

\begin{algorithm}[t]
\caption{Extract policy cover from minimum flow}\label{alg:extract-cover}
\begin{algorithmic}
\STATE \textbf{Input:} deterministic MDP (without reward) $\cM := (\cS, \cA, \{f_h\}_{h\in[H]}, s_1, H)$, feasible integer flow $\eta$
\STATE Initialize $\Pi_{\mathrm{cover}} \leftarrow \emptyset$
\WHILE{$\varphi(\eta) > 0$}
\STATE Initialize a policy $\pi$ with arbitrary actions
\FOR{$h=1,\dots,H$}
\STATE $\pi_h(s_h) \leftarrow \argmax_{a\in\cA_h(s_h)}\eta_h(s,a)$
\STATE $\eta_h(s_h,\pi_h(s_h)) \leftarrow \eta_h(s_h,\pi_h(s_h)) - 1$
\STATE $s_{h+1} \leftarrow f_h(s_h,\pi_h(s_h))$
\ENDFOR
\STATE $\Pi_{\mathrm{cover}} \leftarrow \Pi_{\mathrm{cover}} \cup \{\pi\}$
\ENDWHILE
\end{algorithmic}
\end{algorithm}

\paragraph{Eliminating arcs}

In our applications, we compute minimum policy covers for subsets of the arcs $\cE'$ that contain only non-eliminated actions. One natural question is: what if, instead of setting the lower bound function to zero for arcs not in $\cE'$, we use a lower bound function that is uniformly equal to one but solve the minimum flow problem on a sub-graph with only arcs in $\cE'$ available? One argument against this idea is that the resulting minimum policy cover might have a strictly larger size. Figure \ref{fig:cover-elim} shows an example. Suppose that we want to compute a minimum policy cover visiting all arcs except $\bar{a}$. Then, if we set the lower bound function for $\bar{a}$ to zero and solve the minimum flow problem on the full MDP, we get an optimal value of $1+m$ by sending a flow of one on each action in $s_1$ and then redirecting $m-1$ flow to $\bar{a}$. If instead we make $\bar{a}$ unavailable, we cannot do this trick and the optimal flow becomes $2m$.

\begin{figure}
\centering
\includegraphics[scale=0.8]{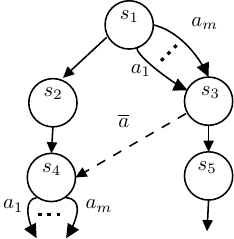}
\caption{Example to show why eliminating actions from the sampling rule is not a good idea. The size of a minimum policy cover is $m+1$ when $\bar{a}$ is used and $2m$ when it is unavailable.}\label{fig:cover-elim}
\end{figure}

%% file: sections/app_lower_bounds.tex

\newpage 
\section{Lower Bounds}\label{app:lower-bounds}

We first present an important result to derive lower bounds, a change-of-distribution lemma which is a direct instantiation of Lemma 1 in \cite{kaufmann2016complexity} (see also Lemma 8 in \cite{tirinzoni2021fully} and references therein).

\begin{lemma}\label{lem:CD_mdps} Let $\cM$ and $\widetilde{\cM}$ be two MDPs with identical state-action space and deterministic transitions but possibly different
rewards distributions denoted by $(\nu_h^{\cM}(s,a))_{s,a,h}$ and $(\nu_h^{\widetilde{\cM}}(s,a))_{s,a,h}$ respectively. For every algorithm, every stopping time $\tau$ and every event $\cE \in \cF_{\tau}$, it holds that 
\[\sum_{h=1}^{H}\sum_{s \in \cS_h} \sum_{a \in \cA_h} \bE_{\cM}[n_{h}^{\tau}(s,a)] \mathrm{KL}\left(\nu_h^{\cM}(s,a),\nu_h^{\widetilde{\cM}}(s,a)\right)\geq \kl\left(\bP_{\cM}(\cE),\bP_{\widetilde{\cM}}(\cE)\right)\]
where $\KL$ denotes the Kullback-Leibler divergence and $\kl(x,y) = x\log(x/y) + (1-x)\log((1-x)/(1-y))$ the binary relative entropy.
\end{lemma}

If the rewards follow Gaussian distributions with variance $\sigma^2$, we have \[\mathrm{KL}\left(\nu_h^{\cM}(s,a),\nu_h^{\widetilde{\cM}}(s,a)\right) = \frac{\left(r_h^{\cM}(s,a) - r_h^{\widetilde{\cM}}(s,a)\right)^2}{2\sigma^2}.\] 

\subsection{Instance-dependent lower bound}\label{proof:LB_instance}

In this section, we prove Theorem \ref{th:instance-lb}. We first state and prove three lemmas which bound the local number of visits to different $(s,a,h)$. Then, we combine them to prove the main result.

\begin{lemma}[Lower bound for sub-optimal pairs]\label{lem:lb-suboptimal}
For any $h\in[H]$ and any non-$\epsilon$-optimal pair $(s,a) \notin \cZ_h^\epsilon$,
\begin{align*}
    \bE[n_{h}^{\tau}(s, a)] \geq \frac{2\hl{\sigma^2}\log(1/3\delta)}{(\overline{\Delta}_h(s,a) + \epsilon)^2}.
\end{align*}
\end{lemma}
\begin{proof}
Consider the alternative MDP $\widetilde{\cM} := (\cS, \cA, \{f_h,\widetilde{\nu}_h\}_{h\in[H]}, s_1, H)$ which is equivalent to $\cM$ except that the reward is modified only at $(s,a,h)$ as $\widetilde{\nu}_h(s,a) = \cN(r_h(s,a) + \Delta, \hl{\sigma^2})$ with $\Delta > \overline{\Delta}_h(s,a) + \epsilon$, while $\widetilde{\nu}_{h'}(s',a') = {\nu}_{h'}(s',a')$ on all other state-action-stage triplets. It is easy to see that, for any $\pi\in\Pi_{s,a,h}$,
\begin{align*}
\widetilde{V}_1^\pi(s_1) = \sum_{l=1}^H \widetilde{r}_l(s_l^\pi,a_l^\pi) = \sum_{l=1}^H {r}_l(s_l^\pi,a_l^\pi) + \Delta = V_1^\pi(s_1) + \Delta.
\end{align*}
Similarly, for any policy $\pi\notin\Pi_{s,a,h}$, we have $\widetilde{V}_1^{\pi}(s_1) = {V}_1^{\pi}(s_1)$.
Let $ \pi^0 \in \argmax_{\pi\in \Pi_{s,a,h}} V_1^\pi(s_1)$. Then,
\begin{align*}
   \widetilde{V}_1^{\pi^0}(s_1) 
   &= V_1^{\pi^0}(s_1)+ {\Delta}
   \\ &> V_1^{\pi^0}(s_1)+ \overline{\Delta}_h(s,a)+\epsilon\\
   &= V_1^{\star}(s_1)+\epsilon \\
   &\ge \max_{\pi \notin \Pi_{s,a,h}} V_1^{\pi}(s_1)+\epsilon
   \\ &= \max_{\pi \notin \Pi_{s,a,h}} \widetilde{V}_1^{\pi}(s_1)+\epsilon.
\end{align*}
This means that\footnote{\blue{Recall that we use $\widehat{\pi}$ to denote the policy returned by the recommendation rule of the algorithm}} $\bP_{\widetilde{\cM}}(\widehat{\pi} \in \Pi_{s,a,h}) \geq 1-\delta$. Moreover, $\bP_{{\cM}}(\widehat{\pi} \in \Pi_{s,a,h}) \leq \delta$ since $(s,a,h)$ is not visited by any $\epsilon$-optimal policy. Therefore, Lemma~\ref{lem:CD_mdps} implies that
\begin{align*}
   \bE[n_h^\tau(s,a)] 
   \geq \frac{2\hl{\sigma^2}}{\Delta^2} \mathrm{kl}(\bP_{{\cM}}(\widehat{\pi} \in \Pi_{s,a,h}),\bP_{\widetilde{\cM}}(\widehat{\pi} \in \Pi_{s,a,h}))
   \geq \frac{2\hl{\sigma^2}}{\Delta^2}\mathrm{kl}(\delta,1-\delta) 
   \geq \frac{2\hl{\sigma^2}}{\Delta^2} \log(1/3\delta).
\end{align*}
This holds for any $\Delta > \overline{\Delta}_h(s,a) + \epsilon$ and the proof is concluded by taking the limit.
\end{proof}

\begin{lemma}[Lower bound for non-unique $\epsilon$-optimal pairs]\label{lem:lb-optimal-nonunique}
For any $h\in[H]$ and any $\epsilon$-optimal pair $(s,a) \in \cZ_h^\epsilon$, if $|\cZ_h^\epsilon| > 1$,
\begin{align*}
    \bE[n_{h}^{\tau}(s, a)] \geq \frac{\hl{\sigma^2}\log(1/4\delta)}{4\epsilon^2}.
\end{align*}
\end{lemma}
\begin{proof}
Take any pair $(s,a)$ in $\cZ_h^\epsilon$. We distinguish two cases.

\paragraph{Case 1: $\bP_{\cM}(\widehat{\pi} \in \Pi_{s,a,h}) \leq 1/2$.}

We can build the same alternative MDP $\widetilde{\cM}$ as in the proof of Lemma \ref{lem:lb-suboptimal}, for which we have $\bP_{\widetilde{\cM}}(\widehat{\pi} \in \Pi_{s,a,h}) \geq 1-\delta$. Thus, using Lemma~\ref{lem:CD_mdps},
\begin{align*}
   \bE[n_h^\tau(s,a)] 
   &\geq \frac{2\hl{\sigma^2}}{\Delta^2} \mathrm{kl}(\bP_{{\cM}}(\widehat{\pi} \in \Pi_{s,a,h}),\bP_{\widetilde{\cM}}(\widehat{\pi} \in \Pi_{s,a,h}))
   \\ &\geq \frac{2\hl{\sigma^2}}{\Delta^2}\mathrm{kl}(1/2,1-\delta)
    \\ &= \frac{2\hl{\sigma^2}}{\Delta^2}\mathrm{kl}(1/2,\delta)  
   \\ &\geq \frac{\hl{\sigma^2}}{\Delta^2} \log(1/4\delta),
\end{align*}
where we used the fact that $\mathrm{kl}(x,y) = \mathrm{kl}(1-x,1-y)$ and $\mathrm{kl}(x,y) \geq x\ln(1/y) - \ln(2)$. By taking the limit $\Delta \rightarrow \overline{\Delta}_h(s,a) + \epsilon$ and using $\overline{\Delta}_h(s,a) \leq \epsilon$, we thus conclude that $\bE[n_h^\tau(s,a)] \geq \frac{\hl{\sigma^2}\log(1/4\delta)}{(\overline{\Delta}_h(s,a) + \epsilon)^2} \geq \frac{\hl{\sigma^2}\log(1/4\delta)}{4\epsilon^2}$. 

\paragraph{Case 2: $\bP_{\cM}(\widehat{\pi} \in \Pi_{s,a,h}) > 1/2$.}

Consider the alternative MDP $\widetilde{\cM} := (\cS, \cA, \{f_h,\widetilde{\nu}_h\}_{h\in[H]}, s_1, H)$ which is equivalent to $\cM$ except that the reward is modified only at $(s,a,h)$ as $\widetilde{\nu}_h(s,a) = \cN(r_h(s,a) - \Delta, \hl{\sigma^2})$ with $\Delta > 2\epsilon - \overline{\Delta}_h(s,a)$, while $\widetilde{\nu}_{h'}(s',a') = {\nu}_{h'}(s',a')$ on all other state-action-stage triplets. Then, for $ \pi^0 \in \underset{\pi\in \Pi\setminus\Pi_{s,a,h}}{\argmax} V_1^\pi(s_1)$,
\begin{align*}
\widetilde{V}_1^{\pi^0}(s_1) 
=  V_1^{\pi^0}(s_1) 
&\geq V_1^\star(s_1) - \epsilon \pm \max_{\pi\in\Pi_{s,a,h}}{V}_1^{\pi}(s_1)
\\ &= \overline{\Delta}_h(s,a) - \epsilon + \max_{\pi\in\Pi_{s,a,h}}{V}_1^{\pi}(s_1)
\\ &= \overline{\Delta}_h(s,a) - \epsilon + \max_{\pi\in\Pi_{s,a,h}}\widetilde{V}_1^{\pi}(s_1) + \Delta > \max_{\pi\in\Pi_{s,a,h}}\widetilde{V}_1^{\pi}(s_1) + \epsilon,
\end{align*}
where the first inequality is due to the fact that, since $\cZ_h^\epsilon > 1$, there exists at least one $\epsilon$-optimal policy which does not visit $(s,a)$ at step $h$ (i.e., which belongs to $\Pi\setminus\Pi_{s,a,h}$).
This implies that $\bP_{\widetilde{\cM}}(\widehat{\pi} \in \Pi_{s,a,h}) \leq \delta$. Thus, using Lemma~\ref{lem:CD_mdps},
\begin{align*}
   \bE[n_h^\tau(s,a)] 
   &\geq \frac{2\hl{\sigma^2}}{\Delta^2} \mathrm{kl}(\bP_{{\cM}}(\widehat{\pi} \in \Pi_{s,a,h}),\bP_{\widetilde{\cM}}(\widehat{\pi} \in \Pi_{s,a,h}))
   \geq \frac{2\hl{\sigma^2}}{\Delta^2}\mathrm{kl}(1/2,\delta) 
   \geq \frac{\hl{\sigma^2}}{\Delta^2} \log(1/4\delta).
\end{align*}
By taking the limit $\Delta \rightarrow 2\epsilon - \overline{\Delta}_h(s,a)$, we thus conclude that $\bE[n_h^\tau(s,a)] \geq \frac{\hl{\sigma^2}\log(1/4\delta)}{(2\epsilon - \overline{\Delta}_h(s,a))^2} \geq \frac{\hl{\sigma^2}\log(1/4\delta)}{4\epsilon^2}$. 
\end{proof}

\begin{lemma}[Lower bound for unique $\epsilon$-optimal pairs]\label{lem:lb-optimal-unique}
For any $h\in[H]$ and any $\epsilon$-optimal pair $(s,a) \in \cZ_h^\epsilon$, if $|\cZ_h^\epsilon| = 1$,
\begin{align*}
   \bE[n_h^\tau(s,a)] \geq \frac{2\hl{\sigma^2}\log(1/3\delta)}{(\overline{\Delta}_{\min}^h + \epsilon)^2},
\end{align*}
where $\overline{\Delta}_{\min}^h := \min_{(s',a') : \overline{\Delta}_h(s',a') > 0} \overline{\Delta}_h(s',a')$.
\end{lemma}
\begin{proof}
Note that, since $(s,a) \in \cZ_h^\epsilon$ and $|\cZ_h^\epsilon| = 1$, then $\Pi^\epsilon \cap (\Pi\setminus\Pi_{s,a,h}) = \emptyset$ (i.e., all $\epsilon$-optimal policies visit $(s,a,h)$). Therefore, $\bP_{\cM}(\widehat{\pi} \in \Pi_{s,a,h}) \geq 1-\delta$. We now use a construction similar to the one in Case 2 of the proof of Lemma \ref{lem:lb-optimal-nonunique}.

Consider the alternative MDP $\widetilde{\cM} := (\cS, \cA, \{f_h,\widetilde{\nu}_h\}_{h\in[H]}, s_1, H)$ which is equivalent to $\cM$ except that the reward is modified only at $(s,a,h)$ as $\widetilde{\nu}_h(s,a) = \cN(r_h(s,a) - \Delta, \hl{\sigma^2})$ with $\Delta > \max_{\pi\in\Pi_{s,a,h}}{V}_1^{\pi}(s_1) - \max_{\pi\in \Pi\setminus\Pi_{s,a,h}}V_1^\pi(s_1) + \epsilon$, while $\widetilde{\nu}_{h'}(s',a') = {\nu}_{h'}(s',a')$ on all other state-action-stage triplets. Then, for $ \pi^0 \in \underset{\pi\in \Pi\setminus\Pi_{s,a,h}}{\argmax} V_1^\pi(s_1)$,
\begin{align*}
\widetilde{V}_1^{\pi^0}(s_1) 
=  V_1^{\pi^0}(s_1) 
&= \max_{\pi\in \Pi\setminus\Pi_{s,a,h}}V_1^\pi(s_1) \pm \max_{\pi\in\Pi_{s,a,h}}{V}_1^{\pi}(s_1)
\\ &= \max_{\pi\in \Pi\setminus\Pi_{s,a,h}}V_1^\pi(s_1) - \max_{\pi\in\Pi_{s,a,h}}{V}_1^{\pi}(s_1) + \max_{\pi\in\Pi_{s,a,h}}\widetilde{V}_1^{\pi}(s_1) + \Delta
\\ &> \max_{\pi\in\Pi_{s,a,h}}\widetilde{V}_1^{\pi}(s_1) + \epsilon.
\end{align*}
This implies that $\bP_{\widetilde{\cM}}(\widehat{\pi} \in \Pi_{s,a,h}) \leq \delta$. Thus, applying Lemma~\ref{lem:CD_mdps},
\begin{align*}
   \bE[n_h^\tau(s,a)] 
   &\geq \frac{2\hl{\sigma^2}}{\Delta^2} \mathrm{kl}(\bP_{{\cM}}(\widehat{\pi} \in \Pi_{s,a,h}),\bP_{\widetilde{\cM}}(\widehat{\pi} \in \Pi_{s,a,h}))
   \geq \frac{2\hl{\sigma^2}}{\Delta^2}\mathrm{kl}(1-\delta,\delta) 
   \geq \frac{2\hl{\sigma^2}}{\Delta^2} \log(1/3\delta).
\end{align*}
Now note that, since an optimal policy belongs to $\Pi_{s,a,h}$,
\begin{align*}
\max_{\pi\in\Pi_{s,a,h}}{V}_1^{\pi}(s_1) - \max_{\pi\in \Pi\setminus\Pi_{s,a,h}}V_1^\pi(s_1) 
& = V_1^\star(s_1) - \max_{\pi\in \Pi\setminus\Pi_{s,a,h}}V_1^\pi(s_1)
\\ &= V_1^\star(s_1) - \max_{s'\in\cS_h}\max_{a'\in\cA_h(s') : (s',a') \neq (s,a)} \max_{\pi\in\Pi_{s',a',h}}V_1^\pi(s_1)
\\ & = \min_{(s',a') : \overline{\Delta}_h(s',a') > 0} \overline{\Delta}_h(s',a') = \overline{\Delta}_{\min}^h.
\end{align*}
By taking the limit $\Delta \rightarrow \overline{\Delta}_{\min}^h + \epsilon$, we conclude that $\bE[n_h^\tau(s,a)] \geq \frac{2\hl{\sigma^2}\log(1/3\delta)}{(\overline{\Delta}_{\min}^h + \epsilon)^2}$. 
\end{proof}

We are now ready to prove the main theorem.

\begin{proof}[Proof of Theorem \ref{th:instance-lb}]
The first statement follows easily from Lemma \ref{lem:lb-suboptimal}, Lemma \ref{lem:lb-optimal-nonunique}, and Lemma \ref{lem:lb-optimal-unique}. In fact, for $(s,a)\notin \cZ_h^\epsilon$, Lemma \ref{lem:lb-suboptimal} yields
\begin{align*}
\bE[n_h^\tau(s,a)] \geq \frac{2\hl{\sigma^2}\log(1/3\delta)}{(\overline{\Delta}_h(s,a) + \epsilon)^2} \geq \frac{\hl{\sigma^2}\log(1/3\delta)}{2\max(\overline{\Delta}_h(s,a), \epsilon)^2} = \frac{\hl{\sigma^2}\log(1/3\delta)}{2\max(\overline{\Delta}_h(s,a), \overline{\Delta}_{\min}^h, \epsilon)^2}
\end{align*}
since $\overline{\Delta}_h(s,a) \leq \overline{\Delta}_{\min}^h$. For $(s,a)\in\cZ_h^\epsilon$ when $|\cZ_h^\epsilon| > 1$, the result follows trivially from Lemma \ref{lem:lb-optimal-nonunique} by noting that $\overline{\Delta}_{\min}^h = 0$ and $\overline{\Delta}_h(s,a) \leq \epsilon$. For $(s,a)\in\cZ_h^\epsilon$ when $|\cZ_h^\epsilon| = 1$, using Lemma \ref{lem:lb-optimal-unique} with $\overline{\Delta}_h(s,a) \leq \epsilon$,
\begin{align*}
   \bE[n_h^\tau(s,a)] \geq \frac{2\hl{\sigma^2}\log(1/3\delta)}{(\overline{\Delta}_{\min}^h + \epsilon)^2} \geq \frac{\hl{\sigma^2}\log(1/3\delta)}{2\max(\overline{\Delta}_{\min}^h, \epsilon)^2} = \frac{\hl{\sigma^2}\log(1/3\delta)}{2\max(\overline{\Delta}_h(s,a), \overline{\Delta}_{\min}^h, \epsilon)^2}.
\end{align*}
To prove the second statement, note that the visitation counts $n_h^\tau(s,a)$ form a feasible flow. Therefore, given local lower bounds $\underline{c}_h(s,a)$ for each state-action-stage triplet $(s,a,h)$, the minimum expected stopping when satisfying all such lower bounds is exactly the value of the minimum flow $\varphi^\star(\underline{c})$.
\end{proof}

\subsection{Worst-case lower bound}\label{app:LB_minimax}

In order to derive a worst-case lower bound, we build a deterministic variant of the hard MDP instance introduced by \cite{Omar21LB} for the time-inhomogeneous stochastic setting. A major complication is that stochasticity plays a crucial role for obtaining the right minimax dependence on the horizon $H$ in the latter context. Here we present a different analysis where we shall achieve the optimal dependence by leveraging the theory of minimum flows.

An example of our hard instance is shown in Figure \ref{fig:hard-instance}. Fix some $S\geq 2$, $A \geq 2$, and $H \geq 3\log_2(S)$. As common in existing worst-case lower bounds for MDPs \citep{BanditBook,Omar21LB}, we arrange $S-1$ states in a full binary tree. As such, we will assume that $S-1 = \sum_{i=0}^{d-1} 2^i = 2^d - 1$ for some integer $d\geq 1$ which represents the depth of the tree. The condition $H \geq 3\log_2(S)$ is to make sure that there are enough stages to build the binary tree. These assumptions are made only to simplify the exposition, while our result can be easily generalized to any number of states and stages by considering non-complete binary trees (see also Appendix D of \cite{Omar21LB}).

\begin{figure}
\centering
\includegraphics[scale=0.65]{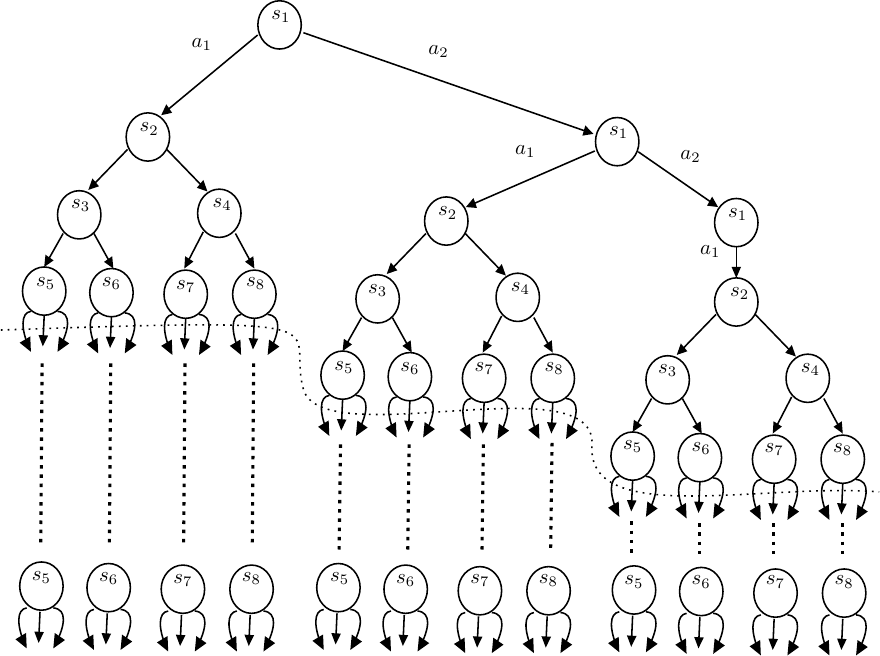}
\caption{Example of hard instance with $S=8$, $A=3$, and $\overline{H}=3$. \hl{The mean reward is zero everywhere, and its distribution is Gaussian with unit variance.}}\label{fig:hard-instance}
\end{figure}

The starting state $s_1$ has two available actions, $a_1$ and $a_2$. Action $a_1$ makes the agent transition to the root $s_2$ of the binary tree containing all other $S-1$ states. When taking action $a_2$, the agent remains in state $s_1$. Such an action is only available up to stage $\bar{H}-1$, where the value of $\bar{H}$ will be specified later. In stage $\bar{H}$, only $a_1$ is available and the agent must thus transition to state $s_2$. In other words, the agent can reach the root of the binary tree $s_2$ from stage $h=2$ (when taking $a_1$ at stage $1$) to stage $h=\bar{H}+1$ (when taking action $a_2$ for all $h = 1,\dots,\bar{H}-1$ and $a_1$ in stage $\bar{H}$). In the leaf states of the binary tree, all $A$ actions are available whose effect effect is to keep the agent in the same state until the final stage.

The intuition behind the hardness of this MDP instance is as follows. In order to learn a near-optimal policy, the agent must figure out which leaf state-action pair of the tree to reach (roughly $SA$ choices) and at which stage to reach it (exactly $\bar{H}$ choices). In graph-theoretical words, any flow leaving the initial state must pass through one of the tree leaves at one stage from $1+d$ to $\bar{H}+d$. If we ``cut'' the tree at those state-action-stage triplets (see the dashed line in Figure \ref{fig:hard-instance}), we have that the value of the flow (and thus $\tau$) can be written as the sum of visits to $\Omega(SA\bar{H})$ triplets. By constructing variants of this hard MDP where we raise or lower the reward of some of these triplets by roughly $\epsilon/\overline{H}$, we can prove (Lemma \ref{lem:lb-zero-reward} stated in Appendix \ref{app:lower-bounds}) that each of these triplets needs to be explored roughly $\Omega(\frac{\bar{H}}{\epsilon^2}\log(1/\delta))$ times. Summing them up and choosing $\bar{H} \geq \Omega(H)$ yields that $\tau$ must be at least $\Omega(\frac{SAH^2}{\epsilon^2}\log(1/\delta))$.

\begin{theorem}\label{th:worst-case-lb}
For any $S,A\geq 2$ and $H\geq 3\log_2(S)$ such that $S=2^d$ for some integer $d\geq 1$, there exists an MDP $\cM$ with $S$ states, $A$ actions, and $H$ stages such that any algorithm which is $(\epsilon,\delta)$-PAC on the class $\mathfrak{M}_1$ must satisfy
\begin{align*}
\bE_{\cM}[\tau] \geq \frac{SAH^2}{72\epsilon^2}\log(1/4\delta).
\end{align*}
\end{theorem}

Now we state a result that is key in the proof of Theorem \ref{th:worst-case-lb}.

\begin{lemma}\label{lem:lb-zero-reward}
Let $\cM$ be any MDP with $r_h(s,a) = 0$ for all $s,a,h$. Then, for any $1 \leq \bar{h} \leq \bar{H} \leq H$ and any policy $\pi\in\Pi$,
\begin{align*}
\sum_{h=\bar{h}}^{\bar{H}}\mathbb{E}[n_h^\tau(s_h^\pi,a_h^\pi)] \geq \frac{(\bar{H}-\bar{h}+1)^2}{4\epsilon^2}\log(1/4\delta).
\end{align*}
\end{lemma}
\begin{proof}
Fix some $1 \leq \bar{h} \leq \bar{H} \leq H$ and policy $\pi\in\Pi$. Define the event
\begin{align*}
E^\pi := \left\{\sum_{h=\bar{h}}^{\bar{H}} \indi{(s_h^{\hat{\pi}},a_h^{\hat{\pi}}) = (s_h^\pi,a_h^\pi) } \geq \frac{\bar{H}-\bar{h}+1}{2}\right\}.
\end{align*}
We distinguish two cases.

\paragraph{Case 1: $\bP_{\cM}(E^\pi) \leq 1/2$.}

Consider the alternative MDP $\widetilde{\cM} := (\cS, \cA, \{f_h,\widetilde{\nu}_h\}_{h\in[H]}, s_1, H)$ which is equivalent to $\cM$ except that the reward is modified as
\begin{align*}
\widetilde{\nu}_h(s,a) = 
\begin{cases}
\cN(2\epsilon/(\bar{H}-\bar{h}+1), 1) &\text{if } (s,a)=(s_h^\pi,a_h^\pi), \bar{h} \leq h \leq \bar{H},
\\ {\nu}_h(s,a) & \text{otherwise.}
\end{cases}
\end{align*}
Note that, for any policy $\pi'\in\Pi$, since the mean reward in $\widetilde{\cM}$ is non-zero only for state-action pairs visited by $\pi$ in stages from $\bar{h}$ to $\bar{H}$,
\begin{align*}
\widetilde{V}_1^{\pi'}(s_1) = \frac{2\epsilon}{\bar{H}-\bar{h}+1}\sum_{h=\bar{h}}^{\bar{H}} \indi{(s_h^{\pi'},a_h^{\pi'}) = (s_h^\pi,a_h^\pi) }.
\end{align*}
This implies that $\widetilde{V}_1^\pi(s_1) = \widetilde{V}_1^\star(s_1) = 2\epsilon$. Moreover, $\widetilde{V}_1^{\pi'}(s_1) < \epsilon$ if $\sum_{h=\bar{h}}^{\bar{H}} \indi{(s_h^{\pi'},a_h^{\pi'}) = (s_h^\pi,a_h^\pi) } < \frac{\bar{H}-\bar{h}+1}{2}$. Therefore, the event $E^\pi$ must have $\bP_{\widetilde{\cM}}(E^\pi) \geq 1-\delta$, otherwise the returned policy would not be $\epsilon$-optimal in $\widetilde{\cM}$. Thus, applying Lemma~\ref{lem:CD_mdps},
\begin{align*}
   \sum_{h=\bar{h}}^{\bar{H}}\bE[n_h^\tau(s_h^\pi,a_h^\pi)] 
   &\geq \frac{(\bar{H}-\bar{h}+1)^2}{2\epsilon^2} \mathrm{kl}(\bP_{\cM}(E^\pi), \bP_{\widetilde{\cM}}(E^\pi))
   \geq \frac{(\bar{H}-\bar{h}+1)^2}{4\epsilon^2}\log(1/4\delta),
\end{align*}
where the second inequality uses the same steps as in the proof of Case 1 of Lemma \ref{lem:lb-optimal-nonunique}.

\paragraph{Case 2: $\bP_{\cM}(E^\pi) > 1/2$.}

We use a similar construction as in Case 1, except that this time we build a new MDP by lowering the reward at pairs visited by $\pi$. Consider the alternative MDP $\widetilde{\cM} := (\cS, \cA, \{f_h,\widetilde{\nu}_h\}_{h\in[H]}, s_1, H)$ which is equivalent to $\cM$ except that the reward is modified as
\begin{align*}
\widetilde{\nu}_h(s,a) = 
\begin{cases}
\cN(-\Delta, 1) &\text{if } (s,a)=(s_h^\pi,a_h^\pi), \bar{h} \leq h \leq \bar{H},
\\ {\nu}_h(s,a) & \text{otherwise.}
\end{cases}
\end{align*}
Here $\Delta > 2\epsilon/(\bar{H}-\bar{h}+1)$. Note that $\widetilde{V}_1^\star(s_1) = 0$, which is attained by any policy not visiting state-action pairs visited by $\pi$ from $\bar{h}$ to $\bar{H}$. Moreover, for any $\pi'\in\Pi$, $\widetilde{V}_1^{\pi'}(s_1) < -\epsilon$ if $\sum_{h=\bar{h}}^{\bar{H}} \indi{(s_h^{\pi'},a_h^{\pi'}) = (s_h^\pi,a_h^\pi) } \geq \frac{\bar{H}-\bar{h}+1}{2}$ and thus $\pi'$ is not $\epsilon$-optimal for $\widetilde{\cM}$. Therefore, $\bP_{\widetilde{\cM}}(E^\pi) \leq \delta$. Thus, applying Lemma~\ref{lem:CD_mdps},
\begin{align*}
   \sum_{h=\bar{h}}^{\bar{H}}\bE[n_h^\tau(s_h^\pi,a_h^\pi)] 
   &\geq \frac{2}{\Delta^2} \mathrm{kl}(\bP_{\cM}(E^\pi), \bP_{\widetilde{\cM}}(E^\pi))
   \geq \frac{1}{\Delta^2}\log(1/4\delta),
\end{align*}
where the second inequality uses the same steps as in the proof of Case 2 of Lemma \ref{lem:lb-optimal-nonunique}. This holds for any $\Delta > 2\epsilon/(\bar{H}-\bar{h}+1)$ and the proof is concluded by taking the limit.
\end{proof}

We now prove the main theorem.

\begin{proof}[Proof of Theorem \ref{th:worst-case-lb}]
Let us consider the hard instance described above and exemplified in Figure \ref{fig:hard-instance}. Let us enumerate by $(\bar{s_i})_{i=1}^n$ the leaf states of the binary tree. For any $i\in[n],a\in[A],j\in[\bar{H}]$, let us denote by $\pi^{iaj}$ the policy which chooses action $a_1$ in $s_1$ at stage $j$ (and thus action $a_2$ in all stages before), which traverses the tree from stage $j+1$ to stage $j+d$, reaching the leaf state $\bar{s}_i$ and playing always action $a$ thereafter.

Now fix some $j\in[\bar{H}]$. Applying Lemma \ref{lem:lb-zero-reward} on the segment of stages from $j+d$ to $j+d+\bar{H}$,
\begin{align*}
\frac{1}{nA}\sum_{i=1}^n\sum_{a=1}^A \sum_{h=j+d}^{j+d+\bar{H}}\mathbb{E}[n_h^\tau(s_h^{\pi^{iaj}},a_h^{\pi^{iaj}})] \geq \frac{(\bar{H}+1)^2}{4\epsilon^2}\log(1/4\delta).
\end{align*}
Since this holds for any $j\in[\bar{H}]$,
\begin{align*}
\sum_{j=1}^{\bar{H}}\sum_{i=1}^n\sum_{a=1}^A \sum_{h=j+d}^{j+d+\bar{H}}\mathbb{E}[n_h^\tau(s_h^{\pi^{iaj}},a_h^{\pi^{iaj}})] \geq \frac{nA\bar{H}(\bar{H}+1)^2}{4\epsilon^2}\log(1/4\delta).
\end{align*}
Now note that the lefthand side can be equivalently written as
\begin{align*}
\sum_{j=1}^{\bar{H}}\sum_{i=1}^n\sum_{a=1}^A \sum_{h=j+d}^{j+d+\bar{H}}\mathbb{E}[n_h^\tau(s_h^{\pi^{iaj}},a_h^{\pi^{iaj}})] = \sum_{l=0}^{\bar{H}} \sum_{j=1}^{\bar{H}}\sum_{i=1}^n\sum_{a=1}^A \mathbb{E}[n_{j+d+l}^\tau(s_{j+d+l}^{\pi^{iaj}},a_{j+d+l}^{\pi^{iaj}})].
\end{align*}
Each element of the outer sum over $l=0,\dots,\bar{H}$ is exactly equal to $\bE[\tau]$ since it is the value of a cut at depth $l$ below each unrolled binary tree (e.g., the one shown in Figure \ref{fig:hard-instance} is for $l=0$). Therefore,
\begin{align*}
\bE[\tau] \geq \frac{nA\bar{H}(\bar{H}+1)}{4\epsilon^2}\log(1/4\delta).
\end{align*}
Now note that $n = 2^{d-1} = S/2$. Then, we only need to choose $\bar{H}$. Clearly, we need that $2\bar{H}+d \leq H$ to have all the segments above within the horizon, i.e., $H \geq 2\bar{H} + \log_2(S)$. If we choose $\bar{H} = H/3$, then $H \geq 3\log_2(s)$. With this choice we get
\begin{align*}
\bE[\tau] \geq \frac{SAH^2}{72\epsilon^2}\log(1/4\delta).
\end{align*}
\end{proof}

%% file: sections/app_algorithms.tex

\section{Sample Complexity Bounds (Proofs of Section \ref{sec:algorithms})}

\subsection{Good event}\label{app:good-event}

We consider the following concentration event
\begin{align*}
\cG := \left\{ \forall t \in \mathbb{N}, \forall h\in[H], s\in\cS_h,a\in\cA_h(s) : |\hat{r}_h^{t}(s,a) - r_h(s,a)| \leq b_h^t(s,a) \right\},
\end{align*}
where we recall that the bonuses $b_h^t(s,a)$ are defined in \eqref{eq:bonus}.
\begin{lemma}\label{lem:good-event-whp}
The good event $\cG$ holds with probability at least $1-\delta$.
\end{lemma}
\begin{proof}
  \hl{The proof easily follows from Hoeffding's inequality for sub-Gaussian distributions and a union bound:
  \begin{align*}
    \bP(\neg \cG) &\leq \bP\left(\exists t\in\mathbb{N}, h\in[H],s\in\cS_h,a\in\cA_h(s) : |\hat{r}_h^{t}(s,a) - r_h(s,a)| > b_h^t(s,a) \right)
    \\ &\leq \sum_{h\in[H]}\sum_{s\in\cS_h}\sum_{a\in\cA_h(s)} \bP\left(\exists t\in\mathbb{N} : |\hat{r}_h^{t}(s,a) - r_h(s,a)| > b_h^t(s,a) \right)
    \\ &\leq \sum_{h\in[H]}\sum_{s\in\cS_h}\sum_{a\in\cA_h(s)} \bP\left(\exists n\in\mathbb{N}^* : |\hat{r}_{h,n}(s,a) - r_h(s,a)| > \sqrt{\frac{2\sigma^2\log\left(\frac{4n^2N}{\delta}\right)}{n}} \right)
    \\ & \leq N \sum_{n=1}^\infty \frac{\delta }{2Nn^2} \leq \delta,
  \end{align*}
  where $\hat{r}_{h,n}(s,a)$ denotes the mean reward estimate after $n$ samples at $(s,a,h)$.}
\end{proof}

\subsection{Properties of Algorithm \ref{alg:elimination-alg}}

\subsubsection{Correctness (Proof of Theorem \ref{th:pac})}\label{app:correctness}

For a given subset of policies $\Pi'\subseteq \Pi$, we define $\cS_h(\Pi') := \{s\in\cS_h \mid \exists \pi\in\Pi' : s_h^\pi=s\}$ as the set of states which are visited by some policy in $\Pi'$ at stage $h$. Thus, $\cS_h = \cS_h(\Pi)$ is the set of all states reachable at stage $h$, while $\cS_h(\Pi^\star)$ is the subset of states visited by optimal policies.

\begin{lemma}\label{lem:opt-actions-not-eliminated}
Under event $\cG$, for any $t\in\N$, $h\in[H]$, $s\in\cS_h(\Pi^\star)$, and $a\in\argmax_{a\in\cA_h(s)}Q_h^\star(s,a)$, we have $a\in\cA^{t}_h(s)$, i.e., $a$ is never eliminated.
\end{lemma}
\begin{proof}
Take any $h\in[H],s\in\cS_h(\Pi^\star),a\in\argmax_{a\in\cA_h(s)}Q_h^\star(s,a)$. Let us prove the result by induction. It clearly holds for $t=0$ since the sets of active actions are initialized with the full sets of actions. Suppose it holds for $t-1$ with $t\geq 1$. Since this implies that $a\in\cA_h^{t-1}(s)$, it is enough to show that
\begin{align*}
 \max_{\pi\in\Pi_{s,a,h} \cap \Pi^{t-1}} \overline{V}_{1}^{t,\pi}(s_1) \geq \max_{\pi\in\Pi} \underline{V}_{1}^{t,\pi}(s_1)
\end{align*}
to guarantee that $a\in\cA_h^t(s)$. Then, for some optimal policy $\pi^\star\in\Pi^\star$,
\begin{align*}
\max_{\pi\in\Pi_{s,a,h} \cap \Pi^{t-1}} \overline{V}_{1}^{t,\pi}(s_1)
\geq \overline{V}_{1}^{t,\pi^\star}(s_1)
 \geq V_1^{\pi^\star}(s_1)
 = \max_{\pi\in\Pi} V_1^{\pi}(s_1) \geq \max_{\pi\in\Pi} \underline{V}_1^{t,\pi}(s_1),
\end{align*}
where the first inequality holds since there exists an optimal policy that visits $(s,a)$ at stage $h$ whose actions (at all visited states) are active by the inductive hypothesis, while the second and the third one are due to event $\cG$. This concludes the proof.
\end{proof}

\begin{lemma}\label{lem:correct-on-event-E}
Under event $\cG$, if the algorithm stops at the end of time $\tau \geq 1$ and returns a policy $\widehat{\pi}$, then $V_1^{\widehat{\pi}}(s_1) \geq V_1^{\star}(s_1) - \epsilon$.
\end{lemma}
\begin{proof}
We have two possible cases. First, suppose the algorithm stops with the first stopping rule. Under event $\cG$, we know that some optimal policy $\pi^\star\in\Pi^\star$ is always active, hence $\pi^\star \in \Pi^{\tau}$. Then,
\begin{align*}
V_1^{\star}(s_1) - V_1^{\widehat{\pi}}(s_1)
 = V_1^{\pi^\star}(s_1) - V_1^{\widehat{\pi}}(s_1)
  &\leq \overline{V}_1^{\pi^\star, \tau}(s_1) - \underline{V}_1^{\widehat{\pi},\tau}(s_1)
   \\ &\leq \max_{\pi\in\Pi^{\tau}} \overline{V}_1^{\pi, \tau}(s_1) - \underline{V}_1^{\widehat{\pi},\tau}(s_1)
   \\ &= \overline{V}_1^{\widehat{\pi}, \tau}(s_1) - \underline{V}_1^{\widehat{\pi},\tau}(s_1)
   \\ &\leq \max_{\pi\in\Pi^{\tau}} \left(\overline{V}_1^{{\pi}, \tau}(s_1) - \underline{V}_1^{{\pi},\tau}(s_1) \right) \leq \epsilon,
\end{align*}
where the first inequality holds by event $\cG$, the second inequality holds since $\pi^\star \in \Pi^{\tau}$, the equality is by definition of the recommendation rule, and the last inequality is due to the stopping rule.

In the second case, if the algorithm stops with the second stopping rule, then Lemma \ref{lem:opt-actions-not-eliminated} ensures that for all states visited by some optimal policy exactly the (necessarily unique) optimal action is left active. Therefore, the recommended policy plays only optimal actions in states that are visited by an optimal policy, which implies that the policy itself is optimal.
\end{proof}

\begin{proof}[Proof of Theorem \ref{th:pac}]
This is a simple combination of Lemma \ref{lem:correct-on-event-E}, which shows that the algorithm is $\epsilon$-correct on event $\cG$, and Lemma \ref{lem:good-event-whp}, which guarantees that $\cG$ holds with probability at least $1-\delta$.
\end{proof}

\subsubsection{Diameter vs Gaps (Proof of Lemma \ref{lem:diam-vs-gap})}

We prove Lemma \ref{lem:diam-vs-gap} stated in Section \ref{sec:algorithms}, an important result for Algorithm \ref{alg:elimination-alg} which relates the diameter of active policies to the sub-optimality gaps of non-eliminated actions. Here we prove that the result holds under the good event $\cG$, which in turns holds with probability at least $1-\delta$.

\begin{proof}[Proof of Lemma \ref{lem:diam-vs-gap}]
Suppose the good event $\cG$ holds and let $t$ be any episode at the end of which the algorithm did not stop. We derive separately a gap-dependent bound for sub-optimal state-action pairs and an $\epsilon$-dependent bound for all state-action pairs. 

\paragraph{Gap-dependent bound (suboptimal state-action pairs)}

Let $h\in[H], s\in\cS_h, a\in\cA_h(s)$ be such that $\overline{\Delta}_h(s,a) > 0$ (i.e., this is a sub-optimal state-action pair at stage $h$) and $a\in\cA^{t}_h(s)$ (i.e., the action has not been eliminated at the end of episode $t$). Then, by definition of the elimination rule,
\begin{align*}
\max_{\pi\in\Pi_{s,a,h}\cap \Pi^{t-1}} \overline{V}_{1}^{t,\pi}(s_1) \geq \max_{\pi\in\Pi} \underline{V}_{1}^{t,\pi}(s_1).
\end{align*}
Using the good event, this implies that, for any optimal policy $\pi^\star$,
\begin{align*}
\max_{\pi\in\Pi_{s,a,h}\cap \Pi^{t-1}} \left({V}_{1}^{\pi}(s_1) + 2\sum_{h=1}^H b_h^t(s_h^\pi,a_h^\pi)\right)
 &\geq \max_{\pi\in\Pi} \left( {V}_{1}^{\pi}(s_1) - 2\sum_{h=1}^H b_h^t(s_h^\pi,a_h^\pi)\right)
 \\ &\geq {V}_{1}^{\star}(s_1) - 2\sum_{h=1}^H b_h^t(s_h^{\pi^\star},a_h^{\pi^\star}).
\end{align*}
Thus,
\begin{align*}
2\max_{\pi\in\Pi^{t-1}}\sum_{h=1}^H b_h^t(s_h^\pi,a_h^\pi) + 2\sum_{h=1}^H b_h^t(s_h^{\pi^\star},a_h^{\pi^\star})
 \geq {V}_{1}^{\star}(s_1) - \max_{\pi\in\Pi_{s,a,h}}{V}_{1}^{\pi}(s_1) = \overline{\Delta}_h(s,a).
\end{align*}
Since all state-action pairs along each optimal trajectory are active under the good event (Lemma \ref{lem:opt-actions-not-eliminated}), $\sum_{h=1}^H b_h^t(s_h^{\pi^\star},a_h^{\pi^\star}) \leq \max_{\pi\in\Pi^{t-1}}\sum_{h=1}^H b_h^t(s_h^\pi,a_h^\pi)$. Therefore, expanding the definition of the bonuses,
\begin{align*}
 \overline{\Delta}_h(s,a) \leq
  4\max_{\pi\in\Pi^{t-1}}\sum_{h=1}^H b_h^t(s_h^\pi,a_h^\pi).
\end{align*}

\paragraph{$\epsilon$-dependent bound}

If the algorithm did not stop at the end of episode $t$, by the first stopping rule,
\begin{align*}
\frac{\epsilon}{2} 
\leq \max_{\pi\in\Pi^{t}} \sum_{h=1}^H b^{t}_h(s_h^\pi,a_h^\pi)
\leq \max_{\pi\in\Pi^{t-1}} \sum_{h=1}^H b^{t}_h(s_h^\pi,a_h^\pi).
\end{align*}

\paragraph{Gap-dependent bound (unique optimal trajectory)}

Finally, let us consider the special case where the optimal trajectory $(s_h^\star,a_h^\star)_{h\in[H]}$ is unique.
The derivation above holds for any state-action pair not belonging to an optimal trajectory (i.e., with positive gap). In this case, it can be trivially extended to optimal state-action pairs. Since the algorithm did not stop at the end of episode $t$, it must be that at least some sub-optimal state-action pair is active (otherwise there would be at most one active action in each state and the stopping condition would be verified). That is, there exist $h\in[H], s\in\cS_h, a\in\cA_h(s)$ with $\overline{\Delta}_h(s,a) \geq \overline{\Delta}_{\min} > 0$ such that $a\in\cA_h^t(s)$. Using the same derivation as above, we obtain
\begin{align*}
 \overline{\Delta}_{\min} \leq
  4\max_{\pi\in\Pi^{t-1}}\sum_{h=1}^H b_h^t(s_h^\pi,a_h^\pi).
\end{align*}
\end{proof}

\subsection{Maximum-coverage algorithm (Proof of Theorem \ref{th:max-cover-sample-comp})}\label{app:max-cover}

\subsubsection{Static maximum-coverage sampling} 
For the purpose of the analysis, we introduce a variant of the maximum-coverage sampling rule, that we refer to as {\it static maximum-coverage}. As we will see, static maximum-coverage and (standard) maximum-coverage are very related, in the sense that in each period $k$ there exists a function $C^k: 2^\Pi \rightarrow [0,\infty)$, called a coverage function, such that static-maximum coverage directly maximizes the function $C^k$, while maximum-coverage greedily builds a set of policies that maximizes it. The pseudo-code of static maximum-coverage is given in Algorithm \ref{alg:static-max-cov}. 

\begin{algorithm}[h]
\caption{Static maximum-coverage sampling}\label{alg:static-max-cov}
\begin{algorithmic}

\STATE \textbf{function} \textsc{StaticMaxCoverage}()
\STATE \quad Let $k_t \leftarrow \min_{h\in[H],s\in\cS_h,a\in\cA_h^{t-1}(s)} n_h^{t-1}(s,a) + 1$ and $\bar{t}_{k_t} \leftarrow \inf_{l\in\mathbb{N}}\{l : k_l=k_t\}$
\STATE \quad \textbf{if} $k_t > k_{t-1}$ \textbf{then}
\STATE \quad \quad Let $\underline{c}_h^{k_t}(s,a) \leftarrow \indi{a\in\cA_h^{\bar{t}_{k_t}-1}(s), n_h^{\bar{t}_{k_t}-1}(s,a) < k_t}$
\STATE \quad \quad Compute $\eta^{k_t}$, an integer minimum flow on $\cG(\cM)$ with lower bounds $\underline{c}^{k_t}$
\STATE \quad \quad Extract a minimum policy cover $\Pi_{\mathrm{cover}}^{t}$ from $\eta^{k_t}$ using Algorithm \ref{alg:extract-cover}
\STATE \quad \textbf{end if}
\STATE \quad \textbf{if} $t\ \mathrm{mod}\ 2 = 1$ \textbf{then}
\STATE \quad\quad Choose $\pi^t$ arbitrarily from $\Pi_{\mathrm{cover}}^{t}$ and remove it: $\Pi_{\mathrm{cover}}^{t+1} \leftarrow \Pi_{\mathrm{cover}}^{t} \setminus \{\pi^t\}$
\STATE \quad\quad \textbf{return} $\pi^t$
\STATE \quad \textbf{else}
\STATE \quad\quad Let $\Pi_{\mathrm{cover}}^{t+1} \leftarrow \Pi_{\mathrm{cover}}^{t}$
\STATE \quad \quad \textbf{return} $\pi^t \leftarrow \textsc{MaxDiameter()}$
\STATE \quad \textbf{end if}
\STATE \textbf{end function}
\STATE
\STATE \textbf{function} \textsc{MaxDiameter}()
\STATE \quad \textbf{return} $\pi^t \leftarrow \argmax_{\pi\in\Pi^{t-1}} \sum_{h=1}^H b_h^{t-1}(s_h^\pi,a_h^\pi)$
\STATE \textbf{end function}
\end{algorithmic}
\end{algorithm}

In words, static maximum-coverage precomputes a set of policies of minimum size, which we call a \emph{minimum policy cover}\footnote{A similar concept of ``policy cover'' was considered by \cite{agarwal2020pc}, where the authors propose an algorithm that incrementally builds a set of policies exploring the whole state-action space in the context of policy gradient methods.}, that guarantees at least one visit to all active under-sampled $(s,a,h)$, i.e., all those such that $a\in\cA_h^{\bar{t}_{k_t}-1}(s)$ and $n_h^{\bar{t}_{k_t}-1}(s,a) < k_t$. It is easy to see that (see also Appendix \ref{app:min-cover}) this problem can be reduced to finding a minimum flow with lower bound function $\underline{c}_h^{k_t}(s,a) := \indi{a\in\cA_h^{\bar{t}_{k_t}-1}(s), n_h^{\bar{t}_{k_t}-1}(s,a) < k_t}$.
Stated otherwise, we require the resulting flow to have a value of at least $1$ for every active $(s,a,h)$ that has less than $k$ visits. Once a minimum (integer) flow $\eta^k$ has been computed, a policy cover can be easily extracted using Algorithm \ref{alg:extract-cover}. Finally, once a minimum policy cover $\Pi_{\mathrm{cover}}^{\bar{t}_k}$ for period $k$ has been extracted, our sampling rule simply switches between playing a policy in this set to ensure good coverage of the whole MDP and playing the policy prescribed by the maximum-diameter sampling rule. 

\subsubsection{Main Theorem}

We state the following theorem, which simultaneously upper bound the sample complexity of max-coverage sampling and its static version.

\begin{theorem}[Formal statement of Theorem \ref{th:max-cover-sample-comp}]\label{th:formal}
With probability at least $1-\delta$, the sample complexity of Algorithm \ref{alg:elimination-alg} combined with either the static maximum-coverage (in which case $C_H := 1$) or the  maximum-coverage (in which case $C_H := \log(H) + 1$) sampling rule is bounded by
\begin{align*}
\tau \leq 2C_H\left(\max_{(s,a,h)}\log\big( g_h(s,a)\big) + 1\right) \varphi^\star(g),
\end{align*}
where $g : \cE \rightarrow [0,\infty)$ is the lower bound function defined by\hl{
\begin{align*}
g_h(s,a) := \frac{32\sigma^2 H^2}{\max\left(\overline{\Delta}_h(s,a),\overline{\Delta}_{\min}, \epsilon \right)^2}  \left( \log\left(\frac{4N^3}{\delta}\right) + L_h(s,a)\right) + 2
\end{align*}
with $L_h(s,a) := 8\log\left( \frac{16\sigma H\log\left(\frac{4N^3}{\delta}\right)}{\max\left(\overline{\Delta}_h(s,a),\overline{\Delta}_{\min}, \epsilon \right)}  \right)$. Moreover, with the same probability,
\begin{align*}
  \tau \leq \frac{256\sigma^2 SAH^2}{\epsilon^2} \left( \log\left(\frac{4SAH}{\delta}\right) + 4\log\left(\frac{512\sigma^2 SAH^2 \log\left(\frac{4SAH}{\delta}\right)}{\epsilon^2}\right) \right) + 8SAH + 2.
\end{align*}}
\end{theorem}

In the next sub-sections, we prove Theorem~\ref{th:formal}.

\subsubsection{Decomposition into periods}\label{sec:periods}

Recall that $k_t$, defined in Algorithm \ref{alg:elimination-alg} for both the static maximum-coverage and the maximum-coverage sampling rules, is the ``target'' number of visits at time $t$. We shall refer to the set of consecutive time steps $\{t\in\mathbb{N} : k_t=k\}$ as the $k$-th \emph{period}. This is intuitively the set of time steps where the sampling rule is trying to make all active triplets reach $k$ visits.  Let $d_k := \sum_{t=1}^\tau \indi{k_t = k}$ be the duration of the $k$-th period. Note that the period could be empty (e.g., this might happen when some under-sampled triplets are eliminated), in which case we have $d_k = 0$. The sample complexity can thus be decomposed as
\begin{align*}
\tau = \sum_{k=1}^{k_\tau} \sum_{t=1}^\tau \indi{k_t = k} = \sum_{k=1}^{k_\tau} d_k.
\end{align*}
Our goal in this section is to bound the duration of each period. In particular, while the duration of the $k$-th period can be trivially bounded by twice the size of the minimum policy cover $\Pi_{\mathrm{cover}}^{\bar{t}_k}$ for the static maximum-coverage sampling rule, we shall see that a similar bound holds also for maximum-coverage.

\paragraph{Static maximum-coverage sampling}

The following bound can be easily derived from the definition of the sampling rule.

\begin{theorem}[Period duration for static maximum-coverage]\label{th:period-duration-max-cover}
When using the static maximum-coverage sampling rule, for any non-empty period $k\in\mathbb{N}$,
\begin{align*}
{d}_k \leq 2\varphi^\star(\underline{c}^k).
\end{align*}
\end{theorem}
\begin{proof}
If $k$ is a non-empty period, there exists a time $\bar{t}_k$ at which the period starts where a minimum policy cover $\Pi_{\mathrm{cover}}^{\bar{t}_k}$ is computed. The size of this cover is exactly the value $\varphi^\star(\underline{c}^k)$ of a minimimum flow computed with lower bound function $\underline{c}_h^k(s,a) := \indi{a\in\cA_h^{\bar{t}_{k}-1}(s), n_h^{\bar{t}_{k}-1}(s,a) < k}$. The stated bound easily follows from the fact that a new policy in the cover is played every two episodes and that the period necessarily ends when all policies in the cover have been played.
\end{proof}

\paragraph{Maximum-coverage sampling}

Let $\bar{t}_k := \inf_{t\in\mathbb{N}}\{t : k_t=k\}$ be the first time step in the period as defined in Algorithm \ref{alg:elimination-alg} (which exists if the period is non-empty). We start by proving the following result, which allows us to better characterize the duration of period $k$ in terms of the first time where all active triplets at the \emph{beginning} of the period receive at least $k$ visits.


\begin{lemma}\label{lem:period-less-adaptive}
For any non-empty period $k\in\mathbb{N}$, almost surely
\begin{align*}
d_k \leq \tilde{d}_k := \inf_{t\in\mathbb{N}}\{t : \min_{h\in[H],s\in\cS_h,a\in\cA_h^{\bar{t}_k-1}(s)} n_h^{t-1}(s,a) \geq k\} - \bar{t}_k  < \infty.
\end{align*}
\end{lemma}
\begin{proof}
Since period $k$ ends when all active pairs are visited at least $k$ times,
\begin{align*}
d_k = \inf_{t\in\mathbb{N}}\{t : k_t > k\} - \bar{t}_k 
&= \inf_{t\geq \bar{t}_k}\{t : \min_{h\in[H],s\in\cS_h,a\in\cA_h^{t-1}(s)} n_h^{t-1}(s,a) \geq k\} - \bar{t}_k .
\\ &\leq \inf_{t\in\mathbb{N}}\{t : \min_{h\in[H],s\in\cS_h,a\in\cA_h^{\bar{t}_k-1}(s)} n_h^{t-1}(s,a) \geq k\} - \bar{t}_k .
\end{align*}
where the inequality holds since $\cA_h^{t-1}(s) \subseteq \cA_h^{\bar{t}_k-1}(s)$ for each $h\in[H],s\in\cS_h$, and $t\geq\bar{t}_k$. To see why $\tilde{d}_k < \infty$, note that, by definition, the sampling rule visits at least one undersampled triplet $(s,a,h)$ (i.e., such that $n_h^{t-1}(s,a) < k$) every two steps. Since there are at most $N$ triplets that need to be visited, we get that $\tilde{d}_k \leq 2N < \infty$ almost surely.
\end{proof}

\paragraph{Reduction to submodular maximization}

Let us define the set function $C^k: 2^\Pi \rightarrow [0,\infty)$ as
\begin{align*}
C^k(\Pi') := \sum_{h=1}^H\sum_{s\in\cS_h}\sum_{a\in\cA_h^{\bar{t}_k-1}(s)} \indi{n_h^{\bar{t}_k-1}(s,a) < k, \exists \pi\in\Pi' : (s_h^{\pi},a_h^{\pi})=(s,a)} \quad \forall \Pi'\subseteq \Pi.
\end{align*}
Moreover, let $\bar{\Pi}_i^{k} := \{{\pi}^{t} \mid t = \bar{t}_k,\dots, \bar{t}_k+i-1\}$ be the set containing the first $i$ policies played by the maximum-coverage sampling rule in period $k$. We note that the first policy selection strategy (the one called at odd steps) is essentially a greedy algorithm approximating the maximization of $C^k$. In fact, maximizing $C^k$ corresponds to finding a set of policies that visit all active triplets at time $\bar{t}_k-1$ that have less than $k$ visits (which, by definition of period, means that they have exactly $k-1$ visits). Instead of directly maximizing the set function $C^k$, such policy selection strategy greedily builds the set $\bar{\Pi}_i^{k}$ by adding, at each round where it is used, the policy visiting the most of these undervisited triplets. Moreover, we note that $C^k$ is a \emph{coverage function}, a kind of function which is known to be monotone submodular and for which greedy maximization is very efficient \citep{nemhauser1978analysis}. Let us prove some of its important properties.

First, we relate the maximization of $C^k$ to the computation of a minimum flow with lower bound function $\underline{c}_h^k(s,a) \leftarrow \indi{a\in\cA_h^{\bar{t}_k-1}(s), n_h^{\bar{t}_k-1}(s,a) < k}$, i.e., the same as the one used by the static maximum-coverage sampling rule. Let $N_k := \sum_{h\in[H]}\sum_{s\in\cS_h}\sum_{a\in\cA_h^{\bar{t}_k-1}(s)}\indi{n_h^{\bar{t}_k-1}(s,a) < k}$ be the total number of triplets that need to be visited in period $k$.
\begin{proposition}[Maximization vs minimum flow]\label{prop:greedy-max-vs-flow}
For each $v\geq \varphi^\star(\underline{c}^k)$,
\begin{align*}
\max_{\Pi' \subseteq \Pi : |\Pi'| \leq v} C^k(\Pi') = \max_{\Pi' \subseteq \Pi} C^k(\Pi') = N_k.
\end{align*}
\end{proposition}
\begin{proof}
Clearly, $C^k(\Pi') \leq N_k$ for all $\Pi'\subseteq\Pi$, which is attained when all undervisited state-action-stage triplets are visited at least once. When the cardinality of $\Pi'$ can be at least $\varphi^\star(\underline{c}^k)$, we can choose $\Pi'$ to include a set of $\varphi^\star(\underline{c}^k)$ policies realizing a minimum 1-flow (i.e., a minimum policy cover as the one computed by static maximum-coverage sampling in period $k$). These, by definition, cover all undervisited triplets, and thus attain the maximal value $N_k$.
\end{proof}
\begin{proposition}[Monotonicity]
For each $\Pi' \subseteq \Pi'' \subseteq \Pi$, $C^k(\Pi') \leq C^k(\Pi'')$.
\end{proposition}
\begin{proof}
This is trivial: since $\Pi''$ contains $\Pi'$, it must visit at least all the triplets visited by $\Pi'$.
\end{proof}
\begin{proposition}[Sub-modularity]
Function $f$ is sub-modular, i.e., for every $\Pi' \subseteq \Pi'' \subseteq \Pi$ and $\bar{\pi} \in \Pi \setminus \Pi''$,
\begin{align*}
C^k(\Pi' \cup \{\bar{\pi}\}) - C^k(\Pi') \geq C^k(\Pi'' \cup \{\bar{\pi}\}) - C^k(\Pi'').
\end{align*}
\end{proposition}
\begin{proof}
Note that
\begin{align*}
C^k(\Pi' \cup &\{\bar{\pi}\}) - C^k(\Pi')
\\ &:= \sum_{h=1}^H\sum_{s\in\cS_h}\sum_{a\in\cA_h^{\bar{t}_k-1}(s), n_h^{\bar{t}_k-1}(s,a) < k} \indi{(s_h^{\bar{\pi}},a_h^{\bar{\pi}})=(s,a), \neg\exists \pi\in\Pi' : (s_h^{\pi},a_h^{\pi})=(s,a)}
\\ &= \sum_{h=1}^H \indi{\neg\exists \pi\in\Pi' : (s_h^{\pi},a_h^{\pi})=(s_h^{\bar{\pi}},a_h^{\bar{\pi}}), a_h^{\bar{\pi}} \in \cA_h^{\bar{t}_k-1}(s_h^{\bar{\pi}}), n_h^{\bar{t}_k-1}(s_h^{\bar{\pi}},a_h^{\bar{\pi}}) < k}
\\ &\geq \sum_{h=1}^H \indi{\neg\exists \pi\in\Pi'' : (s_h^{\pi},a_h^{\pi})=(s_h^{\bar{\pi}},a_h^{\bar{\pi}}), a_h^{\bar{\pi}} \in \cA_h^{\bar{t}_k-1}(s_h^{\bar{\pi}}), n_h^{\bar{t}_k-1}(s_h^{\bar{\pi}},a_h^{\bar{\pi}}) < k}
\\ &= C^k(\Pi'' \cup \{\bar{\pi}\}) - C^k(\Pi''),
\end{align*}
where the inequality holds since $\Pi'\subseteq\Pi''$.
\end{proof}
\begin{proposition}[Greedy maximization]\label{prop:greedy-max}
Let $\bar{\Pi}_i^{k}$ be the set containing the first $i\geq 0$ policies computed by the maximum-coverage sampling rule in period $k$. Then, for any positive integer $v$,
\begin{align*}
C^k(\bar{\Pi}_i^{k}) \geq (1 - e^{-\lfloor (i+1)/2 \rfloor/v}) \max_{\Pi' \subseteq \Pi : |\Pi'| \leq v} C^k(\Pi').
\end{align*}
\end{proposition}
\begin{proof}
This is a simple extension to Theorem 1.5 of \cite{krause2014submodular}, which in turns is a slight generalization of a well-known result by \cite{nemhauser1978analysis}. We report the proof for completeness since we have to deal explicitly with time steps where the maximum-diameter rule (which is not a greedy maximizer of $C^k$) is used.

Fix some positive integers $i,v$. If $i$ is such that $\bar{t}_k+i-1$ is odd (i.e., the first sampling rule is used at step $\bar{t}_k+i-1$), then using Equation 3 to 7 in the proof of Theorem 1.5 of \cite{krause2014submodular},
\begin{align*}
C^\star := \max_{\Pi' \subseteq \Pi : |\Pi'| \leq v} C^k(\Pi') \leq C^k(\bar{\Pi}_{i-1}^{k}) + v(C^k(\bar{\Pi}_{i}^{k}) - C^k(\bar{\Pi}_{i-1}^{k})).
\end{align*}
In particular, note that their inequality 6 holds since, by definition of our first sampling rule,
\begin{align*}
\pi^{\bar{t}_k+i-1} \in \argmax_{\pi\in\Pi} \left( C^k(\bar{\Pi}_{i-1}^{k} \cup \{\pi\}) - C^k(\bar{\Pi}_{i-1}^{k}) \right).
\end{align*}
Rearranging, we get
\begin{align*}
C^\star - C^k(\bar{\Pi}_{i}^{k}) \leq (1-1/v)(C^\star - C^k(\bar{\Pi}_{i-1}^{k})).
\end{align*}
On the other hand, if $i$ is such that $\bar{t}_k+i-1$ is even (i.e., the maximum-diameter sampling rule is used at step $\bar{t}_k+i-1$), then, by monotonicity of $C^k$,
\begin{align*}
C^\star - C^k(\bar{\Pi}_{i}^{k}) \leq C^\star - C^k(\bar{\Pi}_{i-1}^{k}).
\end{align*}
Therefore, unrolling this recursion from $i \geq 0$ and using that $C^k(\bar{\Pi}_{0}^{k}) = C^k(\emptyset) = 0$,
\begin{align*}
C^\star - C^k(\bar{\Pi}_{i}^{k}) \leq (1-1/v)^{\lfloor (i+1)/2 \rfloor}C^\star.
\end{align*}
Using that $1-x\leq e^{-x}$ and rearranging concludes the proof.
\end{proof}
\begin{theorem}[Period duration for maximum-coverage]\label{th:period-duration-ada-max-cover}
When using the maximum-coverage sampling rule, for any non-empty period $k\in\mathbb{N}$,
\begin{align*}
{d}_k \leq 2\varphi^\star(\underline{c}^k)(\log(H)+1).
\end{align*}
\end{theorem}
\begin{proof}
Let $\underline{i} := \sup_{i\in\mathbb{N}}\{i : C^k(\bar{\Pi}_i^k) \leq N_k-\varphi^\star(\underline{c}^k)\}$ be the last iteration in period $k$ at which at least $\varphi^\star(\underline{c}^k)$ triplets still need to be visited by the algorithm. Then, by Proposition \ref{prop:greedy-max} combined with Proposition \ref{prop:greedy-max-vs-flow},
\begin{align*}
N_k-\varphi^\star(\underline{c}^k)
\geq C^k(\bar{\Pi}_{\underline{i}}^k)
\geq (1 - e^{-\lfloor (\underline{i}+1)/2 \rfloor/\varphi^\star(\underline{c}^k)}) \max_{\Pi' \subseteq \Pi : |\Pi'| \leq \varphi^\star(\underline{c}^k)} C^k(\Pi')
 = (1 - e^{-\lfloor (\underline{i}+1)/2 \rfloor/\varphi^\star(\underline{c}^k)}) N_k.
\end{align*}
Thus,
\begin{align*}
\lfloor (\underline{i}+1)/2 \rfloor \leq \varphi^\star(\underline{c}^k)\log(N_k/\varphi^\star(\underline{c}^k)) \leq \varphi^\star(\underline{c}^k)\log(H),
\end{align*}
where the second inequality holds since $\varphi^\star(\underline{c}^k) \geq \max_{h\in[H]} \sum_{s\in\cS_h}\sum_{a\in\cA_h^{\bar{t}_k-1}(s)}\indi{n_h^{\bar{t}_k-1}(s,a) < k}$ by Lemma \ref{lem:flow-bounds} and $N_k \leq H \max_{h\in[H]} \sum_{s\in\cS_h}\sum_{a\in\cA_h^{\bar{t}_k-1}(s)}\indi{n_h^{\bar{t}_k-1}(s,a) < k}$. This implies that $\underline{i} \leq 2\varphi^\star(\underline{c}^k)\log(H) - 1$ if $\underline{i}$ is odd, while $\underline{i} \leq 2\varphi^\star(\underline{c}^k)\log(H)$ if $\underline{i}$ is even. Finally, note that $\tilde{d}_k \leq \underline{i} + 2\varphi^\star(\underline{c}^k)$ since at iteration $\underline{i}+1$ less than $\varphi^\star(\underline{c}^k)$ triplets are missing and the algorithm visits at least a new one every two rounds. Then, the proof is concluded by Lemma \ref{lem:period-less-adaptive}.
\end{proof}

\subsubsection{Elimination periods}

We now bound the period indexes at which sub-optimal state-action pairs are eliminated. All results in this section hold for both the static maximum-coverage and the maximum-coverage sampling rule.

\begin{lemma}[Cover property]\label{lem:lb-visits}
For any non-empty period $k\in\mathbb{N}$, $h\in[H]$, $s\in\cS_h$, and any action $a\in\cA_h^{\bar{t}_k-1}(s)$ that is active when the period begins,
\begin{align*}
n_h^{\bar{t}_k-1}(s,a) \geq k-1.
\end{align*}
\end{lemma}
\begin{proof}
This is trivial from the definition of period: $\bar{t}_k$ is the first time with $k_{\bar{t}_k} = k$ and $k_{\bar{t}_k} := \min_{h\in[H],s\in\cS_h,a\in\cA_h^{\bar{t}_k-1}(s)} n_h^{\bar{t}_k-1}(s,a) + 1$.
\end{proof}

\begin{lemma}[Elimination periods]\label{lem:elim-periods}
Recall that $k_\tau$ is the period in which Algorithm \ref{alg:elimination-alg} stops and define
\begin{align*}
\kappa_{s,a,h} := \inf_{k\in\mathbb{N}} \left\{ k : a\notin\cA_h^{\bar{t}_{k+1}-1}(s) \right\} \wedge k_\tau
\end{align*}
as the period at the end of which $(s,a,h)$ is eliminated. Then, under the good event $\cG$, for any $h\in[H],s\in\cS_h,a\in\cA_h(s)$,\hl{
\begin{align*}
\kappa_{s,a,h} \leq \overline{\kappa}_{s,a,h} := \frac{32\sigma^2 H^2}{\max\left(\overline{\Delta}_h(s,a),\overline{\Delta}_{\min}, \epsilon \right)^2}  \left( \log\left(\frac{4N^3}{\delta}\right) + L_h(s,a)\right) + 1
\end{align*}
where
\begin{align*}
L_h(s,a) := 8\log\left( \frac{16\sigma H\log\left(\frac{4N^3}{\delta}\right)}{\max\left(\overline{\Delta}_h(s,a),\overline{\Delta}_{\min}, \epsilon \right)}  \right).
\end{align*}}
\end{lemma}
\begin{proof}
Take any $k\in\mathbb{N},h\in[H],s\in\cS_h,a\in\cA_h(s)$ such that $a\in\cA_h^{\bar{t}_{k+1}-1}(s)$. Under $\cG$, we have,
\begin{align*}
 \max\left(\frac{\overline{\Delta}_h(s,a)}{4},\frac{\overline{\Delta}_{\min}}{4}, \frac{\epsilon}{2} \right) 
 &\stackrel{(a)}{\leq} \max_{\pi\in\Pi^{\bar{t}_{k+1}-2}}\sum_{h=1}^H b_h^{\bar{t}_{k+1}-1}(s_h^\pi,a_h^\pi)
 \\ &\stackrel{(b)}{\leq} \max_{\pi\in\Pi^{\bar{t}_{k+1}-2}} \sum_{h=1}^H\sqrt{\frac{\beta(n_h^{\bar{t}_{k+1}-1}(s_h^\pi,a_h^\pi),\delta)}{n_h^{\bar{t}_{k+1}-1}(s_h^\pi,a_h^\pi)}}
  \\ &\stackrel{(c)}{\leq} \max_{\pi\in\Pi^{\bar{t}_{k+1}-2}} \sum_{h=1}^H\sqrt{\frac{\beta(\bar{t}_{k+1}-1,\delta)}{n_h^{\bar{t}_{k+1}-1}(s_h^\pi,a_h^\pi)}}
  \\ &\stackrel{(d)}{\leq} \max_{\pi\in\Pi^{\bar{t}_{k+1}-2}} \sum_{h=1}^H\sqrt{\frac{\beta(Nk,\delta)}{n_h^{\bar{t}_{k+1}-1}(s_h^\pi,a_h^\pi)}}
  \\ &\stackrel{(e)}{\leq} H\sqrt{\frac{\beta(Nk,\delta)}{k}},
\end{align*}
where (a) uses Lemma \ref{lem:diam-vs-gap}, (b) follows by expanding the definition of the bonuses, (c) uses that any active state-action-stage triplet cannot be visited more than $t$ times and the end of round $t$, (d) uses that $\bar{t}_{k+1}-1 \leq \sum_{k'=1}^k d_{k'} \leq Nk$ since a trivial bound on the duration of each period is $N$, and (e) uses Lemma \ref{lem:lb-visits}. \hl{Note that
\begin{align*}
\beta(Nk,\delta) = 2\sigma^2\log\left(\frac{4N^3k^2}{\delta}\right) = 2\sigma^2\log\left(\frac{4N^3}{\delta}\right) + 4\sigma^2\log\left(k\right)
\end{align*}
Therefore, we have that, if $(s,a,h)$ is active at the end of period $k$,
\begin{align*}
k \leq 2\sigma^2 H^2 \frac{\log\left(\frac{4N^3}{\delta}\right) + 2\log\left(k\right)}{\max\left(\frac{\overline{\Delta}_h(s,a)}{4},\frac{\overline{\Delta}_{\min}}{4}, \frac{\epsilon}{2} \right)^2} 
\leq 32\sigma^2 H^2 \frac{\log\left(\frac{4N^3}{\delta}\right) + 2\log\left(k\right)}{\max\left(\overline{\Delta}_h(s,a),\overline{\Delta}_{\min}, \epsilon \right)^2}.
\end{align*}
Using Lemma \ref{lem:simplify-ineq-2} with $C = \frac{32\sigma^2 H^2\log\left(\frac{4N^3}{\delta}\right)}{\max\left(\overline{\Delta}_h(s,a),\overline{\Delta}_{\min}, \epsilon \right)^2}$ and $B =  \frac{64\sigma^2 H^2}{\max\left(\overline{\Delta}_h(s,a),\overline{\Delta}_{\min}, \epsilon \right)^2}$, while noting that
\begin{align*}
  \log(B^2 + 2C) \leq 2\log(4C) 
  \leq 4\log\left( \frac{16\sigma H\log\left(\frac{4N^3}{\delta}\right)}{\max\left(\overline{\Delta}_h(s,a),\overline{\Delta}_{\min}, \epsilon \right)}  \right) = \frac{L_h(s,a)}{2},
\end{align*}
we obtain
\begin{align*}
  k &\leq \frac{32\sigma^2 H^2}{\max\left(\overline{\Delta}_h(s,a),\overline{\Delta}_{\min}, \epsilon \right)^2}  \left( \log\left(\frac{4N^3}{\delta}\right) + L_h(s,a)\right).
\end{align*}
The proof is concluded by noting that $\kappa_{s,a,h}$ is smaller than the first integer $k\in\mathbb{N}$ not satisfying the inequality above.}
\end{proof}

\subsubsection{Sample complexity}

We prove first the instance-dependent bound and then the worst-case one.

\begin{theorem}\label{th:sample-comp-instance-dependent-adaptive}
Under the good event $\cG$, the sample complexity of Algorithm \ref{alg:elimination-alg} when combined with either the static maximum-coverage (in which case $C_H := 1$) or the maximum-coverage (in which case $C_H := \log(H) + 1$) sampling rule is bounded by
\begin{align*}
\tau \leq 2C_H(\log(\overline{\kappa}) + 1) \varphi^\star(g),
\end{align*}
where $g : \cE \rightarrow [0,\infty)$ is defined as $g_h(s,a) = \overline{\kappa}_{s,a,h}+1$, with $\overline{\kappa}_{s,a,h}$ being the upper bound on the elimination period of $(s,a,h)$ from Lemma \ref{lem:elim-periods}, and $\overline{\kappa} := \max_{h\in[H],s\in\cS_h,a\in\cA_h(s)} \overline{\kappa}_{s,a,h}$.
\end{theorem}
\begin{proof}
Using the decomposition into periods introduced in Section \ref{sec:periods} followed by Theorem \ref{th:period-duration-max-cover} (for static maximum-coverage sampling) or Theorem \ref{th:period-duration-ada-max-cover} (for maximum-coverage sampling),
\begin{align*}
\tau = \sum_{k=1}^{k_\tau} \sum_{t=1}^\tau \indi{k_t = k} = \sum_{k=1}^{k_\tau} d_k \leq 2C_H\sum_{k=1}^{k_\tau} \varphi^\star(\underline{c}^k),
\end{align*}
where we recall that $\underline{c}_h^k(s,a) := \indi{a\in\cA_h^{\bar{t}_k-1}(s), n_h^{\bar{t}_k-1}(s,a) < k}$. Let $\textbf{1}^k : \cE \rightarrow [0,1]$ be another lower bound function such that  $\textbf{1}_h^k(s,a) = \indi{a\in\cA_h^{\bar{t}_k-1}(s)}$. Then,
\begin{align*}
\sum_{k=1}^{k_\tau} \varphi^\star(\underline{c}^k) \leq \sum_{k=1}^{k_\tau} \varphi^\star(\textbf{1}^k),
\end{align*}
where the inequality is due to Lemma \ref{lem:flow-mono} and $\underline{c}_h^k(s,a) \leq \textbf{1}_h^k(s,a)$ for all $s,a,h$. Let $k\geq 1$ and $\cC^k$ be any maximum cut for the minimum flow problem with lower bounds $\textbf{1}^k$. Then, by Theorem \ref{th:min-flow-max-cut},
\begin{align*}
\varphi^\star(\textbf{1}^k) = \psi(\cC^k, \textbf{1}^k) = \sum_{(s,a,h)\in\cE(\cC^k)}  \textbf{1}^k_h(s,a) = \sum_{(s,a,h)\in\cE(\cC^k)}  \indi{a\in\cA_h^{\bar{t}_k-1}(s)},
\end{align*}
where we recall that, since $\cC^k$ is a maximum cut, it has no backward arc and thus its value is simply the sum of lower bounds on its forward arcs. Plugging this back into our sample complexity bound,
\begin{align*}
\tau &\leq 2C_H\sum_{k=1}^{k_\tau} \sum_{(s,a,h)\in\cE(\cC^k)}  \indi{a\in\cA_h^{\bar{t}_k-1}(s)}
\\ & = 2C_H\sum_{k=1}^{k_\tau} \sum_{(s,a,h)\in\cE(\cC^k)}  \indi{k-1 \leq \kappa_{s,a,h}}
\\ & = 2C_H\sum_{k=1}^{k_\tau} \frac{1}{k}\sum_{(s,a,h)\in\cE(\cC^k)}  k \indi{k-1 \leq \kappa_{s,a,h}}
\\ &\leq 2C_H\sum_{k=1}^{k_\tau} \frac{1}{k}\sum_{(s,a,h)\in\cE(\cC^k)}  (\kappa_{s,a,h}+1)
\\ &\leq 2C_H\sum_{k=1}^{k_\tau} \frac{1}{k} \max_{\cC\in\mathfrak{C}} \sum_{(s,a,h)\in\cE(\cC)}  (\kappa_{s,a,h}+1)
\\ &\leq 2C_H(\log(k_\tau) + 1) \max_{\cC\in\mathfrak{C}} \sum_{(s,a,h)\in\cE(\cC)}  (\kappa_{s,a,h}+1)
\\ &\leq 2C_H(\log(\overline{\kappa}) + 1) \max_{\cC\in\mathfrak{C}} \sum_{(s,a,h)\in\cE(\cC)}  (\overline{\kappa}_{s,a,h}+1),
\end{align*}
where in the last inequality we applied Lemma \ref{lem:elim-periods}. Now note that the maximization in the last line computes a maximum cut for the problem with lower bound function $g : \cE \rightarrow [0,\infty)$ defined by $g_h(s,a) = \overline{\kappa}_{s,a,h}+1$. Therefore, the statement follows by applying  Theorem \ref{th:min-flow-max-cut}.
\end{proof}

\begin{theorem}\label{th:sample-comp-worst-case-adaptive}[Worst-case bound]
Under the good event $\cG$, the sample complexity of Algorithm \ref{alg:elimination-alg} when combined with either the static maximum-coverage or the maximum-coverage sampling rule is bounded by\hl{
\begin{align*}
\tau \leq \frac{256\sigma^2 SAH^2}{\epsilon^2} \left( \log\left(\frac{4SAH}{\delta}\right) + 4\log\left(\frac{512\sigma^2 SAH^2 \log\left(\frac{4SAH}{\delta}\right)}{\epsilon^2}\right) \right) + 8SAH + 2.
\end{align*}}
\end{theorem}
\begin{proof}
The proof is an easy extension of the one of Theorem \ref{th:max-diam-sample-comp} (part 2) where we only need to handle the fact that the maximum-diameter sampling rule is called once every two episodes. We report the full steps for completeness.

Take any time $T$ at the end of which the algorithm did not stop. \hl{Let $\bar{t}$ be the first time step where all active triplets $(s,a,h)$ have at least one visit. Note that $\bar{t} \leq N$ almost surely since max-coverage sampling visits at least one new triplet at each episode (see Appendix \ref{sec:periods}) and the same holds for max-diameter. Then, for any $\bar{t} \leq t \leq T$ such that $t$ is odd,
\begin{align*}
\frac{\epsilon}{2} \leq \max_{\pi\in\Pi^{t}} \sum_{h=1}^H b^{t}_h(s_h^\pi,a_h^\pi) = \sum_{h=1}^H b_h^{t}(s_h^{\pi^{t+1}},a_h^{\pi^{t+1}})
 \leq \sum_{h=1}^H\sqrt{\frac{\beta(t,\delta)}{n^{t}_h(s_h^{\pi^{t+1}},a_h^{\pi^{t+1}})\vee 1}},
\end{align*}
where the first inequality follows from the first stopping rule, the equality uses the fact that the second sampling rule is used at time $t+1$, and the last inequality uses the definition of the bonuses together with $n_h^t(s,a) \leq t$ and $n^{t}_h(s_h^{\pi^{t+1}},a_h^{\pi^{t+1}}) \geq 1$ by definition of $\bar{t}$. Let $\bar{n}_h^t(s,a) := \sum_{\bar{t} \leq l\leq t: (l\ \mathrm{mod}\ 2) = 0} \indi{(s_h^l,a_h^l) = (s,a)} + 1$ be the number of visits to $(s,a,h)$ restricted to even steps (i.e., those by the second sampling rule). Note that $n_h^t(s,a) \geq \bar{n}_h^t(s,a)$ for all $t\geq \bar{t}$. Then, we have the following sequence of inequalities (explained below):
\begin{align*}
\frac{\epsilon}{2}\lfloor (T-\bar{t}+1)/2 \rfloor 
&\stackrel{(a)}{\leq} \sum_{\bar{t} \leq t\leq T: (t\ \mathrm{mod}\ 2) = 1} \sum_{h=1}^H\sqrt{\frac{\beta(t,\delta)}{n^{t}_h(s_h^{\pi^{t+1}},a_h^{\pi^{t+1}}) \vee 1}}
\\ &\stackrel{(b)}{\leq} \sum_{\bar{t} \leq t\leq T: (t\ \mathrm{mod}\ 2) = 1} \sum_{h=1}^H\sqrt{\frac{\beta(t,\delta)}{\bar{n}^{t}_h(s_h^{\pi^{t+1}},a_h^{\pi^{t+1}})}}
\\ &\stackrel{(c)}{=} \sum_{h=1}^H \sum_{s\in\cS_h}\sum_{a\in\cA_h(s)}\sum_{\bar{t} \leq t\leq T: (t\ \mathrm{mod}\ 2) = 1}\indi{(s_h^{\pi^{t+1}},a_h^{\pi^{t+1}})=(s,a)} \sqrt{\frac{\beta(t,\delta)}{\bar{n}^{t}_h(s,a)}}
\\ &\stackrel{(d)}{\leq} \sqrt{\beta(T,\delta)}\sum_{h=1}^H \sum_{s\in\cS_h}\sum_{a\in\cA_h(s)}\sum_{\bar{t} \leq t\leq T: (t\ \mathrm{mod}\ 2) = 1}\indi{(s_h^{\pi^{t+1}},a_h^{\pi^{t+1}})=(s,a)} \sqrt{\frac{1}{\bar{n}^{t}_h(s,a)}}
\\ &\stackrel{(e)}{\leq} \sqrt{\beta(T,\delta)}\sum_{h=1}^H \sum_{s\in\cS_h}\sum_{a\in\cA_h(s)}\sum_{i=1}^{\bar{n}_h^T(s,a)}\sqrt{\frac{1}{i}}
\\ &\stackrel{(f)}{\leq} 2\sqrt{\beta(T,\delta)}\sum_{h=1}^H \sum_{s\in\cS_h}\sum_{a\in\cA_h(s)}\sqrt{\bar{n}_h^T(s,a)} 
\\ &\stackrel{(g)}{\leq} 2\sqrt{\beta(T,\delta) N\sum_{h=1}^H \sum_{s\in\cS_h}\sum_{a\in\cA_h(s)}\bar{n}_h^T(s,a)}
\\ &\stackrel{(h)}{\leq} 2\sqrt{\beta(T,\delta)NH\lfloor T/2 \rfloor},
\end{align*}
where (a) is by summing both sides of the inequality derived at the beginning over all odd $t$ from $\bar{t}$ to $T$, (b) uses that $n_h^t(s,a) \geq \bar{n}_h^t(s,a)$ for all $s,a,h,t\geq \bar{t}$, (c) is trivial, (d) uses the monotonicity of $\beta(\cdot,\delta)$, (e) uses the standard pigeon-hole principle, (f) uses the inequality $\sum_{i=1}^m \sqrt{1/i} \leq 2\sqrt{m}$, (g) uses Cauchy-Schwartz inequality, and (h) uses that the total number of even episodes up to time $T$ is $\lfloor T/2 \rfloor$. Therefore, we obtain the inequality
\begin{align*}
  \frac{\epsilon T}{4} \leq \sqrt{4\sigma^2NH T \left( \log\left(\frac{4N}{\delta}\right) + 2\log\left(T\right) \right)} + \epsilon N,
\end{align*}
where we used that $\bar{t}-1\leq N$. Taking the square of both sides and using $(x+y)^2\leq 2(x^2 + y^2)$,
\begin{align}\label{eq:worst-case-adaptive-ineq-T}
\frac{\epsilon^2 T^2}{16} \leq 8\sigma^2NH T \left( \log\left(\frac{4N}{\delta}\right) + 2\log\left(T\right) \right) + 2\epsilon^2 N^2,
\end{align}
Up to constants, this is the same inequality we obtain in \eqref{eq:worst-case-adaptive-ineq-T-2} for the proof of Theorem \ref{th:max-diam-sample-comp}. By repeating exactly the same steps as for Theorem \ref{th:max-diam-sample-comp}, we obtain
\begin{align*}
  T \leq \frac{256\sigma^2 NH}{\epsilon^2} \left( \log\left(\frac{4N}{\delta}\right) + 4\log\left(\frac{512\sigma^2 NH \log\left(\frac{4N}{\delta}\right)}{\epsilon^2}\right) \right) + 8N.
\end{align*}}
The proof is concluded by noting that $\tau$ cannot be larger than the bound above plus two (since the maximum-diameter rule is called only every two steps) and that $N\leq SAH$.
\end{proof}

\begin{proof}[Proof of Theorem \ref{th:max-cover-sample-comp}]
The proof simply combines Theorem \ref{th:sample-comp-instance-dependent-adaptive} and Theorem \ref{th:sample-comp-worst-case-adaptive} together with the fact that the good event $\cG$ holds with probability at least $1-\delta$ (Lemma \ref{lem:good-event-whp}).
\end{proof}

\subsection{Maximum-diameter sampling}\label{app:max-diam}
We now state the main Theorem which gives guarantees on the sample complexity of EPRL when it is coupled with Maximum Diameter sampling (Line 17 in Algorithm \ref{alg:elimination-alg}).
\begin{theorem}\label{th:max-diam-sample-comp}
With probability at least $1-\delta$, the sample complexity of Algorithm \ref{alg:elimination-alg} combined with the maximum-diameter sampling rule is bounded as
\hl{
\begin{align*}
\tau &\leq \sum_{h\in[H]} \sum_{s\in\cS_h}\sum_{a\in\cA_h(s)} \frac{32\sigma^2H^2}{\max\left(\overline{\Delta}_h(s,a),{\overline{\Delta}_{\min}}, {\epsilon} \right)^2} \left(\log\left(\frac{4N}{\delta}\right) + L\right) + N + 1,
\end{align*}
where $L := 8\log\left( \frac{16N \sigma H \log\left(\frac{4N}{\delta}\right)}{\epsilon} \right)$. Moreover, with the same probability,
\begin{align*}
\tau \leq \frac{128\sigma^2 SAH^2}{\epsilon^2} \left( \log\left(\frac{4SAH}{\delta}\right) + 4\log\left(\frac{256\sigma^2 SAH^2 \log\left(\frac{4SAH}{\delta}\right)}{\epsilon^2}\right) \right) + 2SAH + 1.
\end{align*}}
\end{theorem}
\begin{proof}

We first derive the instance-dependent bound and then focus on the worst-case one separately.

\paragraph{Part 1. Instance-dependent bound.}

We use the following ``target trick'' to obtain a sample complexity which scales as the sum of inverse gaps. Instead of bounding the number of times each state-action pair is visited, we imagine that each played policy ``targets'' some state-action pair and bound the number of times each state-action pair is targeted. Formally, we say that the policy $\pi^t$ played at time $t$ targets $(s,a)$ at stage $h$ if the following event occurs:
\begin{align*}
G_{s,a,h}^t := \left\{ h \in \argmin_{l\in[H]} n_l^{t-1}(s_l^{\pi^t},a_l^{\pi^t}), s_h^{\pi^t} = s, a_h^{\pi^t} = a \right\}.
\end{align*}
Intuitively, we say that policy $\pi^t$ targets the state-action pair (along its trajectory) that has been visited the least so far. Then, since at each time step at least one state-action-stage triplet is targeted,
\begin{align}\label{eq:target-trick-tau}
\tau \leq \sum_{h=1}^H \sum_{s\in\cS_h}\sum_{a\in\cA_h(s)} Z_h^\tau(s,a),
\end{align}
where $Z_h^\tau(s,a) := \sum_{t=1}^\tau \indi{G_{s,a,h}^t}$ is the number of times $(s,a,h)$ is targeted up to the stopping time. Thus, we shall focus on bounding $Z_h^t(s,a)$ for some fixed time $t$. Note that $Z_h^t(s,a) \leq n_h^t(s,a)$ since a targeted state-action-stage triplet is necessarily visited at time $t$. Moreover, $n_h^t(s,a)$ can be much larger than $Z_h^t(s,a)$ since $(s,a,h)$ could be visited even without being the target.

\paragraph{Bounding $Z_h^t(s,a)$}

Let $t$ be any episode at which the algorithm did not stop. For any $(s,a,h)$, 
\begin{align*}
 \max\left(\frac{\overline{\Delta}_h(s,a)}{4},\frac{\overline{\Delta}_{\min}}{4}, \frac{\epsilon}{2} \right) 
 &\stackrel{(a)}{\leq} \max_{\pi\in\Pi^{t-1}}\sum_{h=1}^H b_h^t(s_h^\pi,a_h^\pi) 
   \stackrel{(b)}{\leq} \max_{\pi\in\Pi^{t-1}}\sum_{h=1}^H b_h^{t-1}(s_h^\pi,a_h^\pi)
  \\ &\stackrel{(c)}{=} 
 \sum_{h=1}^H b_h^{t-1}(s_h^{\pi^t},a_h^{\pi^t})
 \stackrel{(d)}{\leq} \hl{\sum_{h=1}^H\sqrt{\frac{\beta(t,\delta)}{n^{t-1}_h(s_h^{\pi^t},a_h^{\pi^t})}}},
\end{align*}
where (a) is from Lemma \ref{lem:diam-vs-gap}, (b) from the monotonicity of the bonuses, (c) from the definition of the maximum-diameter sampling rule, and (d) from the definition of the bonuses. Now we distinguish two cases. If $G_{s,a,h}^t$ holds, then
\begin{align*}
\forall l\in[H] : n^{t-1}_l(s_l^{\pi^t},a_l^{\pi^t}) \geq n^{t-1}_h(s,a) \geq Z_{h}^{t-1}(s,a).
\end{align*}
Plugging this into the inequality above and rearranging, we obtain
\begin{align*}
\max\left(\frac{\overline{\Delta}_h(s,a)}{4},\frac{\overline{\Delta}_{\min}}{4}, \frac{\epsilon}{2} \right) \leq H\sqrt{\frac{\beta(t,\delta)}{Z_{h}^{t-1}(s,a)}} \implies Z_h^t(s,a) \leq \frac{16H^2\beta(t,\delta)}{\max\left(\overline{\Delta}_h(s,a),{\overline{\Delta}_{\min}}, {\epsilon} \right)^2} + 1.
\end{align*}
On the other hand, in case $G_{s,a,h}^t$ does not hold, then $Z_h^t(s,a) = Z_h^{t-1}(s,a)$ and we can recursively apply the reasoning above to obtain the same bound on $Z_h^t(s,a)$. 

\paragraph{Bounding $\tau$}

Evaluating this bound at $t=\tau-1$, plugging it into \eqref{eq:target-trick-tau}, and expanding the definition of the threshold $\beta$, we obtain
\hl{
\begin{align*}
\tau &\leq \sum_{h\in[H]} \sum_{s\in\cS_h}\sum_{a\in\cA_h(s)} \left( \frac{16H^2\beta(\tau-1,\delta)}{\max\left(\overline{\Delta}_h(s,a),{\overline{\Delta}_{\min}}, {\epsilon} \right)^2} + 1 \right)  + 1
\\ &\leq \sum_{h\in[H]} \sum_{s\in\cS_h}\sum_{a\in\cA_h(s)} \frac{32\sigma^2H^2}{\max\left(\overline{\Delta}_h(s,a),{\overline{\Delta}_{\min}}, {\epsilon} \right)^2} \left(\log\left(\frac{4N}{\delta}\right) + 2\log(\tau) \right) + N + 1.
\end{align*}
We conclude by applying Lemma \ref{lem:simplify-ineq-2} with $B = \sum_{h\in[H]} \sum_{s\in\cS_h}\sum_{a\in\cA_h(s)} \frac{64\sigma^2H^2}{\max\left(\overline{\Delta}_h(s,a),{\overline{\Delta}_{\min}}, {\epsilon} \right)^2}$ and $C = \sum_{h\in[H]} \sum_{s\in\cS_h}\sum_{a\in\cA_h(s)} \frac{32\sigma^2H^2}{\max\left(\overline{\Delta}_h(s,a),{\overline{\Delta}_{\min}}, {\epsilon} \right)^2} \log\left(\frac{4N}{\delta}\right) + N + 1 $, while noting that
\begin{align*}
\log(B^2 + 2C) 
&\leq \log\left( 2 \left(\frac{64N \sigma^2H^2 \log\left(\frac{4N}{\delta}\right)}{\epsilon^2}\right)^2 + 2N + 2\right)
\\ &\leq \log\left( 4 \left(\frac{64N \sigma^2H^2 \log\left(\frac{4N}{\delta}\right)}{\epsilon^2}\right)^2 \right) 
\leq 4\log\left( \frac{16N \sigma H \log\left(\frac{4N}{\delta}\right)}{\epsilon} \right).
\end{align*}}

\paragraph{Part 2. Worst-case bound.}

Take any time $T$ at the end of which the algorithm did not stop. \hl{Let $\bar{t}$ be the first time step where all active triplets $(s,a,h)$ have at least one visit. Note that $\bar{t} \leq N$ almost surely by definition of the sampling rule: since an active unvisited triplet has infinite confidence interval, the algorithm must visit at least a new one of such triplets in each episode. Then, for any $\bar{t} \leq t \leq T$,}
\begin{align*}
\frac{\epsilon}{2} \leq \max_{\pi\in\Pi^{t}} \sum_{h=1}^H b^{t}_h(s_h^\pi,a_h^\pi) = \sum_{h=1}^H b_h^{t}(s_h^{\pi^{t+1}},a_h^{\pi^{t+1}})
 \leq \sum_{h=1}^H\sqrt{\frac{\beta(t,\delta)}{n^{t}_h(s_h^{\pi^{t+1}},a_h^{\pi^{t+1}})\vee 1}},
\end{align*}
where the first inequality follows from the first stopping rule, the equality uses the definition of the maximum-diameter sampling rule, and \hl{the last inequality uses the definition of the bonuses together with $n_h^t(s,a) \leq t$ and $n^{t}_h(s_h^{\pi^{t+1}},a_h^{\pi^{t+1}}) \geq 1$ since $t \geq \bar{t}$. Then,
\begin{align*}
\frac{\epsilon}{2}(T-\bar{t}+1) 
&\stackrel{(a)}{\leq} \sum_{t=\bar{t}}^T\sum_{h=1}^H\sqrt{\frac{\beta(t,\delta)}{n^{t}_h(s_h^{\pi^{t+1}},a_h^{\pi^{t+1}}) \vee 1}}
\\ &\stackrel{(b)}{=} \sum_{h=1}^H \sum_{s\in\cS_h}\sum_{a\in\cA_h(s)}\sum_{t=\bar{t}}^T\indi{(s_h^{\pi^{t+1}},a_h^{\pi^{t+1}})=(s,a)} \sqrt{\frac{\beta(t,\delta)}{{n}^{t}_h(s,a) \vee 1}}
\\ &\stackrel{(c)}{\leq} \sqrt{\beta(T,\delta)}\sum_{h=1}^H \sum_{s\in\cS_h}\sum_{a\in\cA_h(s)}\sum_{t=\bar{t}}^T\indi{(s_h^{\pi^{t+1}},a_h^{\pi^{t+1}})=(s,a)} \sqrt{\frac{1}{{n}^{t}_h(s,a) \vee 1}}
\\ &\stackrel{(d)}{\leq} \sqrt{\beta(T,\delta)}\sum_{h=1}^H \sum_{s\in\cS_h}\sum_{a\in\cA_h(s)}\sum_{i=2}^{{n}_h^T(s,a)}\sqrt{\frac{1}{i-1}}
\\ &\stackrel{(e)}{\leq} 2\sqrt{\beta(T,\delta)}\sum_{h=1}^H \sum_{s\in\cS_h}\sum_{a\in\cA_h(s)}\sqrt{{n}_h^T(s,a)-1}
\\ &\stackrel{(f)}{\leq} 2\sqrt{\beta(T,\delta) N\sum_{h=1}^H \sum_{s\in\cS_h}\sum_{a\in\cA_h(s)}{n}_h^T(s,a)}
\\ &\stackrel{(g)}{=} 2\sqrt{\beta(T,\delta)NHT} ,
\end{align*}}
where (a) is by summing both sides of the inequality derived at the beginning over all $t$ from $0$ to $T$, (b) is trivial, (c) uses the monotonicity of $\beta(\cdot,\delta)$, (d) uses the standard pigeon-hole principle, (e) uses the inequality $\sum_{i=1}^m \sqrt{1/i} \leq 2\sqrt{m}$, (f) uses Cauchy-Schwartz inequality, and (g) uses that the total number of samples after $T$ episodes is $TH$. Therefore, we obtain the inequality,
\hl{\begin{align*}
  \frac{\epsilon T}{2} \leq 2\sqrt{2\sigma^2 NHT \left( \log\left(\frac{4N}{\delta}\right) + 2\log\left(T\right) \right)} + \frac{\epsilon}{2} N,
\end{align*}
where we used $\bar{t} - 1 \leq N$.
Taking the square of both sides and using $(x+y)^2\leq 2(x^2 + y^2)$,
\begin{align}\label{eq:worst-case-adaptive-ineq-T-2}
\frac{\epsilon^2 T^2}{4} \leq 16\sigma^2 NHT \left( \log\left(\frac{4N}{\delta}\right) + 2\log\left(T\right) \right) + \frac{\epsilon^2}{2} N^2
\end{align}
This implies that the lhs is below twice the maximum between the two terms in the rhs. Suppose the first term in \eqref{eq:worst-case-adaptive-ineq-T-2} is the maximum. Then,
\begin{align}\label{eq:abc1}
  \frac{\epsilon^2 T^2}{4} \leq 32\sigma^2 NHT \left( \log\left(\frac{4N}{\delta}\right) + 2\log\left(T\right) \right).
\end{align}
Using $\log(T)\leq\sqrt{T}$, a crude bound on $T$ is
\begin{align*}
  \frac{\epsilon^2 T^2}{4} \leq 32\sigma^2 NHT \left( \log\left(\frac{4N}{\delta}\right) + 2\sqrt{T} \right) \leq 64\sigma^2 NHT^{3/2} \log\left(\frac{4N}{\delta}\right),
\end{align*}
which implies that
\begin{align*}
  T \leq \left(\frac{256\sigma^2 NH \log\left(\frac{4N}{\delta}\right)}{\epsilon^2}\right)^2.
\end{align*}
Plugging this into the log term in \eqref{eq:abc1} and solving for $T$,
\begin{align*}
  T \leq \frac{128\sigma^2 NH}{\epsilon^2} \left( \log\left(\frac{4N}{\delta}\right) + 4\log\left(\frac{256\sigma^2 NH \log\left(\frac{4N}{\delta}\right)}{\epsilon^2}\right) \right).
\end{align*}

Suppose now that the second term in \eqref{eq:worst-case-adaptive-ineq-T-2} is the maximum. Then, we directly get $T \leq 2N$.Then, $T$ must be below the maximum of the two obtained bounds and hence below their sum,
\begin{align*}
  T \leq \frac{128\sigma^2 NH}{\epsilon^2} \left( \log\left(\frac{4N}{\delta}\right) + 4\log\left(\frac{256\sigma^2 NH \log\left(\frac{4N}{\delta}\right)}{\epsilon^2}\right) \right) + 2N.
\end{align*}}
This holds for any $T$ at the end of which the algorithm did not stop. Therefore, $\tau$ cannot be larger than the bound above plus one. The proof is concluded by noting that $N\leq SAH$.
\end{proof}


\subsection{Auxiliary Results}

\begin{lemma}\label{lem:simplify-ineq-2}
Let $B,C \geq 1$. If $k \leq B\log(k) + C$, then
\begin{align*}
k \leq B\log(B^2 + 2C) + C.
\end{align*}
\end{lemma}
\begin{proof}
Since $\log(k) \leq \sqrt{k}$ for any $k\geq 1$, we have that $k \leq B\sqrt{k} + C$.
Solving this second-order inequality, we get the crude bound $\sqrt{k} \leq \frac{B}{2} + \sqrt{\frac{B^2}{4} + C}$, which in turns yields $k \leq B^2 + 2C$ using that $(x+y)^2 \leq 2(x^2+y^2)$ for $x,y\geq 0$. The statement follows by plugging this bound into the logarithm.
\end{proof}


%% file: sections/app_tree.tex

\section{Refined Results for Tree-based MDPs} \label{app:tree}

In this appendix, we show that all our results can be refined for the specific class of deterministic MDPs represented by a tree, i.e., where each reachable state has exactly one incoming arc (except for the initial state which has none). This implies that there exists a unique path to reach each state $s\in\cS_h$ at stage $h>1$ from the root.

\subsection{Instance-dependent lower bound}

In the case of tree-based MDPs, one can derive a lower bound with an improved $H$ factor. The intuition behind this result is the following: While in general MDPs the policies going through different triplets $(s,a,h)$ and $(s',a',h)$ may share some common state-action pairs at any further stage $l \ge h$, such phenomenon does not occur in tree-based MDPs. This makes the learning problem more difficult, as learning whether $(s,a,h)$ is optimal or not does not gives us side-information about $(s',a',h)$. Throughout this section, we will be using the same notation as Section \ref{sec:lower-bounds}.
\begin{theorem}\label{th:LB_tree_based}
Suppose that $\cM$ is tree-based. Then:
\begin{align*}
    \bE[\tau] \geq \max_h \sum_{s\in \cS_h, a\in \cA_h(s) }\frac{\hl{\sigma^2}(H-h+1)\log(1/4\delta)}{4\max(\overline{\Delta}_h(s,a), \overline{\Delta}_{\min}^h, \epsilon)^2},
\end{align*}
where $\overline{\Delta}_{\min}^h := \min_{(s',a') : \overline{\Delta}_h(s',a') > 0} \overline{\Delta}_h(s',a')$.

\end{theorem}
The proof of this theorem relies on the following lemma which is refined version of Lemmas \ref{lem:lb-suboptimal}, \ref{lem:lb-optimal-nonunique} and \ref{lem:lb-optimal-unique} in the case of tree-based MDPs. Before we state the lemma, we define for any triplet $(s,a,h)$ the set $$E(s,a,h) = \big\{(s',a',l):\ l \in [|h,H|],\ s'\in \cS_l,\ a'\in\cA_l(s),\ \exists \pi \in \Pi_{s,a,h},\ s_l^\pi = s', a_l^\pi = a' \big\}.$$
In words, $E(s,a,h)$ is the set of triplets at stages $l\geq h$ that are visited by the policies in the set $\Pi_{s,a,h}$.
\begin{lemma}\label{lem:lower-bounds-tree-based}
Suppose that $\cM$ is tree-based and fix any stage $h\in [H]$. We have:
\begin{enumerate}
    \item For suboptimal pairs $(s,a) \notin \cZ_h^\epsilon$: 
    \begin{align*}
        \sum_{(s',a',l) \in E(s,a,h)} \bE[n_l^\tau(s',a')] \geq \frac{2\hl{\sigma^2}(H-h+1)^2}{(\overline{\Delta}_h(s,a)+\epsilon)^2}\log(1/2.4\delta).
    \end{align*}
    \item For non-unique optimal pairs $(s,a) \in \cZ_h^\epsilon$ and $|\cZ_h^\epsilon|>1$:
        \begin{align*}
        \sum_{(s',a',l) \in E(s,a,h)} \bE[n_l^\tau(s',a')] \geq \frac{\hl{\sigma^2}(H-h+1)^2}{4\epsilon^2}\log(1/4\delta).
    \end{align*}
    \item For unique optimal pairs $(s,a) \in \cZ_h^\epsilon$ and $|\cZ_h^\epsilon| = 1$:
     \begin{align*}
        \sum_{(s',a',l) \in E(s,a,h)} \bE[n_l^\tau(s',a')]\geq \frac{2\hl{\sigma^2}(H-h+1)^2}{(\overline{\Delta}_{\min}^h+\epsilon)^2}\log(1/4\delta),
    \end{align*}
    where $\overline{\Delta}_{\min}^h := \min_{(s',a') : \overline{\Delta}_h(s',a') > 0} \overline{\Delta}_h(s',a')$.
\end{enumerate}

\end{lemma}
\begin{proof}
We distinguish four cases. 
\paragraph{ Case 1: $(s,a) \notin \cZ_h^\epsilon$.}
Consider the alternative MDP $\widetilde{\cM} := (\cS, \cA, \{f_h,\widetilde{\nu}_h\}_{h\in[H]}, s_1, H)$ which is equivalent to $\cM$ except that the reward is f only at the pairs  $(s',a',l) \in E(s,a,h)$ as $\widetilde{\nu}_l(s',a') = \cN(r_l(s',a') + \Delta, \hl{\sigma^2})$ with $\Delta > \frac{\overline{\Delta}_h(s,a)+\epsilon}{H-h+1}$, while the reward distribution remains the same on all other state-action-stage triplets. Note that the values of policies in $\Pi\setminus\Pi_{s,a,h}$ remain unchanged. On the other hand, for all $\pi \in \argmax_{\pi \in \Pi_{s,a,h}} V_1^\pi(s_1)$,

\[\widetilde{V}_1^{\pi}(s_1) = V_1^{\pi}(s_1)+ (H-h+1)\Delta > V_1^{\pi}(s_1)+ \overline{\Delta}_h(s,a)+\epsilon = V_1^{\star}(s_1)+\epsilon \ge \max_{\pi \notin \Pi_{s,a,h}} \widetilde{V}_1^{\pi}(s_1)+\epsilon,
\]
where the first equality is because we increased the mean reward by $\Delta$ at the pairs $(s_l^\pi, a_l^\pi)_{l\in [|h,H|]}$ and the second equality comes from the definition of $\overline{\Delta}_h(s,a)$ and the fact that $\pi \in \argmax_{\pi \in \Pi_{s,a,h}} V_1^\pi(s_1)$. Now from the inequality above we deduce that $\bP_{\widetilde{\cM}}(\widehat{\pi} \in \Pi_{s,a,h}) \geq 1-\delta$. On the other hand, since $(s,a) \notin \cZ_h^\epsilon$, $\bP_{\cM}(\widehat{\pi} \in \Pi_{s,a,h}) \leq \delta$. Therefore Lemma 1 from \cite{kaufmann2016complexity} implies that:
\begin{align*}
    \sum_{(s',a',l) \in E(s,a,h)} \bE[n_l^\tau(s',a')]&\geq \frac{2}{\Delta^2}\mathrm{kl}(\bP_{{\cM}}(\widehat{\pi} \in \Pi_{s,a,h}),\bP_{\widetilde{\cM}}(\widehat{\pi} \in \Pi_{s,a,h}))\\
    &\geq \frac{2\hl{\sigma^2}}{\Delta^2}\kl(\delta, 1-\delta) \geq \frac{2\hl{\sigma^2}}{\Delta^2}\log(1/2.4\delta).
\end{align*}
This holds for any $\Delta > \frac{\overline{\Delta}_h(s,a)+\epsilon}{H-h+1}$ and the first statement is obtained by taking the limit.

\paragraph{Case 2: $(s,a) \in \cZ_h^\epsilon,\ |\cZ_h^\epsilon|>1$ and $\bP_{\cM}(\widehat{\pi} \in \Pi_{s,a,h}) \leq 1/2$.} We consider the same $\widetilde{\cM}$ from the previous case. We still have that $\bP_{\widetilde{\cM}}(\widehat{\pi} \in \Pi_{s,a,h}) \geq 1-\delta$. Using Lemma 1 from \cite{kaufmann2016complexity} we get
$$\sum_{(s',a',l) \in E(s,a,h)} \bE[n_l^\tau(s',a')] \geq \frac{2}{\Delta^2}\mathrm{kl}(\bP_{{\cM}}(\widehat{\pi} \in \Pi_{s,a,h}),\bP_{\widetilde{\cM}}(\widehat{\pi} \in \Pi_{s,a,h})) \geq \frac{2\hl{\sigma^2}}{\Delta^2}\kl(1/2, 1-\delta).$$ 
By taking the limit $\Delta \to \frac{\overline{\Delta}_h(s,a)+\epsilon}{H-h+1}$ we get:
\begin{align*}
    \sum_{(s',a',l) \in E(s,a,h)} \bE[n_l^\tau(s',a')]
    &\geq \frac{2\hl{\sigma^2}}{(\overline{\Delta}_h(s,a)+\epsilon)^2}\kl(1/2, 1-\delta)\\
    &= \frac{2\hl{\sigma^2}(H-h+1)^2}{\left(\overline{\Delta}_h(s,a)+\varepsilon\right)^2} \kl(1/2, \delta)\\
   & \ge \frac{\hl{\sigma^2}(H-h+1)^2}{4\varepsilon^2} \log(1/4\delta),
\end{align*}
where we used the fact that $\kl(x,y) = \kl(1-x,1-y)$, $\kl(x,y) \geq x\ln(1/y) - \ln(2)$ and $\overline{\Delta}_h(s,a) \le \epsilon$.

\paragraph{Case 3: $(s,a) \in \cZ_h^\epsilon,\ |\cZ_h^\epsilon|>1$ and $\bP_{\cM}(\widehat{\pi} \in \Pi_{s,a,h}) \geq 1/2$.} Consider the alternative MDP $\widetilde{\cM} := (\cS, \cA, \{f_h,\widetilde{\nu}_h\}_{h\in[H]}, s_1, H)$ which is equivalent to $\cM$ except that the reward is modified only at the pairs  $(s',a',l) \in E(s,a,h)$ as $\widetilde{\nu}_l(s',a') = \cN(r_l(s',a') - \Delta, \hl{\sigma^2})$ with $\Delta > \frac{2\epsilon - \overline{\Delta}_h(s,a)}{H-h+1}$, while the reward distribution remains the same on all other state-action-stage triplets. Note that the values of policies in $\Pi\setminus\Pi_{s,a,h}$ remain unchanged. On the other hand, for all $\pi \in \Pi_{s,a,h}$,
\begin{align*}
\widetilde{V}_1^{\pi}(s_1) &= V_1^{\pi}(s_1)- (H-h+1)\Delta \\
&< V_1^{\pi}(s_1)+ \overline{\Delta}_h(s,a)-2\epsilon \\
& \leq V_1^{\star}(s_1)- 2\epsilon \\
&\leq \max_{\pi \notin \Pi_{s,a,h}} V_1^{\pi}(s_1)-\epsilon = \max_{\pi \notin \Pi_{s,a,h}} \widetilde{V}_1^{\pi}(s_1)-\epsilon,    
\end{align*}
where the first equality is because we decreased the mean reward by $\Delta$ at the pairs $(s_l^\pi, a_l^\pi)_{l\in [|h,H|]}$ and the last inequality is due to the fact that since $|\cZ_h^\epsilon| > 1$, there exists at least one $\epsilon$-optimal policy which does not visit $(s,a)$ at step $h$ (i.e., which belongs to $\Pi\setminus\Pi_{s,a,h}$). From the inequality above we deduce that $\bP_{\widetilde{\cM}}(\widehat{\pi} \in \Pi_{s,a,h}) \leq \delta$. Applying Lemma 1 from \cite{kaufmann2016complexity} to $\cM$ and $\widetilde{\cM}$ gives:
\begin{align*}
    \sum_{(s',a',l) \in E(s,a,h)} \bE[n_l^\tau(s',a')]&\geq \frac{2\hl{\sigma^2}}{\Delta^2}\mathrm{kl}(\bP_{{\cM}}(\widehat{\pi} \in \Pi_{s,a,h}),\bP_{\widetilde{\cM}}(\widehat{\pi} \in \Pi_{s,a,h}))\\
    &\geq \frac{2\hl{\sigma^2}}{\Delta^2}\kl(1/2, \delta).
\end{align*}
By taking the limit $\Delta \to \frac{2\epsilon - \overline{\Delta}_h(s,a)}{H-h+1}$ we get:
\begin{align*}
    \sum_{(s',a',l) \in E(s,a,h)} \bE[n_l^\tau(s',a')]
    &\geq \frac{2\hl{\sigma^2}(H-h+1)^2}{\left(2\epsilon-\overline{\Delta}_h(s,a)\right)^2}\kl(1/2, \delta)\\
   & \ge \frac{\hl{\sigma^2}(H-h+1)^2}{4\varepsilon^2} \log(1/4\delta),
\end{align*}
where we used the fact that $\kl(x,y) \geq x\ln(1/y) - \ln(2)$ and $\overline{\Delta}_h(s,a) \le \epsilon$. Cases 2 and 3 combined prove the second statement of the lemma.

\paragraph{Case 4: $(s,a) \in \cZ_h^\epsilon,\ |\cZ_h^\epsilon|=1$.} 
Consider the alternative MDP $\widetilde{\cM} := (\cS, \cA, \{f_h,\widetilde{\nu}_h\}_{h\in[H]}, s_1, H)$ which is equivalent to $\cM$ except that the reward is modified only at the pairs  $(s',a',l) \in E(s,a,h)$ as $\widetilde{\nu}_l(s',a') = \cN(r_l(s',a') - \Delta, \hl{\sigma^2})$ with $\Delta > \frac{\epsilon + \overline{\Delta}_{\min}^h}{H-h+1}$, while the reward distribution remains the same on all other state-action-stage triplets. Note that the values of policies in $\Pi\setminus\Pi_{s,a,h}$ remain unchanged. On the other hand, for all $\pi \in \Pi_{s,a,h}$,
\begin{align*}
\widetilde{V}_1^{\pi}(s_1) &= V_1^{\pi}(s_1)- (H-h+1)\Delta \\
&< V_1^{\pi}(s_1)-\overline{\Delta}_{\min}^h -\epsilon\\
& \leq V_1^{\star}(s_1)-\overline{\Delta}_{\min}^h -\epsilon \\
&= \max_{\pi \notin \Pi_{s,a,h}} V_1^{\pi}(s_1)-\epsilon = \max_{\pi \notin \Pi_{s,a,h}} \widetilde{V}_1^{\pi}(s_1)-\epsilon,    
\end{align*}
where in the last equality we used the fact that since $(s,a)$ is the only $\epsilon$-optimal pair, $\{(s',a'):\ \overline{\Delta}_h(s',a') =0\} = \{(s,a)\}$ and therefore $\overline{\Delta}_{\min}^h = \min_{(s',a')\neq (s,a)} \overline{\Delta}_h(s',a') = V_1^{\star}(s_1)-\max_{\pi \notin \Pi_{s,a,h}} V_1^{\pi}(s_1)$. From the inequality above we deduce that $\bP_{\widetilde{\cM}}(\widehat{\pi} \in \Pi_{s,a,h}) \leq \delta$. On the other hand, since $(s,a)$ is the only $\epsilon$-optimal pair in $\cM, \bP_{\cM}(\widehat{\pi} \in \Pi_{s,a,h}) \geq 1-\delta$. Using Lemma 1 from \cite{kaufmann2016complexity} to $\cM$ and $\widetilde{\cM}$ gives:
\begin{align*}
    \sum_{(s',a',l) \in E(s,a,h)} \bE[n_l^\tau(s',a')]&\geq \frac{2\hl{\sigma^2}}{\Delta^2}\mathrm{kl}(\bP_{{\cM}}(\widehat{\pi} \in \Pi_{s,a,h}),\bP_{\widetilde{\cM}}(\widehat{\pi} \in \Pi_{s,a,h}))\\
    &\geq \frac{2\hl{\sigma^2}}{\Delta^2}\kl(1-\delta, \delta) \geq \frac{2\hl{\sigma^2}}{\Delta^2}\log(1/2.4\delta).
\end{align*}
By taking the limit $\Delta \to \frac{\epsilon + \overline{\Delta}_{\min}^h}{H-h+1}$ we get:
\begin{align*}
    \sum_{(s',a',l) \in E(s,a,h)} \bE[n_l^\tau(s',a')]
    &\geq \frac{2\hl{\sigma^2}(H-h+1)^2}{(\epsilon + \overline{\Delta}_{\min}^h)^2}\log(1/2.4\delta).
\end{align*}
This proves the last statement of the lemma.
\end{proof}

We are now ready to prove Theorem \ref{th:LB_tree_based}.
\begin{proof}[Proof of Theorem \ref{th:LB_tree_based}]
Fix a stage $h\in [H]$. Since in a tree-based MDP the policies that visit different triplets at stage $h$ do not cross paths later, then for any $(s,a,h) \neq (s',a',h):\ E(s,a,h)\cap E(s',a',h) = \emptyset$. Besides $\underset{s\in \cS_h, a \in \cA_h(s)}{\bigcup} E(s,a,h) \subset \{(s',a', l):\ l \in [|h,H|],\ s' \in \cS_l,\ a'\in \cA_l(s) \}$. Therefore one can write:
\begin{align}
    \bE[\tau] &= \frac{1}{H-h+1} \sum_{l=h}^{H} \sum_{s'\in \cS_l, a'\in \cA_l(s)} \bE[n_l^\tau(s',a')] \nonumber\\  
    &\ge \frac{1}{H-h+1}\sum_{s\in \cS_h, a\in \cA_h(s)}\sum_{(s',a',l) \in E(s,a,h)} \bE[n_l^\tau(s',a')]
\label{eq:tau_counting}
\end{align}
Combining inequality (\ref{eq:tau_counting}) with the bounds from Lemma \ref{lem:lower-bounds-tree-based} finishes the proof.
\end{proof}

\subsection{Sample complexity of maximum-diameter sampling}

\begin{theorem}\label{th:max-diam-sample-comp-tree}
With probability at least $1-\delta$, the sample complexity of Algorithm \ref{alg:elimination-alg} combined with the maximum-diameter sampling rule (Line 17 of Algorithm \ref{alg:elimination-alg}) is bounded as\hl{
\begin{align*}
\tau &\leq \max_{h\in[H]} \sum_{s\in\cS_h}\sum_{a\in\cA_h(s)} \left( \frac{128\sigma^2H^2}{\max\left(\overline{\Delta}_h(s,a),{\overline{\Delta}_{\min}}, {\epsilon} \right)^2} \left(\log\left(\frac{4N}{\delta}\right) + L\right) + 2 \right),
\end{align*}
where $L := 8\log \left( \frac{32\sigma N H \log\left(\frac{4N}{\delta}\right)}{\epsilon} \right)$.}
\end{theorem}
\begin{proof}
\hl{Suppose that event $\cG$ holds and let $t$ be any episode at which the algorithm did not stop. For any active $(s,a,h)$, by Lemma \ref{lem:diam-vs-gap} and the same decomposition as in the proof of Theorem \ref{th:max-diam-sample-comp},
\begin{align*}
 \max\left(\frac{\overline{\Delta}_H(s,a)}{4},\frac{\overline{\Delta}_{\min}}{4}, \frac{\epsilon}{2} \right) \leq \sum_{h=1}^H\sqrt{\frac{\beta(t,\delta)}{n^{t-1}_h(s_h^{\pi^t},a_h^{\pi^t})}} \leq H\sqrt{\frac{\beta(t,\delta)}{n^{t-1}_H(s_H^{\pi^t},a_H^{\pi^t})}},
\end{align*}
where the last inequality holds since in a tree-based MDP there exists a unique path to reach each leaf, which implies that $\forall h\in[H] : n^{t-1}_h(s_h^{\pi^t},a_h^{\pi^t}) \geq n^{t-1}_H(s_H^{\pi^t},a_H^{\pi^t})$. Summing this inequality over all episodes where $(s,a)$ is visited at the final stage $H$ starting from its second visit up to episode $T$,
\begin{align*}
 \max\left(\frac{\overline{\Delta}_H(s,a)}{4},\frac{\overline{\Delta}_{\min}}{4}, \frac{\epsilon}{2} \right) (n_H^{T}(s,a)-1) &\leq 
 H\sqrt{\beta(T,\delta)}\sum_{i=2}^{n_H^T(s,a)} \sqrt{\frac{1}{i-1}}
 \\ &\leq 2H\sqrt{\beta(T,\delta)(n_H^T(s,a)-1)}.
\end{align*}
Solving the inequality above,
\begin{align*}
n_H^T(s,a) \leq \frac{64 H^2\beta(T,\delta)}{\max\left(\overline{\Delta}_h(s,a),{\overline{\Delta}_{\min}}, {\epsilon} \right)^2} + 1.
\end{align*}
Evaluating this bound at $T=\tau-1$,
\begin{align*}
\tau &= \sum_{s\in\cS_H}\sum_{a\in\cA_H(s)} n_H^{\tau}(s,a) \leq
 \sum_{s\in\cS_H}\sum_{a\in\cA_H(s)} \left( \frac{64H^2\beta(\tau,\delta)}{\max\left(\overline{\Delta}_h(s,a),{\overline{\Delta}_{\min}}, {\epsilon} \right)^2} + 2 \right)
\\ &= \sum_{s\in\cS_H}\sum_{a\in\cA_H(s)} \left(\frac{128\sigma^2 H^2}{\max\left(\overline{\Delta}_H(s,a),{\overline{\Delta}_{\min}}, {\epsilon} \right)^2} \left(\log\left(\frac{4N}{\delta}\right) + 2\log(\tau) \right) + 2\right).
\end{align*}
The proof is concluded by applying Lemma \ref{lem:simplify-ineq-2} with $B = \sum_{s\in\cS_H}\sum_{a\in\cA_H(s)} \frac{256\sigma^2 H^2}{\max\left(\overline{\Delta}_H(s,a),{\overline{\Delta}_{\min}}, {\epsilon} \right)^2}$ and $C = \sum_{s\in\cS_H}\sum_{a\in\cA_H(s)}\left( \frac{128\sigma^2 H^2}{\max\left(\overline{\Delta}_H(s,a),{\overline{\Delta}_{\min}}, {\epsilon} \right)^2} \log\left(\frac{4N}{\delta}\right) + 2 \right)$, while noting that
\begin{align*}
    \log(B^2 + 2C) \leq \log \left( 2 \left( \frac{256\sigma^2 N H^2 \log\left(\frac{4N}{\delta}\right)}{\epsilon^2} \right)^2 + 2N\right) \leq 4\log \left( \frac{32\sigma N H \log\left(\frac{4N}{\delta}\right)}{\epsilon} \right).
\end{align*}}
\end{proof}

\subsection{Sample complexity of maximum-coverage sampling}

\begin{theorem}\label{th:max-cover-sample-comp-tree}
With probability at least $1-\delta$, the sample complexity of Algorithm \ref{alg:elimination-alg} combined with either the maximum-coverage or the static maximum-coverage (Algorithm \ref{alg:static-max-cov}) sampling rule is bounded by\hl{
\begin{align*}
\tau \leq 2 \max_{h\in[H]} \sum_{s\in\cS_h}\sum_{a\in\cA_h(s)}  \left(\frac{32\sigma^2 H^2}{\max\left(\overline{\Delta}_h(s,a),\overline{\Delta}_{\min}, \epsilon \right)^2}  \left( \log\left(\frac{4N^3}{\delta}\right) + L_h(s,a)\right) + 2\right),
\end{align*}
where $L_h(s,a) := 8\log\left( \frac{16\sigma H\log\left(\frac{4N^3}{\delta}\right)}{\max\left(\overline{\Delta}_h(s,a),\overline{\Delta}_{\min}, \epsilon \right)}  \right)$.}
\end{theorem}

Before proving Theorem \ref{th:max-cover-sample-comp-tree}, we need to state an important result.

\begin{lemma}\label{lem:period-duration-tree}
In a tree-based MDP, when using either the maximum-coverage or the static maximum-coverage sampling rule, the duration of any non-empty period $k\in\mathbb{N}$ can be bounded as
\begin{align*}
{d}_k \leq 2\sum_{s\in\cS_H}\sum_{a\in\cA_H(s)} \indi{a\in\cA_h^{\bar{t}_k-1}(s)}.
\end{align*}
\end{lemma}
\begin{proof}
For static maximum-coverage sampling, Theorem \ref{th:period-duration-max-cover} followed by Theorem \ref{th:min-flow-max-cut} yields that 
\begin{align*}
d_k \leq 2\varphi^\star(\underline{c}^k) = 2\sum_{(s,a,h) \in \cE(\cC^k)} \indi{a\in\cA_h^{\bar{t}_k-1}(s), n_h^{\bar{t}_k-1}(s,a) < k} \leq 2\sum_{(s,a,h) \in \cE(\cC^k)} \indi{a\in\cA_h^{\bar{t}_k-1}(s)},
\end{align*}
where $\cC^k$ is a maximum cut for the minimum flow problem with lower bounds $\underline{c}^k$. Now note that any $(s,a,h) \in \cE(\cC^k)$ such that $a\in\cA_h^{\bar{t}_k-1}(s)$ reaches a \emph{distinct} leaf $(s',a',H)$ such that $a'\in\cA_H^{\bar{t}_k-1}(s')$. To see why, note that, if some active $(s,a,h) \in \cE(\cC^k)$ reaches no active triplet at the last stage $H$, then we can recursively prove that the sub-tree with root $(s,a,h)$ has been eliminated, which implies that $(s,a,h)$ has been eliminated as well. To see why these triplets are distinct, suppose that there exist two triplets $(s,a,h),(s',a',h') \in \cE(\cC^k)$ that reach the same leaf $(s'',a'',H)$. Since, by definition of forward arcs of a cut, $(s,a,h)$ and $(s',a',h')$ cannot be on the same path, this implies that there exist two different paths to reach the same leaf from the root, which violates the tree-based assumption. This allows us to conclude that $\sum_{(s,a,h) \in \cE(\cC^k)} \indi{a\in\cA_h^{\bar{t}_k-1}(s)} \leq \sum_{s\in\cS_H}\sum_{a\in\cA_H(s)} \indi{a\in\cA_h^{\bar{t}_k-1}(s)}$, and the proof follows.

The reasoning for maximum-coverage sampling is similar. First note that, at each step of period $k$, the sampling rule must play a policy visiting a distinct leaf than those previously visited in the same period. In fact, since there is a unique path to reach each leaf, if the same leaf is visited twice, then at the second visit the value of the objective function would be zero, which cannot happen unless the period has already terminated. Moreover, once all leaves have been visited, by the reasoning above, we are sure that the algorithm has covered a maximum cut for the lower bound function $\underline{c}^k$. That is, all under-sampled triplets have been visited and the period terminates. This proves the stated bound. 
\end{proof}

\begin{proof}[Proof of Theorem \ref{th:max-cover-sample-comp-tree}]
Using Lemma \ref{lem:period-duration-tree} and following the same steps as in the proof of Theorem \ref{th:max-cover-sample-comp}, 
\begin{align*}
\tau = \sum_{k=1}^{k_\tau} \sum_{t=1}^\tau \indi{k_t = k} = \sum_{k=1}^{k_\tau} d_k 
&\leq 2\sum_{k=1}^{k_\tau} \sum_{s\in\cS_H}\sum_{a\in\cA_H(s)} \indi{a\in\cA_h^{\bar{t}_k-1}(s)}
\\ &= 2 \sum_{s\in\cS_H}\sum_{a\in\cA_H(s)} \sum_{k=1}^{k_\tau}\indi{k-1 \leq \kappa_{s,a,h}}
\\ &\leq 2 \sum_{s\in\cS_H}\sum_{a\in\cA_H(s)} (\kappa_{s,a,h}+1)
\leq  2 \sum_{s\in\cS_H}\sum_{a\in\cA_H(s)} (\overline{\kappa}_{s,a,h}+1),
\end{align*}
where in the last inequality we applied Lemma \ref{lem:elim-periods}.
\end{proof}

%% file: sections/app_experiments.tex

\section{Experiment Details}\label{sec:app_computational}
For the implementation, we used {\bf rl-berry} \cite{rlberry}, an open-source python library for implementing and performing parallel Monte-Carlo simulations of  RL algorithms. The code and instructions can be found in the supplementary material.

\paragraph{Computational aspects} We run the experiment on an internal cluster made of 32 CPUs. To speed-up computations, we only perform eliminations every $SA$ episodes for maximum-diameter and at the end of each phase for maximum-coverage. The total run time is 48 hours.

\paragraph{On the choice of baselines} The only algorithms for PAC RL in Episodic MDPs that we are aware of are BPI-UCRL \cite{Kaufmann21RFE}, BPI-UCBVI \cite{Menard21RFE} and MOCA \cite{wagenmaker21IDPAC}. However, we note that BPI-UCRL and BPI-UCBVI only differ in the type of bonus that they use to build confidence regions on the transition probabilities. This means that in our setting of deterministic MDPs, both algorithms are actually equivalent. On the other hand, MOCA has a rather involved design with several unspecified numerical constants and we could not find any open-source implementation of it by the authors. This is why only BPI-UCRL appears in our comparisons.

\paragraph{BPI-UCRL} Whereas EPRL uses confidence intervals on the value of every policy, BPI-UCRL \cite{Kaufmann21RFE} is based on confidence intervals for the optimal value function. Such confidence intervals were originally proposed for stochastic MDPs with known reward function, which require a confidence region for the unknown transition probabilities. In deterministic MDPs, one can easily build confidence intervals on the optimal values by relying on confidence intervals for the unknown mean rewards:
\begin{align*}
\overline{Q}_h^{t,\star}(s,a) := \hat{r}_h^{t}(s,a) + b_h^t(s,a) + \overline{V}_{h+1}^{t,\star}(f_h(s,a)), \quad \overline{V}_{h}^{t,\star}(s) :=  \max_{b}\overline{Q}_h^{t,\star}(s,b),
\end{align*}
\begin{align*}
\underline{Q}_h^{t,\star}(s,a) := \hat{r}_h^{t}(s,a) - b_h^t(s,a) + \underline{V}_{h+1}^{t,\star}(f_h(s,a)), \quad \underline{V}_{h}^{t,\star}(s) :=  \max_{b'}\underline{Q}_h^{t,\star}(s,b')\;.
\end{align*}
using the same exploration bonus as in \eqref{eq:bonus}. In BPI-UCRL, the (optimistic) samplig rule is 
\[\pi^t_h(s) = \argmax_{a \in \cA_h(s)} \overline{Q}_h^{t-1,\star}(s,a).\]
The stopping rule is 
\[\tau^{\text{BPI-UCRL}} = \inf \left\{ t \in \N :  \overline V_1^{t,\star}(s_1) - \underline V_1^{t,\star}(s_1)\leq \varepsilon \right\},\]
while the recommendation rule is the greedy policy with respect to $\underline{Q}^{t,\star}_h(s,a)$.

\paragraph{Value-based eliminations} In our implementation, we used an additional elimination rule for both EPRL and BPI-UCRL, which we call value-based elimination: $a$ is eliminated from $\cA_h^{t}(s)$ if 
\[\overline{Q}_h^{t-1,\star}(s,a) < \underline{V}_h^{t-1,\star}(s).\]
It is easy to justify that on our good event, this sampling rule does not eliminate any optimal action, hence the correctness is preserved. Moreover, adding these eliminations does not alter the sample complexity results obtained in Theorems \ref{th:max-cover-sample-comp} and \ref{th:max-diam-sample-comp} as they can only improve the sample complexity of the resulting algorithms. 

\paragraph{Bonuses in practice} The threshold $\hl{\beta(t,\delta) := 2\sigma^2\log\left(4t^2N/\delta\right)}$ recommended by theory can be overly conservative. In practice, we found that a smaller threshold of $\beta(t,\delta) := 2\sigma^2\log\left((t+1)/\delta\right)$ (i.e., ignoring the union bound) is still empirically correct, \hl{and used it in our experiments}.